\documentclass{article}
\usepackage{amsthm, amssymb, amsmath}
\usepackage{verbatim, float}
\usepackage{mathrsfs}
\usepackage{graphics}
\usepackage{subfig}
\usepackage[usenames,dvipsnames]{pstricks}
\usepackage{epsfig}
\usepackage{pst-grad} % For gradients
\usepackage{pst-plot} % For axes
\usepackage{cases}
\usepackage{multirow}
\usepackage{url}
\usepackage{array}
\usepackage{tabu}

\newcommand{\nc}{\newcommand}

\theoremstyle{definition}
 
\newtheorem{theorem}{Theorem}[section]

\newtheorem{lemma}[theorem]{Lemma}
\newtheorem{corollary}[theorem]{Corollary}

\theoremstyle{remark}

\nc\remove[1]{}

\nc\bfa{{\boldsymbol a}}\nc\bfA{{\mathbf A}}\nc\cA{{\mathcal A}}
\nc\bfb{{\boldsymbol b}}\nc\bfB{{\mathbf B}}\nc\cB{{\mathcal B}}
\nc\bfc{{\boldsymbol c}}\nc\bfC{{\mathbf C}}\nc\cC{{\mathcal C}}
\nc\bfd{{\boldsymbol d}}\nc\bfD{{\mathbf D}}\nc\cD{{\mathcal D}}\nc\sD{{\mathscr D}}
\nc\bfe{{\boldsymbol e}}\nc\bfE{{\mathbf E}}\nc\cE{{\EuScript E}}
\nc\bff{{\boldsymbol f}}\nc\bfF{{\mathbf F}}\nc\cF{{\mathcal F}}
\nc\bfg{{\boldsymbol g}}\nc\bfG{{\mathbf G}}\nc\cG{{\mathcal G}}
\nc\bfh{{\boldsymbol h}}\nc\bfH{{\mathbf H}}\nc\cH{{\mathcal H}}
\nc\bfi{{\boldsymbol i}}\nc\bfI{{\mathbf I}}\nc\cI{{\mathcal I}}
\nc\bfj{{\boldsymbol j}}\nc\bfJ{{\mathbf J}}\nc\cJ{{\mathcal J}}
\nc\bfk{{\boldsymbol k}}\nc\bfK{{\mathbf K}}\nc\cK{{\mathcal K}}
\nc\bfl{{\boldsymbol l}}\nc\bfL{{\mathbf L}}\nc\cL{{\mathcal L}}\nc\sL{{\mathscr L}}
\nc\bfm{{\boldsymbol m}}\nc\bfM{{\mathbf M}}\nc\cM{{\mathcal M}}
\nc\bfn{{\boldsymbol n}}\nc\bfN{{\mathbf N}}\nc\cN{{\mathcal N}}
\nc\bfo{{\boldsymbol o}}\nc\bfO{{\mathbf O}}\nc\cO{{\mathcal O}}
\nc\bfp{{\boldsymbol p}}\nc\bfP{{\mathbf P}}\nc\cP{{\mathcal P}}
\nc\bfq{{\boldsymbol q}}\nc\bfQ{{\mathbf Q}}\nc\cQ{{\mathcal Q}}\nc\sQ{{\mathscr Q}}
\nc\bfr{{\boldsymbol r}}\nc\bfR{{\mathbf R}}\nc\cR{{\mathcal R}}
\nc\bfs{{\boldsymbol s}}\nc\bfS{{\mathbf S}}\nc\cS{{\mathcal S}}
\nc\bft{{\boldsymbol t}}\nc\bfT{{\mathbf T}}\nc\cT{{\mathcal T}}\nc\sT{{\mathscr T}}
\nc\bfu{{\boldsymbol u}}\nc\bfU{{\mathbf U}}\nc\cU{{\mathcal U}}
\nc\bfv{{\boldsymbol v}}\nc\bfV{{\mathbf V}}\nc\cV{{\mathcal V}}
\nc\bfw{{\boldsymbol w}}\nc\bfW{{\mathbf W}}\nc\cW{{\mathcal W}}\nc\sW{{\mathscr W}}
\nc\bfx{{\boldsymbol x}}\nc\bfX{{\mathbf Z}}\nc\cX{{\EuScript X}}
\nc\bfy{{\boldsymbol y}}\nc\bfY{{\mathbf Y}}\nc\cY{{\EuScript Y}}\nc\sY{{\mathscr Y}}
\nc\bfz{{\boldsymbol z}}\nc\bfZ{{\mathbf Z}}\nc\cZ{{\mathcal Z}}\nc\sZ{{\mathscr Z}}

\def\h_q{\qopname\relax{no}{h_q}}
\def\h{\qopname\relax{no}{h}}

\usepackage[nonatbib, preprint]{nips_2018}

\usepackage[utf8]{inputenc} % allow utf-8 input
\usepackage[T1]{fontenc}    % use 8-bit T1 fonts
\usepackage{hyperref}       % hyperlinks
\usepackage{url}            % simple URL typesetting
\usepackage{booktabs}       % professional-quality tables
\usepackage{amsfonts}       % blackboard math symbols
\usepackage{nicefrac}       % compact symbols for 1/2, etc.
\usepackage{microtype}      % microtypography

\title{Quadratic Decomposable Submodular Function Minimization}

\author{
  Pan Li \\
  UIUC \\
  \texttt{panli2@illinois.edu} \\
  \And
  Niao He \\
  UIUC \\
   \texttt{niaohe@illinois.edu} \\
  \And
    Olgica Milenkovic \\
  UIUC \\
  \texttt{milenkov@illinois.edu} \\
}
\begin{document}

\maketitle
\begin{abstract} We introduce a new convex optimization problem, termed \emph{quadratic decomposable submodular function minimization}. The problem is closely related to decomposable submodular function minimization and arises in many learning on graphs and hypergraphs settings, such as graph-based semi-supervised learning and PageRank. We approach the problem via a new dual strategy and describe an objective that may be optimized via random coordinate descent (RCD) methods and projections onto cones. We also establish the linear convergence rate of the RCD algorithm and develop efficient projection algorithms with provable performance guarantees. Numerical experiments in semi-supervised learning on hypergraphs confirm the efficiency of the proposed algorithm and demonstrate the significant improvements in prediction accuracy with respect to state-of-the-art methods.\footnote{The code for QDSFM is available at https://github.com/lipan00123/QDSDM.}
\end{abstract}
\section{Introduction}
Given $[N]=\{1,2,...,N\}$, a submodular function $F: 2^{[N]} \rightarrow \mathbb{R}$ is a set function that for any $S_1, S_2 \subseteq [N]$ satisfies $F(S_1) + F(S_2) \geq F(S_1\cup S_2) + F(S_1 \cap S_2)$. Submodular functions are ubiquitous in machine learning as they capture rich combinatorial properties of set functions and provide useful regularization functions for supervised and unsupervised learning~\cite{bach2013learning}. Submodular functions also have continuous Lov{\'a}sz extensions~\cite{lovasz1983submodular}, which establish solid connections between combinatorial and continuous optimization problems. 
%\textcolor{red}{This is to vague, I would suggest removing it; and its too much notation for the intro: From combinatorial perspective, $F(S)$ can be understood as the cost to split a subset $S$ of variables from the rest of $[N]$. Correspondingly, if $x\in \mathbb{R}^N$ is a given potential over $[N]$,  $F(S)$ imposes smoothness between groups $\{x_i\}_{i\in S}$ and $\{x_i\}_{i\in [N]/S}$.} 

Due to their versatility, submodular functions and their Lov{\'a}sz extensions are frequently used in applications such as learning on directed/undirected graphs and hypergraphs~\cite{zhu2003semi,zhou2004learning}, image denoising via total variation regularization~\cite{osher2005iterative, chambolle2009total} and MAP inference in high-order Markov random fields~\cite{kumar2003discriminative}. 
%\textcolor{red}{This is again too vague, I suggest remove this part that I cut out from above: used to model applications with combinatorial relations from low order to high order from symmetric cases to asymmetric cases} 
%\textcolor{blue}{For above two descriptions, I am okay with remove those statements while may like to add one or two sentences how they are used as regularization. I would like to emphasize they can be used to translate high-order combinatorial relations into continuous spaces.}
In many optimization settings involving submodular functions, one encounters the convex program
\begin{align*}
\min_{x} \sum_{i\in [N]}(x_i - a_i)^2 + \sum_{r\in[R]} \left[f_r(x)\right]^p,
\end{align*}
where $a\in \mathbb{R}^N$, $p\in\{1,2\}$, and where for all $r$ in some index set $[R]$, $f_r$ stands for the Lov{\'a}sz extension of a submodular function $F_r$ that describes a combinatorial structure over the set $[N]$. For example, in image denoising, each parameter $a_i$ may correspond to the observed value of a pixel $i$, while the functions $\left[f_r(x)\right]^p$ may be used to impose smoothness constraints on pixel neighborhoods. One of the main difficulties in solving this optimization problem comes from the nondifferentiability of the second term: a direct application of subgradient methods leads to convergence rates as slow as $1/\sqrt{k}$, where $k$ denotes the number of iterations~\cite{shor2012minimization}.  
%\textcolor{red}{You never said what $t$ is suppose to stand for - number of steps?}\textcolor{blue}{I think it is fine either with an explanation or not, as people with a little optimization background will know what it indicates. However, maybe we want to keep the notation consistent with what we later used as \# iterations?}

In recent years, the above described optimization problem with $p=1$ has received significant interest in the context of \emph{decomposable submodular function minimization} (DSFM)~\cite{stobbe2010efficient}. The motivation for studying this particular setup is two-fold: first, solving the convex optimization problem directly recovers the combinatorial solution to the submodular min-cut problem $\min_{S\subseteq [N]} F(S)$, where $F(S)=\sum_{r\in [R]} F_r(S) - 2\sum_{i\in S}a_{i}$~\cite{fujishige2005submodular}; second, minimizing a submodular function decomposed into a sum of simpler components  $F_r$, $r\in [R],$ is much easier than minimizing an unrestricted submodular function $F$ over a large set $[N]$.  %\textcolor{blue}{Several solvers pertaining to the DSFM problem have been developed recently including Jegelka et al.'s reflection method~\cite{jegelka2013reflection}, Nishihara et al.'s alternating projection method~\cite{nishihara2014convergence}, Ene et al.'s coordinate descent method~\cite{ene2015random} and Pan et al.'s parallel counterparts~\cite{li2018revisiting}. These approaches tackle the problem via its dual and establish linear convergence results by proving the weak strong convex property~\cite{ene2017decomposable}.}
There are several milestone results for the DSFM problem: Jegelka et al.~\cite{jegelka2013reflection} first tackled the problem by considering its dual and proposed a solver based on Douglas-Rachford splitting. Nishihara et al.~\cite{nishihara2014convergence} established the linear convergence rate of alternating projection methods for solving the dual problem. Ene et al.~\cite{ene2015random,ene2017decomposable} presented linear convergence rates of coordinate descent methods and subsequently tightened the results via submodular flows. Pan et al.~\cite{li2018revisiting} improved those methods by leveraging incidence relations of the arguments of submodular function components.

Here,  we focus on the other important case when $p=2$; we refer to the underlying optimization problem as \emph{quadratic DSFM} (QDSFM). QDSFM appears naturally in a wide spectrum of applications, including learning on graphs and hypergraphs, and in particular, semi-supervised learning and PageRank. 
% (to be discussed in detail later)
It has also been demonstrated both theoretically~\cite{johnson2007effectiveness} and empirically~\cite{zhou2004learning,hein2013total} that employing regularization with quadratic terms offers significantly improved predictive performance when compared to the case when $p=1$. Despite the importance of the QDSFM problem, its theoretical and algorithmic developments have not reached the same level of maturity as those for the DSFM problem. To the best of our knowledge, only a few reported works~\cite{hein2013total,zhang2017re} have provided solutions for specific instances of QDSFMs with sublinear convergence guarantees.

This work takes a substantial step towards solving the QDSFM problem in its most general form by developing a family of algorithms with \emph{linear convergence rate} and \emph{small iteration cost}, including the randomized coordinate descent (RCD) and alternative projection (AP) algorithms. %For space consideration, we discuss only RCD in the main text and postpone the discussion of AP in Supplement. 
Our contributions are as follows. First, we derive a new dual formulation for the QDSFM problem since an analogue of the dual transformation for the DSFM problem is not applicable. Interestingly, the dual QDSFM problem requires one to find the best approximation of a hyperplane via a product cone as opposed to a product polytope, encountered in the dual DSFM problem. Second, we develop a linearly convergent RCD (and AP) algorithm for solving the dual QDSFM. Because of the special underlying conic structure, new analytic approaches are needed to prove the weak strong convexity of the dual QDSFM, which essentially guarantees linear convergence. Third, we develop generalized Frank-Wolfe and min-norm-point methods for efficiently computing the conic projection required in each step of RCD (and AP) and provide a $1/k$-rate convergence analysis. Finally, we evaluate our methods on semi-supervised learning over hypergraphs using synthetic and real datasets, and demonstrate superior performance both in convergence rate and prediction accuracy compared to existing methods.
%To test the performance of our methods and the new projection algorithms, we consider the problem of semi-supervised learning on hypergrpahs based on synthetic and real data from the UCI datatbase. We show that the QDSFM framework with special projections suitable for hypergraph learning offers superior performance compared to existing techniques.

% and they pertain to 
%\textcolor{blue}{As our experiments strongly depend on the special projection for hypergraphs, do we mention something here?}

\section{Notation and Problem Formulation}
%\subsection{Lov{\'a}sz extension}
%We find the following terminology and notation useful for our subsequent exposition. \textcolor{blue}{Remove this?}
For a submodular function $F$ defined over the ground set $[N]$, the \emph{Lov{\'a}sz extension} is a convex function $f: \mathbb{R}^{N} \rightarrow \mathbb{R},$ defined for all $x\in\mathbb{R}^{N}$ according to
\begin{align}\label{lovasext}
f(x) = \sum_{k=1}^{N-1} F(\{i_1,...,i_k\})(x_{i_k} - x_{i_{k+1}})  + F([N])x_{i_{N}},
\end{align} 
where $x_{i_1} \geq x_{i_2} \geq \cdots \geq x_{i_{N}}$.  
For a vector $x\in\mathbb{R}^N$ and a set $S\subseteq[N]$, let $x(S) = \sum_{i\in[S]} x_i$ where $x_i$ is the component of $x$ in $i$th dimention. Then, the \emph{base polytope} of $F$, denoted by $B$, is defined as 
\begin{align}\label{basepolytope}
B = \{y\in \mathbb{R}^{N}| y(S) \leq F(S),\; \forall S\subset [N],\,y([N]) = F([N])\}.
\end{align} 
%This statement is equivalent to saying that at least one of the two values $F(\{i\})$ and $F( [N]) - F( [N]/\{i\})$ is nonzero. 
%For simplicity, we use $S_r \subseteq [N]$ to denote the set of variables incident to $F_r$. \textcolor{blue}{Not sure if it is good to use subscript $r$ here as later, we will formally define the problem}
%, which in these applications typically refer to the subset of vertices in one edge/hyperedge. 
Using the base polytope, the Lov{\'a}sz extension can also be written as $f(x) =\max_{y\in B} \langle y, x\rangle.$ 

We say that an element $i\in [N]$ is incident to $F$ if there exists a $S\subset [N]$ such that $F(S) \neq F(S \cup \{i\})$. %Throughout the paper, $|S|$ is used to denote the cardinality of the set $S$, and
Furthermore, we use $(x)_+$ to denote the function $\max\{x, 0\}$. Given a positive diagonal matrix $W\in \mathbb{R}^{N\times N}$ and a vector $x\in \mathbb{R}^{N}$, we define the $W$-norm according to $\|x\|_{W} = \sqrt{\sum_{i=1}^N W_{ii} x_i^2},$ and simply use $\|\cdot\|$ when $W=I$, the identity matrix. 
 %For $W=I$, we write $\|\cdot\|$ instead of $\|\cdot\|_{I}$.
 For an index set $[R]$, we denote the $R$-product of $N$-dimensional Euclidean spaces by $\otimes_{r\in[R]} \mathbb{R}^{N}$. A vector $y\in \otimes_{r\in[R]} \mathbb{R}^{N}$ is written as $(y_1, y_2,...,y_R),$ where $y_r\in \mathbb{R}^N$ for all $r \in [R]$. The $W$-norm induced on $\otimes_{r\in[R]} \mathbb{R}^{N}$ equals $\|y\|_{I(W)}= \sqrt{\sum_{r=1}^R \|y_r\|_W^2}$. We reserve the symbol $\rho$ for $\max_{y_r\in B_r, \forall r} \sqrt{\sum_{r\in [R]}\|y_r\|_1^2}$.

Next, we formally state the QDSFM problem. Consider a collection of submodular functions $\{F_r\}_{r\in[R]}$ defined over the ground set $[N]$, and denote their Lov{\'a}sz extensions and base polytopes by $\{f_r\}_{r\in[R]}$ and $\{B_r\}_{r\in[R]}$, respectively. We use $S_r \subseteq [N]$ to denote the set of variables incident to $F_r$ and make the further assumption that the functions $F_r$ are normalized and nonnegative, i.e., that $F_r(\emptyset) = 0$ and $F_r \geq 0$. These two mild constraints are satisfied by almost all submodular functions that arise in practical applications. We consider the following minimization problem: 
\begin{align}\label{QDSFM}
\text{QDSFM:} \quad\quad \min_{x\in \mathbb{R}^N}\|x - a\|_W^2 + \sum_{r\in[R]}\left[f_r(x)\right]^2,
\end{align}
where $a\in \mathbb{R}^N$ is a given vector and $W\in\mathbb{R}^{N\times N}$ is a \emph{positive diagonal} matrix. As an immediate observation, the problem has a unique solution, denoted by $x^*$, due to the strong convexity of~\eqref{QDSFM}.

\section{Applications}\label{application}
We start by reviewing some important machine learning problems that give rise to QDSFM. %These problems will be discussed in greater detail elsewhere.\

\textbf{\emph{Semi-supervised Learning}} (SSL) is a learning paradigm that allows one to utilize the underlying structure or distribution of unlabeled samples, whenever the information provided by labeled samples does not suffice for learning an inductive predictor~\cite{gammerman1998learning, joachims2003transductive}. A standard setting for a $K$-class transductive learner is as follows: given $N$ data points $\{z_i\}_{i\in[N]},$ and labels for the first $l$ ($\ll N$) samples $\{y_i| y_i \in [K] \ \}_{i\in[l]}$, the learner is asked to infer the labels for all the remaining data points $i\in[N]/[l]$. The widely-used SSL problem with least squared loss requires one to solve $K$ regularization problems: for each class $k\in [K]$, set the scores of data points within the class to
$$\hat{x}^{(k)} = \arg\min_{x^{(k)}} \beta \|x^{(k)} - a^{(k)}\|^2 + \Omega(x^{(k)}),$$
where $a^{(k)}$ represents the information provided by the known labels, i.e.,  $a_i^{(k)} = 1$ if $y_i = k$, and $0$ otherwise, $\beta$ denotes a hyperparameter and $\Omega$ stands for a smoothness regularizer. The labels of the data points are inferred according to $\hat{y}_{i} = \arg\max_{k} \{\hat{x}_i^{(k)}\}$. For typical graph and hypergraph learning problems, $\Omega$ is often chosen to be a Laplacian regularizer constructed using $\{z_i\}_{i\in[N]}$ (see Table~\ref{tab:TLexamples}). In Laplacian regularization, each edge/hyperedge corresponds to one functional component in the QDSFM problem. Note that the variables may also be normalized with respect to their degrees, in which case the normalized Laplacian is used instead. For example, in graph learning, one of the terms in $\Omega$ assumes the form $w_{ij} (x_i/\sqrt{d_i} - x_j/\sqrt{d_j})^2,$ where $d_i$ and $d_j$ correspond to the degrees of the vertex variables $i$ and $j$, respectively. It can be shown using some simple algebra that the normalization term reduces to the matrix $W$ used in the definition of the QDSFM problem~\eqref{QDSFM}.
\begin{table}[h] 
\scriptsize
\begin{tabular}{|p{3.1cm}<{\centering}|p{5.5cm}<{\centering}|p{4.3cm}<{\centering}|}
\hline
One component in $\Omega(x)$ & Description of the combinatorial structure & The submodular function \\
\hline
$w_{r} (x_i - x_j)^2$, $S_r = \{i,j\}$ & Graphs: Nearest neighbors~\cite{zhou2004learning}/Gaussian similarity~\cite{zhu2003combining}  & $F_{r}(S) = \sqrt{w_{ij}}$ if $|S\cap \{i,j\}|=1$  \\
\hline
$w_{r} \max_{i,j \in S_r} (x_i - x_j)^2$ & Hypergraphs: Categorical features~\cite{hein2013total} & \tiny{$F_r(S) = \sqrt{w_{r}}$ if $|S\cap S_r| \in [1, |S_r| -1]$}  \\
\hline
$w_{r}\max\limits_{(i,j)\in H_r \times T_r}(x_i -x_j)_+^2$ & Directed hypergraphs: citation networks~\cite{zhang2017re} & \tiny{$F_r(S) = \sqrt{w_{r}}$ if $|S\cap H_r| \geq 1$, $| ([N]/S)\cap T_r| \geq 1$}  \\
\hline
General $[f_r(x)]^2$ & Submodular Hypergraphs: Mutual Information~\cite{li2017inhomogeneous,li2018submodular} & A symmetric submodular function\\
\hline
\end{tabular}
\centering 
\caption{\footnotesize{Laplacian regularization in semi-supervised learning. In the third column, whenever the stated conditions are not satisfied, it is assumed that $F_r=0$. For directed hypergraphs, $H_r$ and $T_r$ are subsets of $S_r$ termed the head and the tail set. When $H_r = T_r = S_r$, one recovers the setting for undirected hypergraphs.}}\label{tab:TLexamples}
\end{table}
\\\textbf{\emph{PageRank}} (PR) is a well-known method used for ranking Web pages~\cite{page1999pagerank}. Web pages are linked and they naturally give rise to a graph $G=(V, E)$, where, without loss of generality, one may assume that $V = [N]$. Let $A$ and $D$ be the adjacency matrix and diagonal degree matrix of $G$, respectively. PR essentially finds a fixed point $p\in \mathbb{R}^N$ via the iterative procedure $p^{(t+1)} = (1-\alpha) s + \alpha AD^{-1} p^{(t)}$, where $s \in \mathbb{R}^N$ is a fixed vector and $\alpha \in (0,1]$. It is easy to verify that $p$ is a solution of the problem
\begin{align}\label{PReq}
\min_p\frac{(1-\alpha)}{\alpha} \|p - s\|_{D^{-1}}^2 + (D^{-1} p)^T(D -A)(D^{-1} p)  =\|x - a\|_{W}^2 + \sum_{ij\in E}(x_i - x_j)^2,
\end{align}
where $x = D^{-1} p$, $a = D^{-1}s$ and $W = \frac{(1-\alpha)}{\alpha} D$. Obviously,~(\ref{PReq}) may be viewed as a special instance of the QDSFM problem. Note that the PR iterations on graphs take the form $D^{-\frac{1}{2}}(p^{(t+1)} - p^{(t)})= (1-\alpha) D^{-\frac{1}{2}}(s-  p^{(t)}) - \alpha L(D^{-\frac{1}{2}}p^{(t)}),$ where $L = I - D^{-\frac{1}{2}}AD^{-\frac{1}{2}}$ is the normalized Laplacian of the graph. The PR problem for hypergraphs is significantly more involved, and may be formulated using diffusion processes (DP) based on a normalized hypergraph Laplacian operator $L$~\cite{chan2018spectral}. The underlying PR procedure reads as 
$\frac{d x}{dt} = (1-\alpha)(a - x) - \alpha L(x),$
where $x(t)\in \mathbb{R}^N$ is the potential vector at time $t$. Tracking this DP precisely for every time point $t$ is a difficult task which requires solving a densest subset problem~\cite{chan2018spectral}. However, the stationary point of this problem, i.e., a point $x$ that satisfies $(1-\alpha)(a - x) - \alpha L(x) = 0$ may be easily found by solving the optimization problem $$\min_{x} (1-\alpha )\|x - a\|^2 + \alpha \langle x, L(x)\rangle.$$
The term $\langle x, L(x)\rangle$ matches the normalized regularization term for hypergraphs listed in Table~\ref{tab:TLexamples}, i.e., $\sum_{r} \max_{i,j\in S_r}(x_i/\sqrt{d_i} - x_j/\sqrt{d_j})^2$. Clearly, once again this leads to the QDSFM problem. The PR equation for directed or submodular hypergraphs can be stated similarly using the Laplacian operators described in~\cite{chan2017diffusion,li2018submodular,yoshida2017cheeger}. The PR algorithm defined in this manner has many advantages over the multilinear PR method based on higher-order Markov chains~\cite{gleich2015multilinear}, since it allows for arbitrarily large orders and is guaranteed to converge for any $\alpha \in (0,1]$. In a companion paper, we provide a more detailed analysis of the above described PR method. %based on the QDSFM optimization framework. %As a final remark, we observe that in the context of total variation denoising, Chambolle et al.~\cite{chambolle2009total} proposed a regularization term called ``oscillation'' that takes the same form as the regularizer for hypergraphs in Table~\ref{tab:TLexamples}; here, $S_r$ refers to a subset of adjacent pixels of images. 

\section{Algorithms for Solving the QDSFM Problem}
We describe next the \emph{first known linearly convergent} algorithms for solving the QDSFM problem. To start with, observe that the QDSFM problem is convex since the  Lov{\'a}sz  extensions $f_r$ are convex and nonnegative. But the objective is in general nondifferentiable. To address this issue, we consider the  dual of the QDSFM problem. A natural idea is to try to mimic the approach used for DSFM by invoking the characterization of the Lov{\'a}sz extension, $f_r(x) = \max_{y_r\in B_r} \langle y_r, x\rangle$, $\forall r$. However, this leads to a semidefinite programing problem for the dual variables $\{y_r\}_{r\in[R]}$, which is complex to solve for large problems. Instead, we establish a new dual formulation that overcomes this obstacle.  The dual formulation hinges upon on the following key observation: 
\begin{align}\label{conjugate} 
[f_r(x)]^2 =  \max_{\phi_r\geq 0} \phi_r f_r(x) -  \frac{ \phi_r^2}{4}  = \max_{\phi_r\geq 0}\max_{y_r\in \phi_r B_r}\; \langle y_r, x\rangle - \frac{\phi_r^2}{4}.
\end{align}
%\Niao{Should we put this as a lemma or just an inline equation, since this is straightforward? Why removing Italic font for the lemmas and theorems, just curious...\\}

%The following transformation of the quadratic form of $f_e$  is important. 
%\begin{lemma}
%\begin{align}\label{conjugate}
% \vartheta_e [f_e(A_e x)]^2  = \max_{r_e, y_e}\; \langle y_e, A_e x\rangle - \frac{r_e^2}{4\vartheta_e}  \quad \text{s.t. $r_e \geq 0, \; y_e \in r_e B_e$}. 
%\end{align}
%\end{lemma}
%\begin{proof}
%\begin{align*}
% RHS=\max_{r_e\geq 0}\max_{y_e\in r_e B_e}\; \langle y_e, A_e x\rangle - \frac{r_e^2}{4\vartheta_e}  = \max_{r_e\geq 0} r_e f_e(A_e x) -  \frac{r_e^2}{4\vartheta_e}  = \vartheta_e[f_e(A_e x)]^2. 
%\end{align*}
%\end{proof}
Let $y = (y_{1}, y_2,...,y_R)$ and $\phi = (\phi_1, \phi_2, ..., \phi_R)$. Using equation~\eqref{conjugate}, we arrive at 
%the following two dual formulations for the original problem of Lemma~\eqref{dualform}.

%For each hyperedge $e$, define a convex cone $C_e = \{(y_e, r_e) | r_e \geq 0, y_e\in r_e B_e \}$ as the feasible set of the dual variables associated with the hyperedge $e$.
%Using formula~\eqref{conjugate}, we may achieve the following two dual formulations of the original problem
\begin{lemma} \label{dualform} The following optimization problem is dual to \eqref{QDSFM}:
\begin{align}\label{CDform}
 \min_{y, \phi}\;g(y, \phi): = \|\sum_{r\in[R]} y_r - 2Wa\|_{W^{-1}}^2 + \sum_{r\in[R]} \phi_r^2, \quad \text{s.t. $y\in\otimes_{r\in [R]} \phi_rB_r$, $\phi \in\otimes_{r\in [R]} \mathbb{R}_{\geq 0}$}.
\end{align}
By introducing $\Lambda = (\lambda_r) \in \otimes_{r\in [R]} \mathbb{R}^{N}$, the previous optimization problem can be rewritten as  
\begin{align} \label{APform}
\hspace{-0.2cm} \min_{y,\phi, \Lambda}\sum_{r\in [R]}\left[\| y_r - \frac{\lambda_r}{\sqrt{R}}\|_{W^{-1}}^2 + \phi_r^2\right], \;\text{s.t. $y\in\otimes_{r\in [R]} \phi_rB_r$, $\phi \in\otimes_{r\in [R]} \mathbb{R}_{\geq 0}$,$\sum_{r\in [R]} \lambda_r = 2Wa$}. 
\end{align}
% \vspace{-0.3cm}
The primal variables in both cases are recovered via $x = a - \frac{1}{2}W^{-1}\sum_{r\in[R]} y_r$.
\end{lemma}

Counterparts of the above results for the DSFM problem were discussed in~Lemma 2 of~\cite{jegelka2013reflection}. 
However, there is a significant difference between~\cite{jegelka2013reflection} and the QDSFM problem, since in the latter setting we use a conic set constructed from base polytopes of submodular functions. More precisely, for each $r$, we define a convex 
cone $C_r = \{(y_r, \phi_r) | \phi_r \geq 0, y_r\in \phi_r B_r \}$ which gives the feasible set of the dual variables 
$(y_r, \phi_r)$. The optimization problem~\eqref{APform} essentially asks one to find the best approximation of an affine 
space in terms of a product cone $\otimes_{r\in[R]} C_r$, as opposed to a product polytope encountered in DSFM. Several algorithms have been developed for solving the DSFM problem, including the Douglas-Rachford splitting method (DR)~\cite{jegelka2013reflection}, the alternative projection method (AP)~\cite{nishihara2014convergence} and the random coordinate descent method (RCD)~\cite{ene2015random}. Similarly, for QDSFM, we propose to solve the dual problem~\eqref{CDform} using the RCD method exploiting the separable structure of the feasible set, and to solve~\eqref{APform} using the AP method. Although these similar methods for DSFM may be used for QDSFM, a novel scheme of analysis handling this conic structure is required, which takes all the effort in the rest of this section and the next section. The analysis of the AP method is deferred to the Supplement (Section~\ref{appsec:AP}). It is worth mentioning that results of this work can be easily extended for the DR method, as well as accelerated and parallel variants of the RCD method~\cite{ene2015random, li2018revisiting}.
%, by referring to the ideas outlined in this work.  

\textbf{RCD Algorithm.} 
Define the projection $\Pi$ onto a convex cone $C_r$ as follows: for a given point $b$ in $\mathbb{R}^{N}$, let $\Pi_{C_r}(b) = \arg\min_{(y_r, \phi_r)\in C_r} \|y_r - b\|_{W^{-1}}^2 + \phi_r^2$. For each coordinate $r$, optimizing over the dual variables $(y_r, \phi_r)$ is equivalent to computing a projection onto the cone $C_r$. This gives rise to the RCD method summarized in Algorithm 1. 
\begin{table}[htb]
\centering
\begin{tabular}{l}
\hline
\label{RCDMalg}
\textbf{Algorithm 1:} \textbf{RCD Solver for~\eqref{CDform}} \\
\hline
\ 0: For all $r$, initialize $y_r^{(0)}\leftarrow 0$, $\phi_r^{(0)}$ and $k\leftarrow 0$\\
\ 1: In iteration $k$:\\
\ 2: \; Uniformly at random pick an $r\in [R]$.\\
%\ 3: \; $\tilde{a}^{(k)}\leftarrow a^{(k)} + A_e^Ty_{e}^{(k)}$ \\
\ 3: \; $(y_{r}^{(k+1)}, \phi_{r}^{(k+1)})\leftarrow  \Pi_{C_{r}}(2Wa - \sum_{r'\neq r} y_{r'}) $\\
%\ 5: \; $a^{(k+1)}\leftarrow \tilde{a}^{(k)} - A_e^Ty_{e}^{(k+1)}$ \\
\ 4: \; Set $y_{r'}^{(k+1)}\leftarrow y_{r'}^{(k)}$ for $r'\not = {r}$ \\
\hline
\end{tabular}
\end{table}

In Section~\ref{innerloop}, we describe efficient methods to compute the projections. But throughout the remainder of this section, we treat the projections as provided by an oracle. Note that each iteration of the RCD method only requires the computation of one projection onto a single cone. In contrast, methods such as DR, AP and the primal-dual hybrid gradient descent (PDHG) proposed in~\cite{chambolle2011first} used for SSL on hypergraphs~\cite{hein2013total}, require performing a complete gradient descent and computing a total of $R$  projections at each iteration. Thus, from the perspective of iteration cost, RCD is significantly more efficient, especially when $R$ is large and computing $\Pi(\cdot)$ is costly.
% PDHG also requires additional parameter tuning. 

The objective $g(y,\phi)$ described in~\eqref{CDform} is not strongly convex in general. 
Inspired by the work for DSFM~\cite{ene2015random}, in what follows, we show that this objective indeed satisfies a weak strong convexity condition, which guarantees linear convergence of the RCD algorithm. Note that due to the additional term $\phi$ that characterizes the conic structures, extra analytic effort is required than that for the DSFM case. We start by providing a general result that characterizes relevant geometric properties of the cone $\otimes_{r\in [R]} C_r$. 
%the RCD for the DSFM problem has a linear convergence rate due to the weak strong convexity of the objective. Hence, a natural question is to ask if our objective $g(y,\phi)$ has the weak strong convexity property and what the induced convergence rate of Algorithm 1 would be in this case. This issue is addressed in the what follows.
% However, inspired by the work for DSFM~\cite{ene2015random}, we only require $g(y,r)$ carries a weaker property shown in Lemma~\ref{strongconv} to guarantee the linear convergence rate of RCDM for QDSFM problems. We start by proving the following lemma that is the first step towards proving linear convergence rates of the RCD method, and  which characterizes a geometric properties of $\otimes_{r\in [R]} C_r$. %This lemma modifies Lemma 4.1~\cite{ene2017decomposable} and Lemma 2.1~\cite{li2018revisiting}

\begin{lemma}\label{submodularcone}
Consider a feasible solution $(y,\phi)\in \otimes_{r\in [R]} C_r$ and a nonnegative vector $\phi' = (\phi_{r}') \in \otimes_{r\in [R]}\mathbb{R}_{\geq 0}$. Let $s$ be an arbitrary point in the base polytope of $\sum_{r\in [R]} \phi_r' F_r$, and let 
$W^{(1)}, W^{(2)}$ be two positive diagonal matrices. Then, there exists a $y' \in \otimes_{r\in [R]} \phi_r'B_r$ such that 
$\sum_{r\in [R]} y_r' = s$ and \vspace{-0.2cm} 
$$\|y- y'\|_{I(W^{(1)})}^2 +\|\phi - \phi'\|^2 \leq \mu(W^{(1)}, W^{(2)})\left[\|\sum_{r\in [R]} y_r - s\|_{W^{(2)}}^2 + \|\phi -\phi'\|^2\right],$$ \vspace{-0.3cm} where 
\begin{align}\label{defmu}
\mu(W^{(1)}, W^{(2)}) = \max\left\{\sum_{i\in[N]}W_{ii}^{(1)}\sum_{j\in [N]}1/W_{jj}^{(2)},\;\frac{9}{4}\rho^2\sum_{i\in [N]}W_{ii}^{(1)}+1\right\}. 
\end{align}
\end{lemma}
 %\subsection{Random Coordinate Descent Method}\label{RCDM}
%Note that the techniques used to accelerate and parallelize RCDM for DSFM can also be extended to QDSFM as well~\cite{ene2015random}~\cite{li2018revisiting}.
%The objective $g(y,r)$ in~\eqref{CDform} is not strongly convex in general. However, inspired by the works for DSFM~\cite{ene2015random}, we only require $g(y,r)$ carries a weaker property shown in Lemma~\ref{strongconv} to guarantee the linear convergence rate of RCDM for QDSFM problems. In this part, we only introduce sequential random coordinate descent method. 

As a corollary of Lemma~\ref{submodularcone}, the next result establishes the weak strong convexity of $g(y,\phi)$. To proceed, we introduce some additional notation. Denote the set of solutions of problem~\eqref{CDform} by 
\begin{align*}
\Xi= \{(y,\,\phi)| \sum_{r\in [R]}y_r = 2W(a - x^*), \phi_r = \inf_{y_r\in \theta B_r} \theta, \forall r\}.
\end{align*}
%\vspace{-0.1cm}
Note that this representation arises from the relationship between the optimal primal and dual solution as stated in Lemma~\ref{dualform}. We denote the optimal value of the objective over $(y,\,\phi)\in \Xi$ by $g^* = g (y,\phi)$, and define a distance function $d((y,\phi), \Xi) = \sqrt{\min\limits_{(y',\phi')\in\Xi} \|y-y'\|_{I(W^{-1})}^2 + \|\phi -\phi'\|^2}$. %We have
\begin{lemma}\label{strongconv} Suppose that $(y, \phi)\in \otimes_{r\in [R]} C_r$ and that $(y^*, \phi^*)\in\Xi$ minimizes $\|y-y^*\|_{I(W^{-1})}^2 + \|\phi-\phi^*\|^2$. Then 
\begin{align*}
\|\sum_{r\in [R]} (y_r- y_r^*)\|_{W^{-1}}^2+  \|\phi-\phi^*\|^2\geq  \frac{d^2((y,\phi), \Xi)}{\mu(W^{-1}, W^{-1})}.
%\left[\|y-y^*\|_{I(W^{-1})}^2 + \|\phi-\phi^*\|^2\right].
\end{align*}
%where $l= \frac{1}{\mu(W^{-1}, W^{-1}, \Theta)}$. 
\end{lemma}
\vspace{-0.2cm}Based on Lemma~\ref{strongconv}, we can establish the linear convergence rate of the RCD algorithm. %\textcolor{blue}{Add high-level explanation from 4.6 to 4.7. Niao?}
\begin{theorem} \label{linearconv}
After $k$ iterations of Algorithm 1, we obtain a pair $(y^{(k)}, \phi^{(k)})$ that satisfies 
\begin{align*}
&\mathbb{E}\left[g(y^{(k)},  \phi^{(k)})- g^* + d^2((y^{(k)},  \phi^{(k)}), \Xi) \right]  \\
&\leq \left[1- \frac{2}{R[1+\mu(W^{-1}, W^{-1})]}\right]^{k}\left[g(y^{(0)},  \phi^{(0)})- g^* + d^2((y^{(0)}, r^{(0)}), \Xi) \right].
\end{align*}
\end{theorem}
 Theorem~\ref{linearconv} implies that $O(R\mu(W^{-1}, W^{-1})\log\frac{1}{\epsilon})$ iterations are required to obtain an $\epsilon$-optimal solution. 
%For typical SSL problems with normalized Laplacian regularization \textcolor{red}{not clear what you mean here? and PR in Section~\ref{application}, 
%we conclude the complexity of RCD in the following corollary.}  
Below we give the explicit characterization of the complexity for the SSL and PR problems with normalized Laplacian regularization as discussed in Section 3.%,  in terms of the graph/hypergraph parameters.
\begin{corollary}~\label{specialcase}
Suppose that $W = \beta D$, where $\beta$ is a hyper-parameter, and $D$ is a diagonal degree matrix such that $D_{ii} = \sum_{r:\in[R], i\in S_r} \max_{S\subseteq V} [F_r(S)]^2$.  Algorithm 1  requires an expected number of $O(N^2R\max\{1,9\beta^{-1}\}\max_{i,j\in[N]}\frac{D_{ii}}{D_{jj}}\log\frac{1}{\epsilon})$ iterations to return an $\epsilon$-optimal solution. 
%\textcolor{red}{as you need this many iterations, you may mention something less precise in the termination criteria of Alg 1?} \textcolor{blue}{This is only for the special case based on normalized Laplacian regularization.}
\end{corollary}
The term $N^2R$ also appears in the expression for the complexity of the RCD method for solving the DSFM problem~\cite{ene2017decomposable}. The term $\max\{1,9\beta^{-1}\}$ implies that whenever $\beta$ is small, the convergence rate is slow. This makes sense: for example, in the PR problem~\eqref{PReq}, a small $\beta$ corresponds to a large $\alpha$, which typically implies longer mixing times of the underlying Markov process. The term $\max_{i,j\in[N]}\frac{D_{ii}}{D_{jj}}$ arises due to the degree-based normalization.

\section{Computing the Projections $\Pi_{C_r}(\cdot)$}\label{innerloop}
In this section, we provide efficient routines for computing the projection onto the conic set $\Pi_{C_r}(\cdot)$. As the procedure works for all values of $r\in[R]$, we drop the subscript $r$ for simplicity of notation. % Moreover, as the weighted norm in this section does not function in any parameterization, for notational simplicity, we consider to solve 
First, recall that 
\begin{align}\label{projection}
\Pi_{C}(a)  = \arg\min_{(y,\phi)} h(y,\phi) \triangleq \|y- a\|_{\tilde{W}}^2 + \phi^2 \quad \text{s.t. $y\in \phi B$, $\phi\geq 0$},
\end{align} 
where $\tilde{W} =W^{-1}$, and where $B$ denotes the base polytope of the submodular function $F$. Let $h^*$ and $(y^*, \phi^*)$ be the optimal value of the objective function and the argument that optimizes it, respectively. When performing projections, one only needs to consider the variables incident to $F$, and set all other variables to zero. For ease of exposition, we assume that all variables in $[N]$ are incident to $F.$ 

%whose iterations are cheap but may converge quite slow in practice~\cite{jegelka2013reflection} where each iteration is expensive but exact solution can be achieved in finite time
Unlike QDSFM, the DSFM involves the computation of projections onto the base polytopes of submodular functions. Two algorithms,  the Frank-Wolfe (FW) method~\cite{frank1956algorithm} and  the Fujishige-Wolfe minimum norm algorithm (MNP)~\cite{fujishige2011submodular}, are used for this purpose. Both methods assume cheap linear minimization oracles on polytopes and attain a $1/k$-convergence rate. The MNP algorithm is more sophisticated and empirically more efficient. Nonetheless, neither of these methods can be applied directly to cones. To this end, we modify these two methods by adjusting them to the conic structure in~\eqref{projection}  and show that a $1/k-$convergence rate still holds. We refer to the procedures as \emph{the conic MNP method} and \emph{the conic FW method}, respectively. Here we focus mainly on the conic MNP method described in Algorithm 2, as it is more sophisticated. A detailed discussion of the conic FW method and its convergence guarantees can be found in the Supplement (see Section~\ref{appsec:FW}).
\begin{table}[htb]
\centering
\begin{tabular}{l}
\hline
\label{generalalg}
\textbf{Algorithm 2: } \textbf{The Conic MNP Method for Solving~\eqref{projection}} \\
\hline
\textbf{Input}:  $\tilde{W}$, $a$, $B$ and a small positive constant $\delta$. \textbf{Maintain} $\phi^{(k)} = \sum_{q_i\in S^{(k)}}\lambda_i^{(k)}$\\
Choose an arbitrary $q_1\in B$. Set $S^{(0)} \leftarrow \{q_1\}$, $\lambda_1^{(0)}\leftarrow \frac{\langle a, q_1\rangle_{\tilde{W}}}{1+ \|q_1\|_{\tilde{W}}^2}$, $y^{(0)} \leftarrow  \lambda_1q_1$, $k\leftarrow 0$ \\
%$y^{(0)} \leftarrow  \phi^{(0)}q_1$ for some arbitrary $q_1\in B$, $ \phi^{(0)}\leftarrow \langle a, q_1\rangle_{\tilde{W}}/(1+ \|q_1\|_{\tilde{W}}^2)$, $S \leftarrow \{q_1\}$, $\lambda_1\leftarrow 1$, $k\leftarrow 0$.  \\ 
1. Iteratively execute (\textbf{MAJOR LOOP}): \\
2. \; $q^{(k)} \leftarrow \arg\min_{q\in B} \langle  \nabla_{y} h(y^{(k)}, \phi^{(k)}), q \rangle_{\tilde{W}}$\\
3. \; \textbf{If} $\langle y^{(k)} - a, q^{(k)} \rangle_{\tilde{W}} + \phi^{(k)}\geq -\delta$, then \textbf{break}; \textbf{Else} $S^{(k)} \leftarrow S^{(k)}\cup \{q^{(k)}\}$.  \\
%4. \;\; $S \leftarrow S\cup \{q^{(k)}\}$\\
4. \;\;\;\;\;\;\; Iteratively execute (\textbf{MINOR LOOP}):\\
5. \;\;\;\;\;\;\;\;\;\; $\alpha \leftarrow  \arg\min_{\alpha} \|\sum_{q_i^{(k)}\in S^{(k)}} \alpha_i q_i^{(k)} - a\|_{\tilde{W}}^2 + (\sum_{q_i^{(k)}\in S} \alpha_i)^2$, $z^{(k)} \leftarrow \sum_{q_i^{(k)}\in S} \alpha_i q_i^{(k)}$ \\
6. \;\;\;\;\;\;\;\;\;\;  \textbf{If} $\alpha_i \geq 0$ for all $i$ then \textbf{break} \\
7. \;\;\;\;\;\;\;\;\;\;  \textbf{Else} $\theta = \min_{i: \alpha_i < 0} \lambda_i/(\lambda_i - \alpha_i)$, \; $\lambda_i^{(k+1)} \leftarrow \theta\alpha_i + (1-\theta)\lambda_i^{(k)}$,\\
%9. \;\;\;\;\;\;\;\;\;\;\;\;\;\;\;\;\;\;\;\;\;  $\lambda_i^{(k+1)} \leftarrow \theta\alpha_i + (1-\theta)\lambda_i^{(k)}$, $S^{(k+1)}\leftarrow\{i: \lambda^{(k+1)}> 0\}$\\
8. \;\;\;\;\;\;\;\;\;\;\;\;\;\;\;\; $y^{(k+1)} \leftarrow \theta z^{(k)} + (1-\theta)y^{(k)}$, \; $S^{(k+1)}\leftarrow\{i: \lambda^{(k+1)}> 0\}$,  \;$k\leftarrow k+1$ \\
9.\; $y^{(k+1)} \leftarrow z^{(k)}$, $\lambda^{(k+1)} \leftarrow \alpha$, $S^{(k+1)}\leftarrow\{i: \lambda^{(k+1)} > 0\}$, $k\leftarrow k+1$\\
\hline
\end{tabular}
\end{table}

The conic MNP algorithm keeps track of an \emph{active set} $S =\{q_1, q_2,...\}$ and searches for the best solution in its conic hull. Let us denote the cone of an active set $S$ as $\text{cone}(S) = \{\sum_{q_i\in S}\alpha_i q_i| \alpha_i \geq 0\}$ and its linear set as $\text{lin}(S) = \{\sum_{q_i\in S}\alpha_i q_i|\alpha_i\in\mathbb{R}\}$. Similar to the original MNP algorithm,  %at every iteration $k$, Algorithm 2 finds a solution $y^{(k)}$ that lies in the $cone(S)$, i.e., $y^{(k)} = \sum_{q_i\in S} \lambda_i q_i,$ with $\lambda>0$. 
Algorithm 2 also contains two level-loops: MAJOR and MINOR. In the MAJOR loop, we greedily add a new active point $q^{(k)}$ to the set $S$ obtained from the linear minimization oracle w.r.t. the base polytope (Step 2), and by the end of the MAJOR loop, we obtain a $y^{(k+1)}$ that minimizes $h(y,\phi)$ over $\text{cone}(S)$ (Step 3-8).  The MINOR loop is activated when $\text{lin}(S)$ contains some point $z$ that guarantees a smaller value of the objective function than that of the optimal point in $cone(S)$, provided that some active points from $S$ may be removed. Compared to the original MNP method, Steps 2 and 5 as well as the termination Step 3 are specialized for the conic structure.
The following convergence result implies that the conic MNP algorithm also has a convergence rate of order $1/k$; the proof is essentially independent on the submodularity assumption and represents a careful modification of the arguments in~\cite{chakrabarty2014provable} for conic structures. 

\begin{theorem}\label{Wolfegaurantee}
Let $B$ be an arbitrary polytope in $\mathbb{R}^{N}$ and let $C = \{(y,\phi)| y\in \phi B, \phi\geq 0\}$ be the cone induced by the polytope. For some positive diagonal matrix $\tilde{W}$, define $Q = \max_{q\in B} \|q\|_{\tilde{W}}$. Algorithm 2 yields a sequence of $(y^{(k)}, \phi^{(k)})_{k=1,2,...}$ such that $h(y^{(k)}, \phi^{(k)})$ decreases monotonically. Algorithm 2 terminates when $k= O(N\|a\|_{\tilde{W}}\max\{Q^2,1\}/\delta),$ with $h(y^{(k)}, \phi^{(k)}) \leq h^* + \delta\|a\|_{\tilde{W}}$.
\end{theorem}
 %Although it shows to have best empirical performance, the convergence rate $O(\frac{1}{t})$ has not been proven until recently~\cite{chakrabarty2014provable}. 
%\textcolor{red}{which is equivalent to $\langle y -a, rq - y \rangle \geq 0$ and $r = \hat{y}^T a /(1 + \|\hat{y}\|^2)$.} \\
%Decreasing property of Wolfe's alg can be guaranteed. $O(1/t)$ convergence?(Just check the proof in~\cite{chakrabarty2014provable}, which can be extended to our case), Compared to Away-steps, the size of $S$ is bounded by $|e|$. Compared to Fully corrective Non-negative matching pursuit, Wolfe does not required to solve an inner loop nonnegative quadratic programming. 

%Step 4 requires $O(N\log N + N EO)$ to determine $ \inf_{\phi}\{y \in \phi B\}$: first, sort the entries in $y$ in a decreasing order $y_{i_1} \geq y_{i_2} \geq ... \geq y_{i_{N}}$. If $y_{i_{1}} = 0$, then $\phi = 0$. Otherwise, $\phi = \min_{k\in[N-1]} \frac{\sum_{j =1}^{k}y_{i_j}}{F(\{i_1,...,i_k\})}$. In all, each iteration is with complexity $O(N\log N + NEO)$.
%Part of the idea is inspired by the algorithm for undirected hyperedge~\cite{hein2013total}.

Both the (conic) FW and MNP are approximate algorithms for computing the projections for generic polytopes $B$ and their induced cones. We also devised an algorithm of complexity $O(N\log N)$ that \emph{exactly} computes the projection for polytopes $B$ arising in learning on (un)directed hypergraphs. A detailed description of the algorithm for exact projections is described in Section~\ref{appsec:exactprojection} of the Supplement. %\Niao{Perhaps move the discussion of iteration computation cost for these two algorithms back here if space allowed? Or at least mention it.}
%As shown in Section 3, this special case may cover a large number of QDSFM problems in practice, as directed or undirected edges or hyperedges are all its special cases. 
%\newpage
\section{Extension to mix-DSFM}
With the tools to solve both QDSFM and DSFM problems, it is simple to derive an efficient solver for the following mix-DSFM problem: Suppose $\{F_r\}_{r\in[R_1+R_2]}$ are a collection of submodular functions where $F_{r}\geq 0$ for $r\in[R_1]$. Let $f_r$ be the corresponding Lov{\'a}sz extension of $F_r$, $r\in [R_1+R_2]$. We are to solve
\begin{align}\label{mix-DSFM}
\text{mix-DSFM:} \quad\quad \min_{x\in \mathbb{R}^N}\|x - a\|_W^2 + \sum_{r=1}^{R_1}\left[f_{r}(x)\right]^2+\sum_{r=R_1+1}^{R_1+R_2}f_r(x)
\end{align}
By using the same trick in~\eqref{conjugate} for the quadratic term, one may show the dual problem of mix-DSFM is essentially to find the best approximation of an affine space in terms of a mixture product of cones and base polytopes. Furthermore, all other related results, including the weak-strong duality of the  dual, the linear convergence of RCD/AP and the $1/k$-rate convergence of the MNP/FW methods can be generalized to the mix-DSFM case via the same technique developed in this work.

\section{Experiments}
Our dataset experiments focus on SSL learning for hypergraphs on both real and synthetic datasets. For the particular problem at hand, the QDSFM problem can be formulated as follows 
\begin{align}\label{expobj}
\min_{x\in \mathbb{R}^N} \beta \|x - a\|^2 + \sum_{r\in [R]}\max_{i,j\in S_r}(\frac{x_i}{\sqrt{W_{ii}}} - \frac{x_j}{\sqrt{W_{jj}}})^2,
\end{align}
where $a_i \in\{-1, 0, 1\}$ indicates if the corresponding variable $i$ has a negative, missing, or positive label, respectively, $\beta$ is a parameter that balances out the influence of observations and the regularization term,  $\{W_{ii}\}_{i\in[N]}$ defines a positive measure over the variables and may be chosen to be the degree matrix $D$ with $D_{ii} = |\{r\in [R]: i \in S_r\}|$. Each part in the decomposition corresponds to one hyperedge. We compare eight different solvers falling into three categories: (a) our proposed general QDSFM solvers, \emph{QRCD-SPE, QRCD-MNP, QRCD-FW and QAP-SPE}; (b) alternative solvers for the specific problem~\eqref{expobj}, including \emph{PDHG}~\cite{hein2013total}  and \emph{SGD}~\cite{zhang2017re}; (c) SSL solvers that do not use QDSFM as the objective, including \emph{DRCD}~\cite{ene2015random} and  \emph{InvLap}~\cite{zhou2007learning}. The first three methods all have outer-loops that execute RCD, but with different inner-loop projections computed via the \emph{exact projection algorithm} for undirected hyperedges (see Section~\ref{appsec:exactprojection} in the Supplement), or the generic MNP and FW. The QAP-SPE method uses AP in the outer-loop and exact inner-loop projections. PDHG and SGD are the only known solvers for the specific objective~\eqref{expobj}. %PDHG, QAP and QRCD are primal-dual methods, we simply use the designation ``gap'' to refer to the achievable primal gap of a particular method. SGD works directly on the primal problem so the gap in this setting refers to the distance from the optimal value of the objective. 
%We also consider two other solvers that are not particularly suited for the QDSFM problem but are still used for SSL on hypergraphs:  DRCD and InvLap. 
%\textcolor{red}{References? Also, we are drowning in acronyms here, and for some we do not have the full naming} \textcolor{blue}{The references and reasons of the acronyms are in the next description}. 
DRCD is a state-of-the-art solver for DSFM and also uses a combination of outer-loop RCD and inner-loop projections. InvLap first transforms hyperedges into cliques and then solves a Laplacian-based linear problem. All the aforementioned methods, except InvLap, are implemented via C++ in a nonparallel fashion. InvLap is executed via matrix inversion operations in Matlab which may be parallelized. The CPU times of all methods are recorded on a 3.2GHz Intel Core i5. The results are summarized for $100$ independent tests. When reporting the results, we use the primal gap (``gap'') to characterize the convergence properties of different solvers. Additional descriptions of the settings and experimental results for the QRCD-MNP and QRCD-FW methods for general submodular functions may be found in the Supplement. \\
\textbf{Synthetic data.} We generated a hypergraph with $N = 1000$ vertices that belong to two equal-sized clusters. We uniformly at random generated $500$ hyperedges within each cluster and $1000$ hyperedges across the two clusters. Note that in higher-order clustering, we do not need to have many hyperedges within each cluster to obtain good clustering results. Each hyperedge includes $20$ vertices. We also uniformly at random picked $l=1,2,3,4$ vertices from each cluster to represent labeled datapoints. With the vector $x$ obtained by solving~\eqref{expobj}, we classified the variables based on the Cheeger cut rule~\cite{hein2013total}: suppose that 
$\frac{x_{i_1}}{\sqrt{W_{i_1i_1}}}\geq \frac{x_{i_2}}{\sqrt{W_{i_2i_2}}}\geq \cdots \geq \frac{x_{i_N}}{\sqrt{W_{i_Ni_N}}},$
and define $\mathcal{S}_j = \{i_1, i_2, ..., i_j\}$. We partition $[N]$ into two sets $(\mathcal{S}_{j^*}, \bar{\mathcal{S}}_{j^*}),$ where 
$$j^* = \arg\min_{j\in[N]} c(\mathcal{S}_j)\triangleq \frac{| S_r \cap \mathcal{S}_j\neq \emptyset, S_r \cap \bar{\mathcal{S}}_j \neq \emptyset\}|}{\max\{\sum_{r\in[R]} |S_r \cap \mathcal{S}_j|, \sum_{r\in[R]} |S_r \cap \bar{\mathcal{S}}_j|\}}.$$ 
The classification error is defined as (\# of incorrectly classified vertices)$/N$. In the experiment, we used $W_{ii} = D_{ii}$, $\forall\,i,$ and tuned 
$\beta$ to be nearly optimal for different objectives with respect to the classification error rates. 

The top-left figure in Figure~\ref{fig:synthetic1} shows that QRCD-SPE  converges much faster than all other methods when solving the problem~\eqref{expobj} according to the gap metric (we only plotted the curve for $l=3$ as all other values of $l$ produce similar patterns). To avoid clutter, we postpone the results for QRCD-MNP and QRCD-FW to the Supplement (see Section~\ref{sec:mnp-fw}), as these methods are typically $100$ to $1000$ times slower than QRCD-SPE. In the table that follows, we describe the performance of different methods with similar CPU-times. Note that when QRCD-SPE converges (with primal gap $10^{-9}$ achieved after $0.83$s), the obtained $x$ leads to a much smaller classification error than other methods. QAP-SPE, PDHG and SGD all have large classification errors as they do not converge within short CPU time-frames. QAP-SPE and PDHG perform only a small number of iterations, but each iteration computes the projections for all the hyperedges, which is more time-consuming. The poor performance of DRCD implies that the DFSM is not a good objective for SSL. %although the large number of iterations indicates that the inner-loop projection onto the base polytopes are computationally cheaper than those performed on cones.
InvLap achieves moderate classification errors, but still does not match the performance of QRCD-SPE. %Note that InvLap uses Matlab, which is optimized for matrix operations, and is hence fairly efficient. 
Note that InvLap uses Matlab, which is optimized for matrix operations, and is hence fairly efficient.
%Note that InvLap is extremely efficient in this setting.
However, for experiments on real datasets, where one encounters fewer but significantly larger hyperedges, InvLap does not offer as good a performance as the one for synthetic data. The column ``Average $100c(\mathcal{S}_{j^*})$'' also illustrates that the QDSFM objective is a good choice for finding approximate balanced cuts of hypergraphs.  

\begin{figure*}[t]
\centering
\includegraphics[trim={0cm 0cm 0cm 0},clip,width=.27\textwidth]{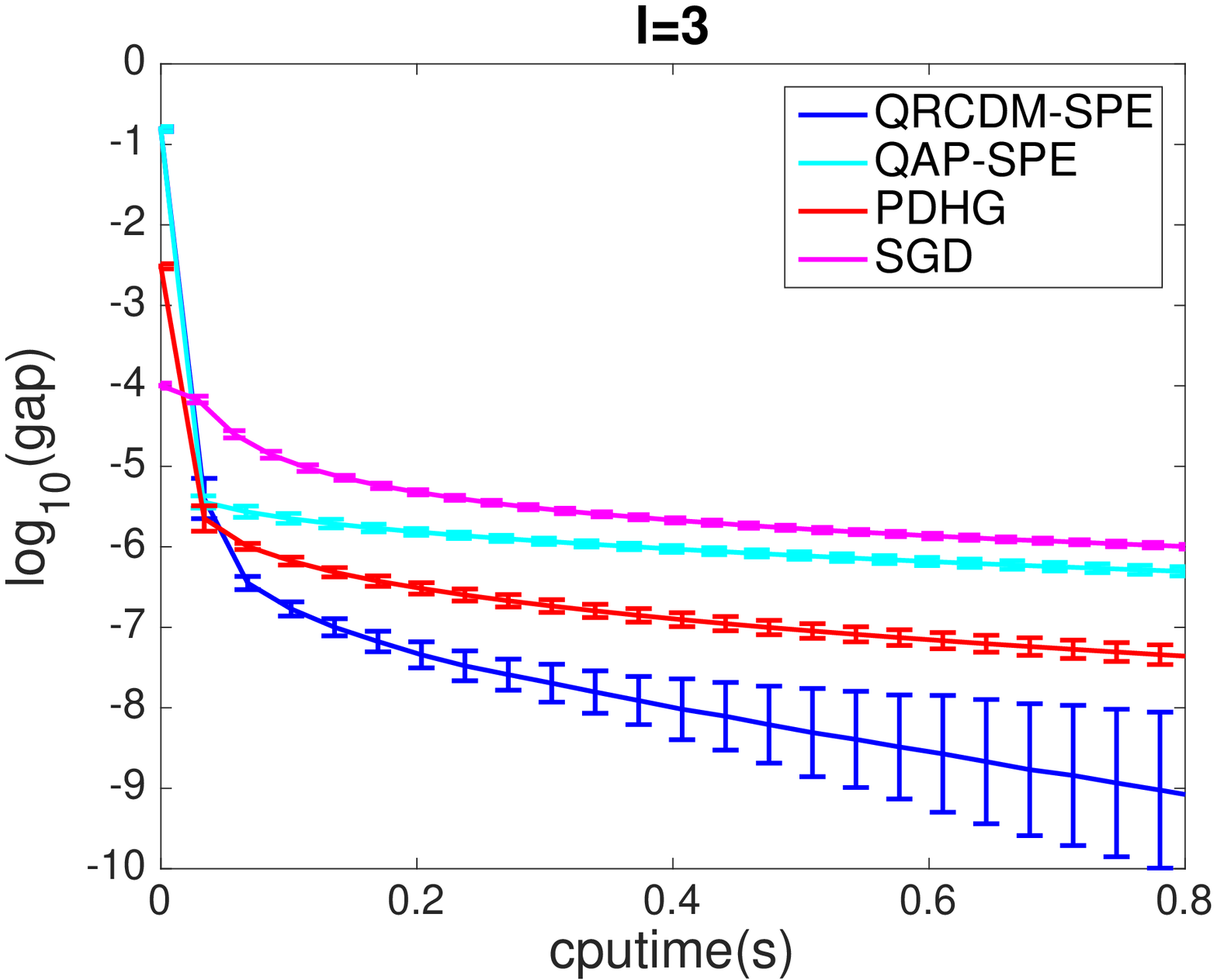}
\hspace{.005\textwidth}
\rule{0.005cm}{3cm}
\hspace{.005\textwidth}
\includegraphics[trim={0cm 0cm 0.0cm 0.0cm},clip,width=.27\textwidth]{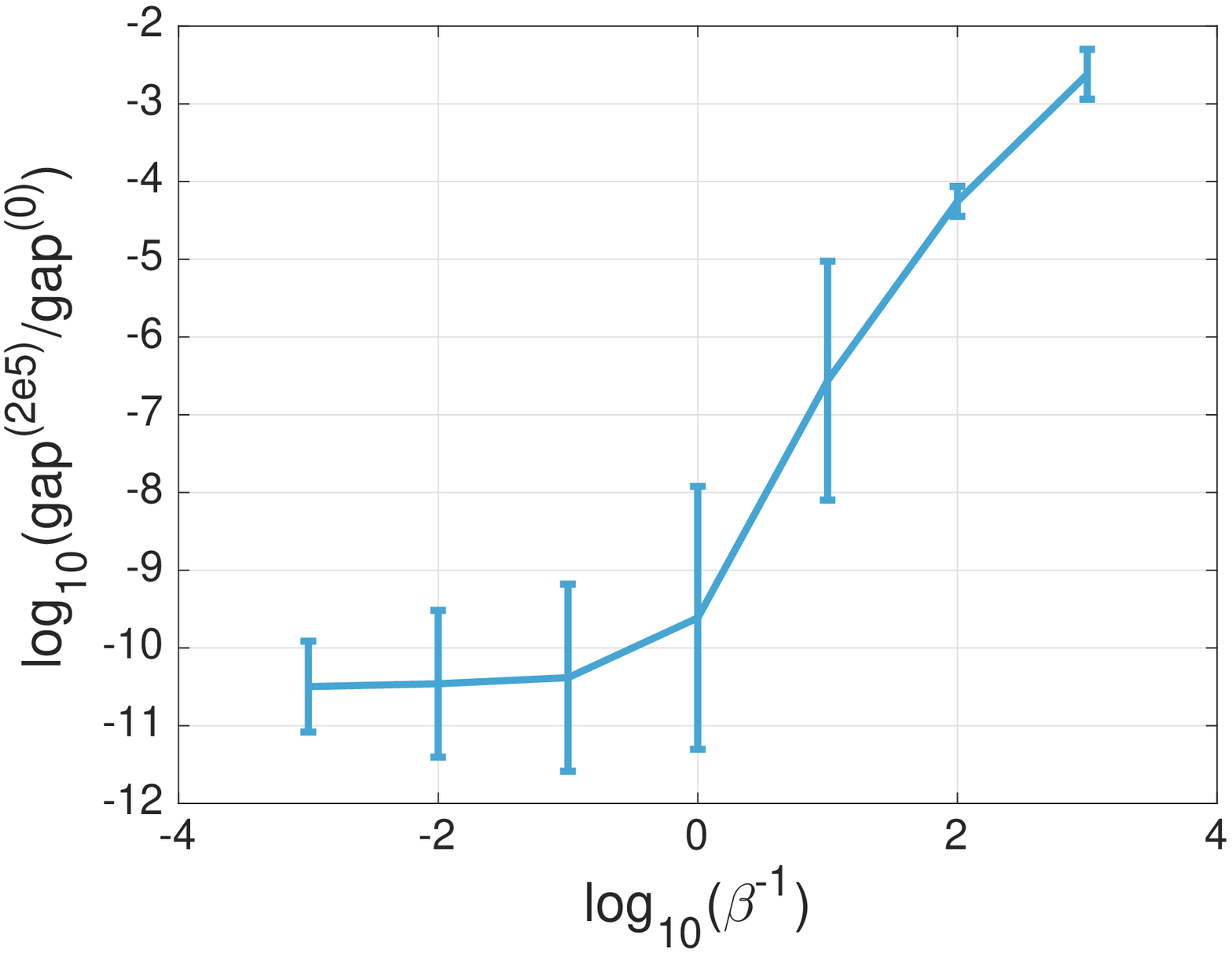}
\includegraphics[trim={0cm 0cm 0.0cm 0.0cm},clip, width=.27\textwidth]{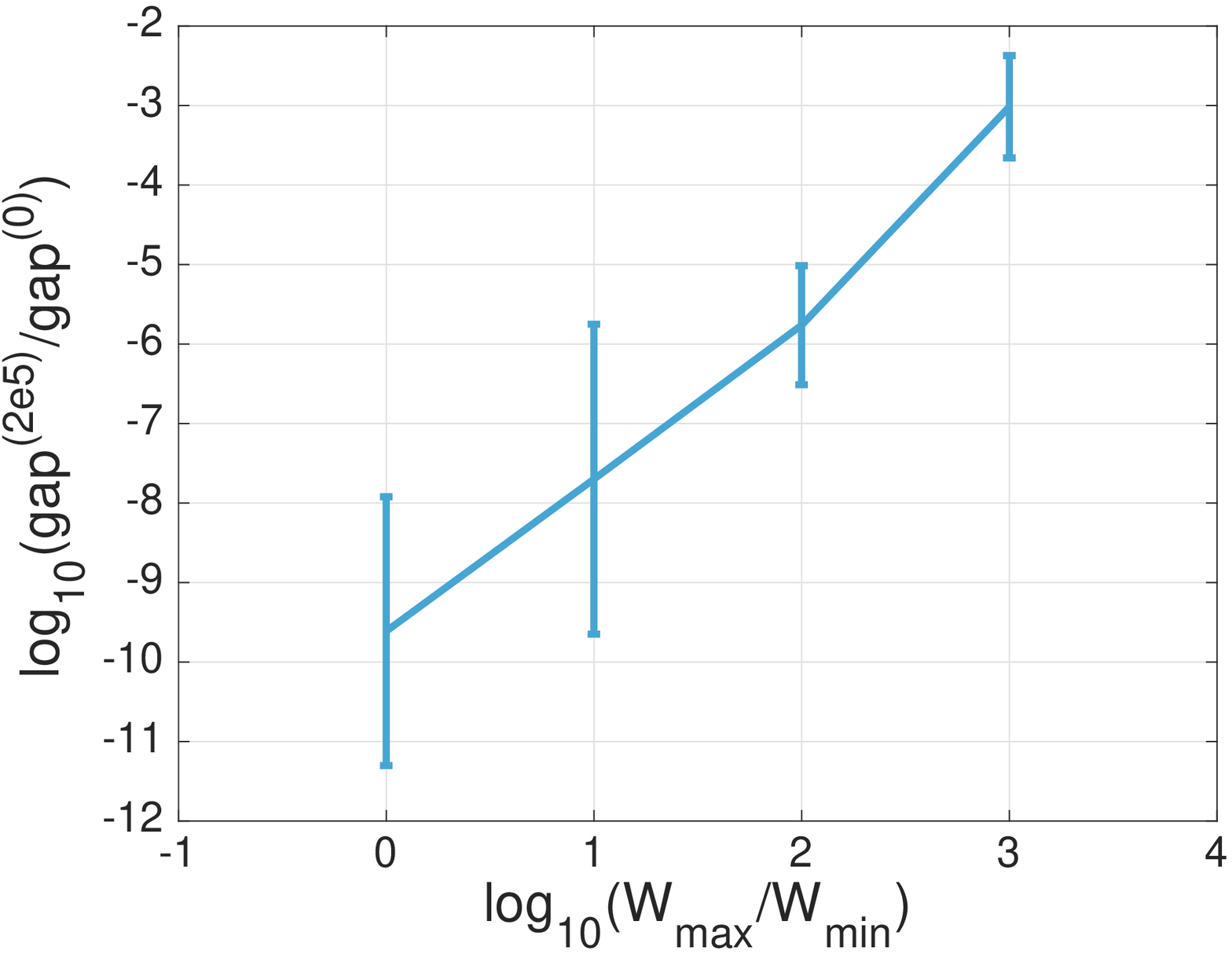}\\
%\caption{Image segmentation: First row: Gaps vs the number of iterations $\times \alpha$. Second row: The number of iterations $\times \alpha$ vs $\alpha$. }
\vspace{0.1cm}
\tiny{
\begin{tabular}{p{0.25cm}<{\centering}|p{1.1cm}<{\centering}| p{0.25cm}<{\centering}| p{0.25cm}<{\centering}| p{0.25cm}<{\centering}| p{0.25cm}<{\centering}| p{0.25cm}<{\centering}| p{0.25cm}<{\centering}| p{0.25cm}<{\centering}| p{0.25cm}<{\centering}|  p{0.25cm}<{\centering}| p{0.25cm}<{\centering}| p{0.25cm}<{\centering}| p{0.25cm}<{\centering}| p{1.0cm}<{\centering}| p{0.8cm}<{\centering}}
\hline
\multirow{3}{*}{Obj.} &\multirow{3}{*}{Solvers}& \multicolumn{8}{c|}{Classification error rate (\%)}  & \multicolumn{4}{c|}{Average 100$c(\mathcal{S}_{j^*})$} &  \multirow{3}{*}{\#iterations} &  \multirow{3}{*}{cputime(s)} \\
\cline{3-14}
&&  \multicolumn{2}{c|}{l=1} & \multicolumn{2}{c|}{l=2} &\multicolumn{2}{c|}{l=3} & \multicolumn{2}{c|}{l=4} & \multirow{2}{*}{l=1} & \multirow{2}{*}{l=2} & \multirow{2}{*}{l=3} & \multirow{2}{*}{l=4}  & & \\
\cline{3-10}
&& MN & MD & MN & MD & MN & MD & MN & MD & & & & & &    \\
\hline
\parbox[t]{2mm}{\multirow{4}{*}{\rotatebox[origin=c]{90}{QDSFM}}} &QRCD-SPE & \textbf{2.93} & \textbf{2.55}  & \textbf{2.23} & \textbf{0.00} & \textbf{1.47} & \textbf{0.00} & \textbf{0.78} & \textbf{0.00} & \textbf{6.81} & \textbf{6.04} & \textbf{5.71} & \textbf{5.41} & $4.8\times 10^5$ & 0.83  \\
&QAP-SPE & 14.9 & 15.0  & 12.6 & 13.2 & 7.33 & 8.10 & 4.07 & 3.80 & 9.51 & 9.21 & 8.14 & 7.09 & $2.7\times 10^2$ & 0.85  \\
&PDHG & 9.05 & 9.65 & 4.56 & 4.05 & 3.02 & 2.55 & 1.74 & 0.95  & 8.64 & 7.32 & 6.81 & 6.11 & $3.0\times 10^2$ & 0.83  \\
&SGD & 5.79 & 4.15 & 4.30 & 3.30 & 3.94 & 2.90 & 3.41 & 2.10 & 8.22 & 7.11 & 7.01 & 6.53 & $1.5 \times 10^4$ &0.86  \\
\cline{1-16}
\multirow{2}{*}{\rotatebox[origin=c]{90}{\tiny{Oth.}}}  &DRCD & 44.7 &44.2 & 46.1 & 45.3 & 43.4 & 44.2 & 45.3 & 44.6 & 9.97 & 9.97 & 9.96 & 9.97 & $3.8 \times 10^6$  &0.85   \\
&InvLap & 8.17 & 7.30 & 3.27  & 3.00 & 1.91 & 1.60 & 0.89 &  0.70 & 8.89 & 7.11 & 6.18 & 5.60 & --- &\textbf{0.07} \\
\hline
\end{tabular}
}
\caption{\footnotesize{Experimental results on synthetic datasets. Top-left: gap vs CPU-time of different QDSFM solvers (with an average $\pm$ standard deviation). Bottom: classification error rates $\&$ Average 100 $c (\mathcal{S}_{j^*})$ for different solvers (MN: mean, MD: median). Top-right: the rate of a primal gap of QRCD after $2\times 10^{5}$ iterations with respect to different choices of the parameters $\beta$ $\&$ $\max_{i,j\in [N]}W_{ii}/W_{jj}$.}} 
\label{fig:synthetic1}
\end{figure*}

We also evaluated the influence of parameter choices on the convergence of QRCD methods. According to Theorem~\ref{linearconv}, the required number of RCD iterations for achieving an $\epsilon$-optimal solution for~\eqref{expobj} is roughly $O(N^2R \max(1, 9/(2\beta))\max_{i, j\in[N]}W_{ii}/W_{jj}\log 1/\epsilon)$ (see Section~\ref{sec:para-dep} in the Supplement). We mainly focus on testing the dependence on the parameters $\beta$ and $\max_{i, j\in[N]}W_{ii}/W_{jj}$, as the term $N^2R$ is also included in the iteration complexity of DSFM and was shown to be necessary given certain submodular structures~\cite{li2018revisiting}. To test the effect of $\beta$, we fix $W_{ii} = 1$ for all $i$, and vary $\beta \in [10^{-3}, 10^{3}]$. To test the effect of $W$, we fix $\beta = 1$ and randomly choose half of the vertices and set their $W_{ii}$ values to lie in $\{1, 0.1, 0.01, 0.001\}$, and set the remaining ones to $1$. The two top-right plots of Figure.~\ref{fig:synthetic1} show our results. The logarithm of gap ratios is proportional to $\log\beta^{-1}$ for small $\beta,$ and $\log \max_{i, j\in[N]}W_{ii}/W_{jj}$, which is not as sensitive as predicted by Theorem~\ref{linearconv}. Moreover, when $\beta$ is relatively large ($>1$), the dependence on $\beta$ levels out. 
% and large values of $\max_{i, j\in[N]}\frac{W_{ii}}{W_{jj}}$ tend to decrease the convergence rate. 
\begin{figure*}[t]
\centering
\includegraphics[trim={0cm 0cm 0cm 0},clip,width=.27\textwidth]{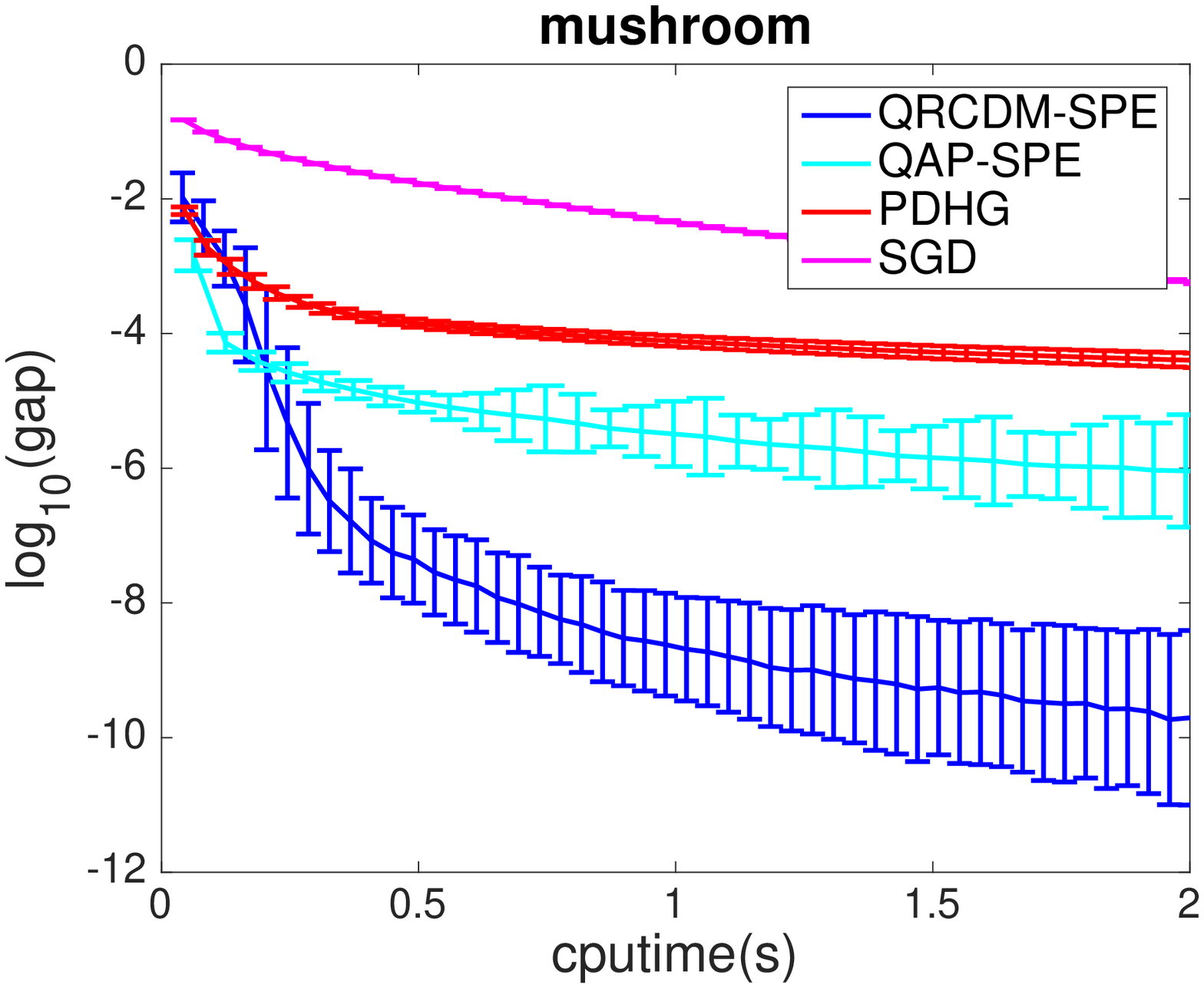}
\includegraphics[trim={0cm 0cm 0cm 0},clip,width=.27\textwidth]{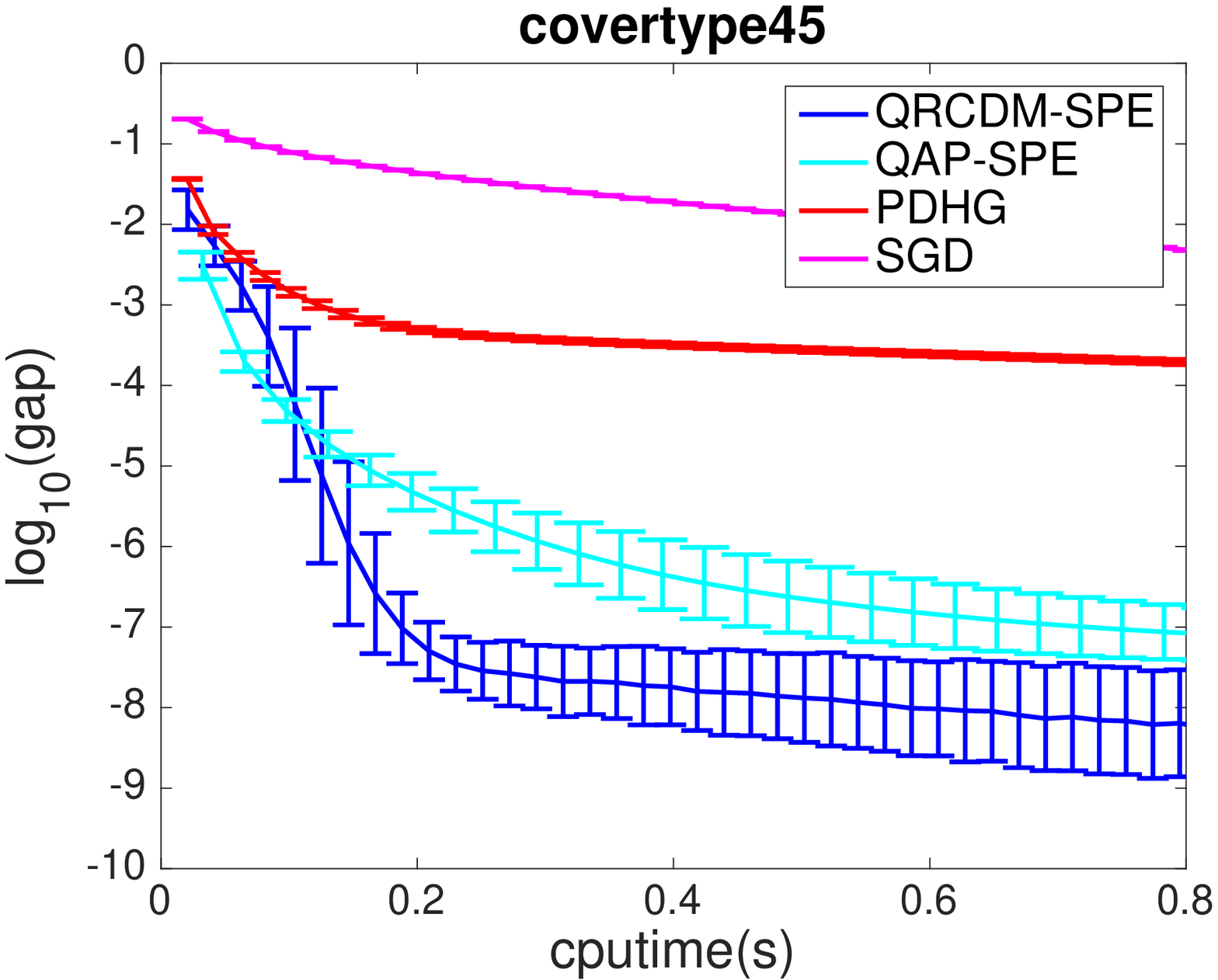}
\includegraphics[trim={0cm 0cm 0cm 0},clip,width=.27\textwidth]{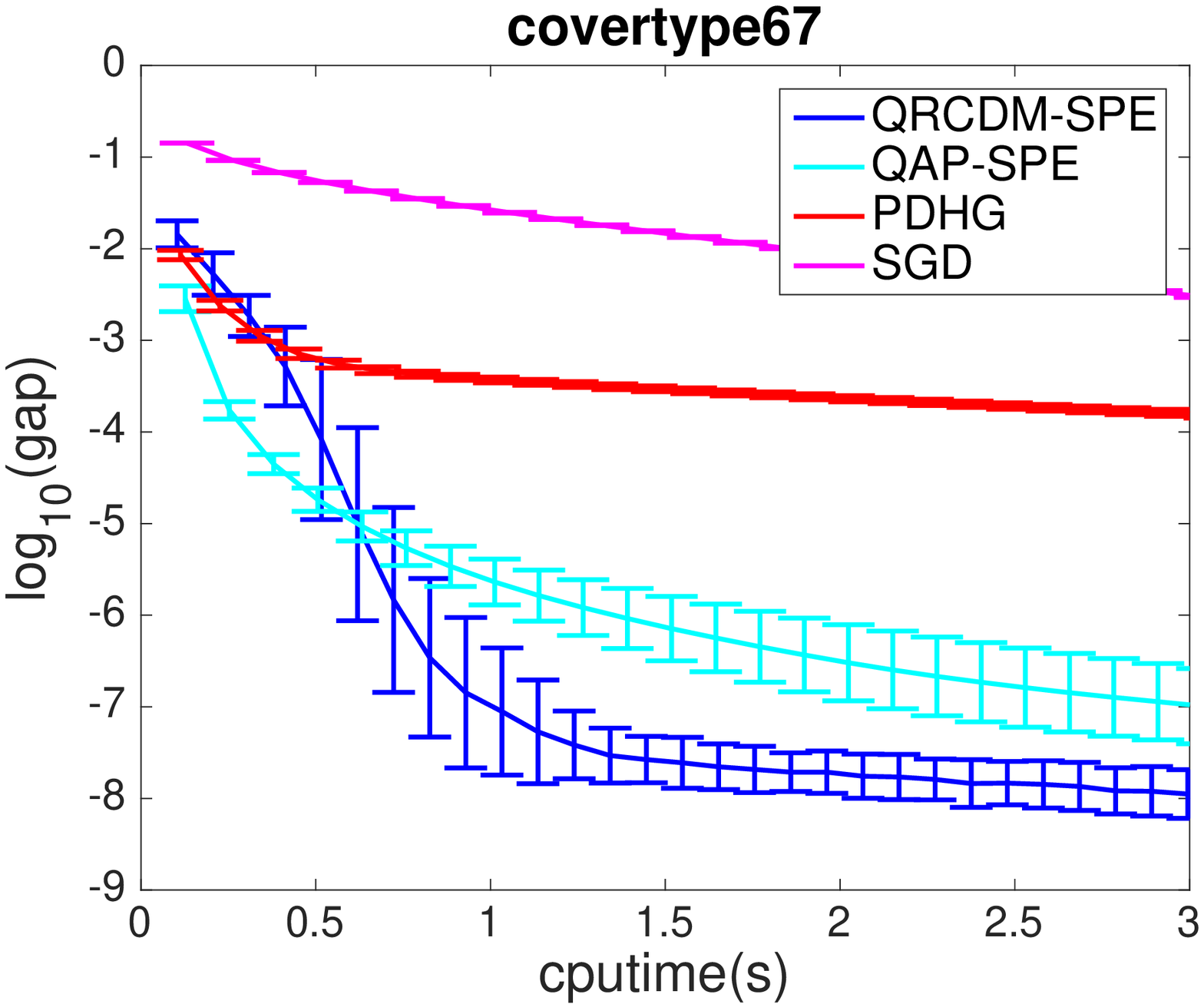}
\caption{\footnotesize{Convergence of different solvers for QDFSM over three different real datasets.}}
\label{fig:realdata}
\end{figure*}

\textbf{Real data.} We also evaluated the proposed algorithms on three UCI datasets: \emph{Mushroom}, \emph{Covertype45}, \emph{Covertype67}, used as standard datasets for SSL on hypergraphs~\cite{zhou2007learning, hein2013total, zhang2017re}. 
Each dataset corresponds to a hypergraph model as described in~\cite{hein2013total}: entries correspond to vertices while each categorical feature is modeled as one hyperedge; numerical features are first quantized into $10$ bins of equal size, and then mapped to hyperedges. Compared to synthetic data, in this datasets, the size of most hyperedges is much larger ($\geq$ 1000) while the number of hyperedges is small ($\approx 100$). Previous works have been shown that fewer classification errors can be achieved by using QDSFM as an objective instead of DSFM or InvLap~\cite{hein2013total}. In our experiment, we focused on comparing the convergence of different solvers for QDSFM. We set $\beta = 100$ and $W_{ii} = 1,$ for all $i,$ and set the number of observed labels to $100$, which is a proper setting as described in~\cite{hein2013total}. Figure.~\ref{fig:realdata} shows the results. Again, the proposed QRCD-SPE and QAP-SPE methods both converge faster than PDHG and SGD, while QRCD-SPE performs the best. Note that we did not plot the results for QRCD-MNP and QRCD-FW as the methods converge extremely slowly due to the large sizes of the hyperedges. InvLap requires $22$, $114$ and $1802$ seconds to run on the Mushroom, Covertype45 and Covertype67 datasets, respectively. Hence, the methods do not scale well. 
%\textbf{Other experiments.} There are more experiments in supplement to evaluate the convergence of QDSFM-MNP and QDSFM-FW that work for more general submodular functions.

\section{Acknowledgement}
The authors gratefully acknowledge many useful suggestions by the reviewers. This work was supported in part by the NIH grant 1u01 CA198943A and the NSF grant CCF 15-27636. 

\bibliographystyle{IEEETran}
\bibliography{example_paper}
 
 \newpage
 
\appendix

\begin{appendix}

\begin{center}
{\Large \textbf{Supplement}}
\end{center}
\end{appendix}

Here, we provide detailed proofs of the main lemmas and theorems, and described implementation details for the numerical experiments. We also present the convergence analysis for the Alternating Projection algorithm for QDSFM and the conic Frank-Wolfe method for computing the projection onto the induced cone, as well as an exact projection algorithm when learning on directed/undirected hypergraphs. 

\section{Proofs for Dual Formulation}
% \subsection{Proof of Lemma~\ref{conjugate}}
% 
%The right hand side of the equation equals 
%$$\max_{\phi_r\geq 0}\max_{y_r\in \phi_r B_r}\; \langle y_r, x\rangle - \frac{\phi_r^2}{4}  = \max_{\phi_r\geq 0} \phi_r f_r(x) -  \frac{ \phi_r^2}{4}  = [f_r(x)]^2.$$ 
%This proves the claimed result.

\subsection{Proof of Lemma~\ref{dualform}}
In all following derivation, we may exchange the order of minimization and maximization (i.e., $\min \max = \max \min$) due to Proposition 2.2~\cite{ekeland1999convex}. Plugging equation~\eqref{conjugate} into \eqref{QDSFM}, we obtain 
\begin{align*}
 &\min_x    \sum_{r\in [R]} [f_r(x)]^2 + \| x- a\|_{W}^2 \\
 =&  \min_x  \max_{\phi_r \geq 0, y_r\in \phi_r B_r} \sum_{r\in [R]}  \left[\langle y_r, x\rangle - \frac{\phi_r^2}{4} \right] + \| x- a\|_{W}^2 \\
  = &   \max_{\phi_r \geq 0, y_r\in \phi_r B_r} \min_x  \sum_{r\in [R]}  \left[\langle y_r, x\rangle - \frac{\phi_r^2}{4} \right] + \| x- a\|_{W}^2 \\
 = &\max_{\phi_r \geq 0, y_r\in \phi_r B_r}   -\frac{1}{4}\|\sum_{r\in [R]}  y_r - 2Wa\|_{W^{-1}}^2- \frac{1}{4} \sum_{r}\phi_r^2 + \|a\|_W^2.  
 \end{align*}
 By eliminating some constants, one obtains the dual~\eqref{CDform}. Next, we prove that the problem~\eqref{APform} is equivalent to~\eqref{CDform}, as follows from removing $\lambda_r$. Note that~\eqref{APform} is equivalent to
 \begin{align*}
 & \min_{\phi_r \geq 0, y_r\in \phi_r B_r, \lambda_r} \max_{\lambda}  \sum_{r\in [R]} \left[\| y_r - \frac{\lambda_r}{\sqrt{R}}\|_{W^{-1}}^2 + \phi_r^2\right] + \left\langle \lambda, \sum_{r\in [R]} \lambda_r - 2Wa \right\rangle \\
= &   \min_{\phi_r \geq 0, y_r\in \phi_r B_r} \max_{\lambda} \min_{\lambda_r} \sum_{r\in [R]}\left[\| y_r -\frac{\lambda_r}{\sqrt{R}}\|_{W^{-1}}^2 +  \phi_r^2\right] + \left\langle \lambda, \sum_{r\in [R]} \lambda_r - 2Wa \right\rangle\\
= &  \min_{\phi_r \geq 0, y_r\in \phi_r B_r}  \max_{\lambda} \sum_{r\in [R]} \left[\frac{1}{4}\|\lambda\|_{W}^2+ \phi_r^2\right] + \left\langle \lambda, \sqrt{R}\sum_{r\in [R]} (y_r - \frac{1}{2} W\lambda)- 2Wa \right\rangle  \\
=& \min_{\phi_r \geq 0, y_r\in \phi_r B_r} \max_{\lambda} -\frac{R}{4} \|\lambda\|_{W}^2 + \sqrt{R}\left\langle \lambda, \sum_{r\in [R]} y_r - 2Wa\right\rangle + \sum_{r\in [R]}\phi_r^2\\
=& \min_{\phi_r \geq 0, y_r\in \phi_r B_r}   \|\sum_{r\in [R]} y_r - 2Wa\|_{W^{-1}}^2 + \sum_{r\in [R]}\phi_r^2,
   \end{align*}
 which is equivalent to~\eqref{CDform}. 

\section{Proofs for Linear Convergence of the RCD Algorithm}

\subsection{Proof of Lemma~\ref{submodularcone}}

We start by recalling the following lemma from~\cite{li2018revisiting} that characterizes the geometric structure of the product base polytope.  
% Next Lemma is a direct extension of Lemma 3.1 in~\cite{li2018revisiting} that generalizes Lemma 4.1~\cite{ene2017decomposable} to the weighted norm space. 
\begin{lemma}[\cite{li2018revisiting}]\label{submodularprop}
Assume that $W \in \mathbb{R}^{N\times N}$ is a positive diagonal matrix. Let $y \in \otimes_{r\in [R]} \phi'_r B_r$ and let 
$s$ be in the base polytope of the submodular function $\sum_{r} \phi'_r F_r $. Then, there exists a point 
$y'\in \otimes_{r\in [R]} \phi'_r B_r$ such that $\sum_{r\in [R]} y_r' = s$ and $\|y'-y\|_{I(W)}\leq \sqrt{\frac{\sum_{i=1}^{N}W_{ii}}{2}} \|\sum_{r\in[R]} y_r - s\|_1$. %Moreover, $\|y'-y\|_{I(W^{(1)})}\leq \sqrt{\frac{\sum_{i=1}^{|V|}W_{ii}^{(1)}\sum_{j=1}^{|V|}1/W_{jj}^{(2)}}{2}} \|\sum_{e}A_e^T y_e- z\|_{W^{(2)}}$.
\end{lemma}
To prove Lemma~\ref{submodularcone}, Lemma~\ref{submodularprop} cannot be used directly since $y,\,y'$ 
are in different product base polytopes, $ \otimes_{r\in [R]} \phi_r B_r$ and $ \otimes_{r\in [R]} \phi_r' B_r$, respectively. 
However, the following lemma shows one can transform $y$ to lie in the same base polytopes that contains $y'$.

\begin{lemma}\label{rescale}
For a given feasible point $(y,\phi)\in  \otimes_{r\in [R]} C_r$, and a nonnegative vector  $\phi' = (\phi_r) \in \otimes_{r\in [R]}\mathbb{R}_{\geq 0}$, one has 
\begin{align*}
\|\sum_{r\in [R]} y_r - s\|_{1} +  \frac{\rho}{2}\|\phi'-\phi\| \geq \|\sum_{r\in [R]} \frac{\phi_r'}{\phi_r}y_r - s\|_{1}.
\end{align*}
\end{lemma}
\begin{proof}
For all $r$, let $\tilde{y}_r = y_r/ \phi_r \in B_r$, and define a function that depends on $\phi$,
\begin{align*}
h(\phi) = \|\sum_{r\in [R]} y_r - s\|_{1} = \|\sum_{r\in [R]} \phi_r\tilde{y}_r  - s\|_{1}. %=  \|\sum_{e\in E} A_e r_e'\tilde{y}_e  - a\|_{W} + \int_{t=0}^1 \langle \rangle
\end{align*}
For all $\phi$ and $r$, $|\nabla_{\phi_r} h(\phi)|  \leq \|\tilde{y}_r\|_{1}.$
Therefore,
\begin{align*}
h(\phi') &= h(\phi) + \int_{t=0}^1 \langle \nabla h|_{\phi+t(\phi'-\phi)}, t(\phi'-\phi)\rangle dt \\
& \geq h(\phi) - \frac{\max_{t\in[0,1]}\|\nabla h|_{\phi+t(\phi'-\phi)}\|}{2} \|\phi'-\phi\| \\
& \geq h(\phi) - \frac{\max_{\tilde{y}_r\in B_r, \forall r}\sqrt{\sum_{r\in[R]} \|\tilde{y}_r\|_{1}^2} }{2}\|\phi'-\phi\| = h(\phi) - \frac{\rho}{2}\|\phi'-\phi\| .
\end{align*}
\end{proof}
\vspace{-0.5cm}

Combining Lemma~\ref{rescale} with Lemma~\ref{submodularprop}, we can establish the claim of Lemma~\ref{submodularcone}. 

First, let $\rho(W^{(1)}) = \max_{y\in\otimes_{r\in[R]}B_r} \sqrt{\sum_{r\in[R]}\|y_r\|_{W_{r}}^2}$. Suppose that $y'\in \otimes_{r\in [R]} \phi'_rB_r$ is such that $\sum_{r\in [R]} y_r'= s$ and it minimizes $\sum_{r\in [R]}\|\frac{\phi_r'}{\phi_r}y_r - y_r'\|_{W^{(1)}}^2$. As $s$ lies in the base polytope of $\sum_{r\in[R]}\phi_r'F_r$, we know such $y'$ exists. Moreover, we have  
\begin{eqnarray*}
&&\|y - y'\|_{I(W^{(1)})} \leq \sum_{r\in [R]}\|y_r'-\frac{\phi_r'}{\phi_r} y_r\|_{W^{(1)}} +  \sum_{r\in [R]}\|y_r - \frac{\phi_r'}{\phi_r} y_r\|_{W^{(1)}} \\
&\stackrel{1)}{\leq}&\sqrt{\frac{\sum_{i\in[N]}W_{ii}^{(1)}}{2}}\|\sum_{r\in [R]} \frac{\phi_r'}{\phi_r} y_r- s\|_{1} + \rho(W^{(1)})\|\phi'-\phi\| \\
&\stackrel{2)}{\leq}& \sqrt{\frac{\sum_{i\in[N]}W_{ii}^{(1)}}{2}}\left[\|\sum_{r\in [R]} y_r- s\|_{1} + \frac{\rho}{2} \|\phi'-\phi\| \right] +  \rho(W^{(1)})\|\phi'-\phi\| \\
& =& \sqrt{\frac{\sum_{i\in[N]}W_{ii}^{(1)}\sum_{j\in[N]}1/W_{jj}^{(2)}}{2}} \|\sum_{r\in [R]} y_r- s\|_{W^{(2)}} + \left[\sqrt{\frac{\sum_{i\in [N]}W_{ii}^{(1)}}{2}}\frac{\rho}{2} + \rho(W^{(1)}) \right]\|\phi'-\phi\|  \\
& \stackrel{3)}{\leq} &\sqrt{\frac{\sum_{i\in [N]}W_{ii}^{(1)}\sum_{j\in[N]}1/W_{jj}^{(2)}}{2}} \|\sum_{r\in[R]} y_r- a\|_{W^{(2)}} + \frac{3}{2}\sqrt{\frac{\sum_{i\in[N]}W_{ii}^{(1)}}{2}}\rho\|\phi'-\phi\|, 
%\left[\sum_{i=1}^{|V|}W_{ii}^{(1)}\sum_{j=1}^{|V|}1/W_{jj}^{(2)}\right]  \left[2\|\sum_{e\in E} A_r^Ty_r - a\|_{W^{(2)}}^2 +  \rho(W^{(2)}, \Theta)^2\|r'-r\|_{\Theta^{-1}}^2\right] + 2\rho(W^{(1)}, \Theta)^2\|r'-r\|_{\Theta^{-1}} %\\
%& \leq  \left\{\sqrt{\frac{\sum_{i=1}^{|V|}W_{ii}^{(1)}\sum_{j=1}^{|V|}1/W_{jj}^{(2)}}{2}}\left[1+ \frac{\rho(W^{(2)}, \Theta)}{2}\right] + \rho(W^{(1)}, \Theta)\right\}\left( \|\sum_{e\in E} A_r^Ty_r - a\|_{W^{(2)}} + \frac{\|r'-r\|_{\Theta^{-1}}}{2}\right)
\end{eqnarray*}
where 1) follows from Lemma~\ref{submodularprop} and the definition of $\rho(W^{(1)})$, 2) is a consequence of Lemma~\ref{rescale} and 3) holds because 
\begin{align*}
\sum_{i\in [N]}W_{ii}^{(1)}\sum_{r\in[R]}\|y_r\|_1^2 \geq \sum_{r\in[R]}\|y_r\|_{W_{r}}^2.
\end{align*}

\subsection{Proof of Lemma~\ref{strongconv}}
If $(y^*, \phi^*)$ is the optimal solution, then it must hold that $\sum_{r\in [R]} y^*_r  = 2W(a - x^*)$ because of the duality between the primal and dual variables. Moreover, we also know that there must exist a nonempty collection $\mathcal{Y}$ of points $y'\in\otimes_{r\in [R]} \phi_r^* B_r$ such that  $\sum_{r\in [R]} y'_r =\sum_{r\in [R]} y^*_r$. Using Theorem~\ref{submodularcone}, and setting $\phi'=\phi^*, s = 2W(a - x^*),\;W^{(1)}, W^{(2)} = W^{-1}$, we can show that there exists some $y' \in \mathcal{Y}$ such that  
\begin{align*}
\|y-y'\|_{I(W^{-1})}^2 + \|\phi-\phi^*\|^2 \leq\mu(W^{-1}, W^{-1})\left[\|\sum_{r\in [R]}(y_r- y_r')\|_{W^{-1}}^2+  \|\phi-\phi^*\|^2\right].
\end{align*}
 According to the definition of $y^*$, one has $\|y-y^*\|_{I(W^{-1})}^2\leq\|y-y'\|_{I(W^{-1})}^2$ for $y'\in\mathcal{Y}$. This concludes the proof. 

\subsection{Proof of Theorem~\ref{linearconv}}
%The proof of a counterpart Theorem for the DSFM problem can be easily extended here~\cite{ene2015random}.
 %Now, we turn our attention to the proof of the theorem. To proceed, we choose

First, suppose that $(y^*, \phi^*) = \arg\min_{(y', \phi')\in \Xi} \|y^{(k)}-y'\|_{I(W^{-1})}^2 + \|\phi^{(k)}-\phi'\|^2$. Throughout the proof, for simplicity, we use $\mu$ to denote $\mu(W^{-1}, W^{-1})$. We start with by establishing the following results.
\begin{lemma} It can be shown that the following inequalities hold:
\begin{align}
&\quad \langle \nabla g(y^{(k)}, \phi^{(k)}), (y^*- y^{(k)}, \phi^* - \phi^{(k)})\rangle\nonumber  \\
&\stackrel{1)}{\leq} g(y^*, \phi^*) - g(y^{(k)}, \phi^{(k)}) -\frac{1}{\mu}\left(\|y^{(k)}-y^*\|_{I(W^{-1})}^2 + \|\phi^{(k)}-\phi^*\|^2\right)  \nonumber \\
& \stackrel{2)}{\leq} \frac{2}{\mu + 1}\left[ g(y^*, \phi^*) - g(y^{(k)}, \phi^{(k)})- \|y^{(k)}-y^*\|_{I(W^{-1})}^2 - \|\phi^{(k)}-\phi^*\|^2\right].\label{sc4}
\end{align} 
\end{lemma}
\begin{proof}
From Lemma~\ref{strongconv}, we infer that
\begin{align}
&\|\sum_{r\in [R]} (y_r- y_r^*)\|_{W^{-1}}^2+  \|\phi-\phi^*\|^2\geq  \frac{1}{\mu}\left[\|y-y^*\|_{I(W^{-1})}^2 + \|\phi-\phi^*\|^2\right] \Rightarrow \nonumber  \\
& g(y^*, \phi^*) \geq g(y^{(k)}, \phi^{(k)})+ \langle \nabla g(y^{(k)},  \phi^{(k)}), (y^*- y^{(k)}, \phi^* - \phi^{(k)})\rangle +\frac{1}{\mu}\left[\|y-y^*\|_{I(W^{-1})}^2 + \|\phi-\phi^*\|^2\right],\label{sc1} \\
& g(y^{(k)}, \phi^{(k)})  \geq g(y^*, \phi^*) +\langle \nabla g(y^{*}, \phi^*), (y^{(k)}- y^{*}, \phi^{(k)} - \phi^*)\rangle+\frac{1}{\mu}\left[\|y-y^*\|_{I(W^{-1})}^2 + \|\phi-\phi^*\|^2\right]. \label{sc2} 
\end{align} 
As $\langle \nabla g(y^{*}, \phi^*), (y^{(k)}- y^{*}, \phi^{(k)} - \phi^*)\rangle \geq 0$, \eqref{sc2} gives 
\begin{align}
g(y^*, \phi^*) -  g(y^{(k)}, \phi^{(k)})\leq  -\frac{1}{\mu}\left[\|y-y^*\|_{I(W^{-1})}^2 + \|\phi^{(k)}-\phi^*\|^2\right]. \label{sc3}
\end{align} 
Inequality~\eqref{sc1} establishes claim 1) in~\eqref{sc4}. Claim 2) in~\eqref{sc4} follows from~\eqref{sc3}.
%Combining \eqref{sc1}, \eqref{sc2} and  \eqref{sc3}, we can prove \eqref{sc4}. 
\end{proof}

The following lemma is a direct consequence of the optimality of $y_r^{(k+1)}$ as the projection $\Pi_{C_r}$. 
\begin{lemma}\label{optcond} One has
\begin{align*}
&\langle \nabla_r g((y^{(k)}, \phi^{(k)})), (y_r^{(k+1)} - y_r^*, \phi_r^{(k+1)} - \phi_r^*)\rangle \\
& \quad\quad \leq 2\langle y_r^{(k)} - y_r^{(k+1)}, y_r^{(k+1)} - y_r^*\rangle_{W^{-1}} + 2 \langle \phi_r^{(k)} - \phi_r^{(k+1)}, \phi_r^{(k+1)} - \phi_r^*\rangle. 
\end{align*}
\end{lemma}
The following lemma follows from a simple manipulation of the Euclidean norm.
\begin{lemma}\label{threepoints}
It holds that
\begin{align*} 
&\|y_r^{(k+1)} -  y_r^{(k)}\|_{W^{-1}}^2 + (\phi_r^{(k+1)} - \phi_r^{(k)})^2 \\
& \quad\quad  =\|y_r^{(k+1)} -  y_r^{*}\|_{W^{-1}}^2 + (\phi_r^{(k+1)} - \phi_r^{*})^2 + \|y_r^{(k)} -  y_r^{*}\|_{W^{-1}}^2 + (\phi_r^{(k)} - \phi_r^{*})^2\\
&\quad \quad\quad +  2\langle y_r^{(k+1)} - y_r^{*},  y_r^{*} - y_r^{(k)}\rangle_{W^{-1}} + 2\langle \phi_r^{(k+1)} - \phi_r^{*}, \phi_r^{*} - \phi_r^{(k)}\rangle\nonumber \\
 & \quad\quad =- \|y_r^{(k+1)} -  y_r^{*}\|_{W^{-1}}^2 - (\phi_r^{(k+1)} - \phi_r^{*})^2  +\|y_r^{(k)} -  y_r^{*}\|_{W^{-1}}^2 + (\phi_r^{(k)} - \phi_r^{*})^2\\
  &  \quad\quad\quad +  2\langle y_r^{(k+1)} - y_r^{*},  y_r^{(k+1)} - y_r^{(k)}\rangle_{W^{-1}} +2 \langle \phi_r^{(k+1)} - \phi_r^{*}, \phi_r^{(k+1)} - \phi_r^{(k)}\rangle
  \end{align*}
  \end{lemma}
%The following lemma shows that the norm $\|\cdot\|_{2, D_{C}'}^2$ that depends on the picked groung $C$ performs equivalently to the norm $\|\cdot\|_{2,D_{\mathcal{P}}'}^2$ in expectation.

Let us analyze next the amount by which the objective function decreases in each iteration. 

The following expectation is with respect to uniformly sampled values of $r\in [R]$ in the $k$-th iteration: 
\begin{align}
&\quad \mathbb{E}\left[g(y^{(k+1)}, \phi^{(k+1)})\right] - g(y^{(k)}, \phi^{(k)}) \nonumber\\
& = \mathbb{E}\left[\langle \nabla_{r} g(y^{(k)}, \phi^{(k)}), (y_r^{(k+1)}-y_r^{(k)}, \phi_r^{(k+1)} - \phi_r^{(k)})\rangle + \|y_r^{(k+1)}-y_r^{(k)}\|_{W^{-1}}^2 + (\phi_r^{(k+1)} - \phi_r^{(k)})^2\right]\nonumber \\
& = \mathbb{E}\left[ \langle \nabla_{r} g(y^{(k)}, \phi^{(k)}), (y_r^*- y_r^{(k)}, \phi_r^*- \phi_r^{(k)}  )\rangle+ \langle \nabla_{r} g(y^{(k)}, \phi^{(k)}), (y_r^{(k+1)}-y_r^*, \phi_r^{(k+1)} - \phi_r^*)\rangle  \right.\nonumber \\
&\left. \quad\quad\quad+ \|y_r^{(k+1)}-y_r^{(k)}\|_{W^{-1}}^2 + (\phi_r^{(k+1)} - \phi_r^{(k)})^2\right] \nonumber \\
& \stackrel{1)}{\leq} \mathbb{E} \left[ \langle \nabla_{r} g(y^{(k)}, \phi^{(k)}), (y_r^*- y_r^{(k)}, \phi_r^*- \phi_r^{(k)}  )\rangle - \|y_r^{(k+1)} - y_r^{*}\|_{W^{-1}}^2   +  \|y_r^{*} - y_r^{(k)}\|_{W^{-1}}^2 \right. \nonumber\\
&\left. \quad\quad\quad- (\phi_r^{(k+1)} - \phi_r^{*})^2 +(\phi_r^{*} - \phi_r^{(k)})^2 \right] \nonumber \\
&\stackrel{2)}{\leq} \frac{1}{R}\langle \nabla g(y^{(k)}, \phi^{(k)}), (y^{*}-y^{(k)}, \phi^{*} - \phi^{(k)})\rangle - \mathbb{E}\left[ \|y^{(k+1)} - y^{*}\|_{I(W^{-1})}^2 +\|\phi^{(k+1)} - \phi^{*}\|^2 \right] \nonumber\\ 
&\quad\quad\quad +  \|y^{(k)} - y^{*}\|_{I(W^{-1})}^2 + \|\phi^{(k)} - \phi^{*}\|^2 \\
& \stackrel{3)}{\leq}  \frac{2}{( \mu+ 1)R}\left[ g(y^*, \phi^*) - g(y^{(k)}, \phi^{(k)})\right] + \left(1-\frac{2}{(\mu + 1)R}\right)\left[\|y^{(k)}-y^*\|_{I(W^{-1})}^2 + \|\phi^{(k)}-\phi^*\|^2\right] \nonumber \\
&\quad\quad\quad  - \mathbb{E}\left[ \|y^{(k+1)} - y^{*}\|_{I(W^{-1})}^2 + \|\phi^{(k+1)} - \phi^{*}\|^2\right]. \label{onestep}
\end{align}
Here, $1)$ is a consequence of Lemma~\ref{optcond} and Lemma~\ref{threepoints}, $2)$ is due to $y^{(k+1)}_{r'} = y^{(k)}_{r'},\, \phi^{(k+1)}_{r'} = \phi^{(k)}_{r'}$ for $r'\neq r$, and $3)$ may be established from~\eqref{sc4}.

Equation \eqref{onestep} further establishes that 
\begin{align*}
&\mathbb{E}\left[g(y^{(k+1)}, \phi^{k+1})- g(y^*, \phi^*) + d^2((y^{(k+1)}, \phi^{(k+1)}), \Xi) \right] \\
\leq &\quad \mathbb{E}\left[g(y^{(k+1)}, \phi^{k+1})- g(y^*, \phi^*) + \|y^{(k+1)} - y^{*}\|_{I(W^{-1})}^2 +  \|\phi^{(k+1)} - \phi^{*}\|_{I(W^{-1})}^2\right] \\
\leq & \quad \left[1-\frac{2}{(\mu + 1)R}\right] \mathbb{E}\left[g(y^{(k)}, \phi^{k})- g(y^*,\phi^*) + d^2((y^{(k)}, r^{(k)}), \Xi) \right] .
\end{align*}
The proof follows by repeating the derivations for all values of $k$. 

\subsection{Proof of Corollary~\ref{specialcase}}
First, we establish an upper bound on $\rho$.
\begin{lemma}\label{boundrho}
Suppose that $D_{ii} = \sum_{r: r\in[R], i \in C_r} \max_{S\subseteq V} [F_r(S)]^2$. Then 
\begin{align*}
\rho^2 \leq 4\sum_{i\in[N]} D_{ii}.
\end{align*}
\end{lemma}
\begin{proof}
For each $r$, consider $y_r \in B_r$. Sort the entries of $y_r$ in descending order. Without loss of generality, assume that the ordering reads as $y_{r, i_1} \geq y_{r, i_2} \geq \cdots \geq y_{r, i_N}$. As $F_r([N]) = \sum_{j=1}^N y_{r, i_j} \geq 0$, we have $y_{r,i_{1}} \geq 0$. If $y_{r, i_N} \geq 0$, then $\|y_r\|_1 = \sum_{k=1}^ N y_{r, i_k} = F_r([N]) \leq \max_{S\subseteq [N]} F_r(S)$. If $y_{r, i_N} < 0$, there exists a $k'$ such that $y_{r, i_{k'}} \geq 0$ and $y_{r, i_{k'+1}} < 0$. Given the definition of $B_r$, we have 
$$\sum_{k=k'+1}^N |y_{r, i_{k}}| \leq \sum_{k=1}^{k'} |y_{r, i_{k}}| \leq  F_r(\{i_1, i_2,...,i_{k'}\}) \leq  \max_{S\subseteq [N]} F_r(S),$$ and thus $\|y_r\|_1 \leq 2 \max_{S\subseteq [N]} F_r(S)$. Moreover, as each variable in $[N]$ is incident to at least one submodular function, we have
\begin{align*}
\sum_{r\in[R]} \max_{S\subseteq [N]} [F_r(S)]^2 \leq  \sum_{i\in[N]} \sum_{r: i\in S_r} \max_{S\subseteq [N]} [F_r(S)]^2 \leq  \sum_{i\in[N]} D_{ii}.
\end{align*} 
Combining all of the above results, we obtain
\begin{align*}
\rho^2 = \sum_{r\in[R]} \max_{y_r\in B_r} \|y_r\|_1^2 \leq  4\sum_{r\in[R]}\max_{S\subseteq [N]} [F_r(S)]^2 \leq 4 \sum_{i\in[N]} D_{ii}.
\end{align*} 
\end{proof}
%\cdots \geq y_{r, i_k} \geq 0 \geq y_{r, i_{k+1}} \geq

When $W=\lambda D$, we have 
\begin{align*}
\sum_{i\in[N]} W_{ii} \sum_{j\in[N]}1/W_{jj} \leq N^2 \max_{i,j} \frac{W_{ii}}{W_{jj}} = N^2 \max_{i,j} \frac{D_{ii}}{D_{jj}},
\end{align*}
and 
\begin{align*}
\rho^2 \sum_{j\in[N]}1/W_{jj} \stackrel{1)}{\leq} 4 \sum_{i\in[N]} D_{ii} \sum_{j\in[N]}1/W_{jj} \leq  \frac{4}{\beta}N^2 \max_{i,j} \frac{D_{ii}}{D_{jj}},
\end{align*}
where $1)$ follows from Lemma~\ref{boundrho}. According to the definition of $\mu(W^{-1}, W^{-1})$ (see ~\eqref{defmu}),
\begin{align*}
\mu(W^{-1}, W^{-1}) \leq N^2 \max\{1, 9\beta^{-1}\}\max_{i,j} \frac{D_{ii}}{D_{jj}}.
\end{align*}
This concludes the proof. 
%\subsection{Alternating Projection Methods}\label{APM}
%
%\begin{table}[htb]
%\centering
%\begin{tabular}{l}
%\hline
%\label{APM}
%\textbf{Algorithm 1: } \textbf{APM for Solving ~\eqref{APform}} \\
%\hline
%\ 0: Initialize for all $e$, $y\phi_r^{(0)}\leftarrow 0$, $r_e^{(0)}$, and $k\leftarrow 0$\\
%\ 1: In iteration $k$:\\
%\ 2: \; $\lambda^{(k+1)} \leftarrow \sum_e A_e^Ty_r^{(k)} -Ws$. \\
%\ 3: \; For all $e\in E$:\\
%\ 4: \quad \; $\lambda_e^{(k+1)} \leftarrow y_r^{(k)} - \frac{1}{2}\vartheta_e D_e^{-1}W^{-1}A_e \lambda^{(k+1)}$. \\
%\ 5: \quad \; $(y_r^{(k+1)}, r_{e}^{(k+1)})\leftarrow  \Pi_{C_{e}}(\lambda_e^{(k+1)}, 2D_eW_{e}^{-1}) $\\
%%\ 5: \; $a^{(k+1)}\leftarrow \tilde{a}^{(k)} - A_e^Ty_r^{(k+1)}$ \\
%\hline
%\end{tabular}
%\end{table}

\section{The Alternating Projection Algorithm for Solving the QDSFM Problem }\label{appsec:AP}
Here, we discuss the alternating projection method (AP) for solving QDSFM problems. AP can be used to solve the best approximation problem~\eqref{APform} by alternatively performing projections between the product cone $\{(y,\phi)| (y,\phi)\in \otimes_{r\in[R]} C_r\}$ and the hyperplane $\{\Lambda = (\lambda_r)_{r\in[R]} | \sum_{r\in[R]} \lambda_r = 2Wa\}$. For some given incidence relations, $S_r$ may be a strict subset of $[N]$, which requires $y_{r,i}$ to be zero if $i\not\in S_r$. By imposing $\lambda_{r, i} = 0$ for $i\not\in S_r$, AP avoids redundant computations and offers better convergence rates. This phenomenon has been observed for DSFM problems in~\cite{li2018revisiting}. To avoid redundant computations, we use AP to solve the following dual problem:
\begin{align} \label{APcompactform}
 &\min_{y,\phi, \Lambda}\quad \sum_{r\in [R]}\left[\| y_r - \lambda_r\|_{\Psi W^{-1}}^2 + \phi_r^2\right] \nonumber \\
 &\quad \text{s.t. $y\in\otimes_{r\in [R]} \phi_rB_r$, $\phi \in\otimes_{r\in [R]} \mathbb{R}_{\geq 0}$}, \nonumber\\ 
 &\quad\quad \text{$\sum_{r\in [R]} \lambda_r = 2Wa$, \text{ and } $\lambda_{r,i} = 0$ for all $i\not\in S_r$, $r\in[R]$}. 
\end{align}
Here, $\Psi\in\mathbb{R}^{N\times N}$ is a positive diagonal matrix in which $\Psi_{ii} = |\{r\in[R]| i\in S_r\}|$ equals the number of submodular functions that $i$ is incident to. 
\begin{lemma}
Problem~\eqref{APcompactform} is equivalent to problem~\eqref{CDform}.
\end{lemma}
\begin{proof}
First, for $r\in [R]$, we define a diagonal matrix $A_r\in \mathbb{R}^{N\times N}$: $(A_r)_{ii} = 1,$ if $i\in S_r,$ and $0$ otherwise.  Consider a Lagrangian dual of~\eqref{APcompactform}, given as
\begin{align*}
&\min_{(y,\phi)\in\otimes_{r\in [R]} C_r}\; \min_{ \Lambda: \lambda_{r,i} = 0,\, \forall(i, r): i\not\in S_r, r\in[R]}\;\max_{\alpha\in \mathbb{R}^N}\quad \sum_{r\in[R]}\left[\| y_r - \lambda_r\|_{\Psi W^{-1}}^2 + \phi_r^2\right] + \langle \alpha, \sum_{r\in[R]} \lambda_r - 2Wa \rangle \\
=&\min_{(y,\phi)\in\otimes_{r\in [R]} C_r}\;\max_{\alpha\in \mathbb{R}^N}\; \min_{ \Lambda: \lambda_{r,i} = 0,\, \forall(i, r): i\not\in S_r, r\in[R]}\quad \sum_{r\in[R]}\left[\| y_r - \lambda_r\|_{\Psi W^{-1}}^2 + \phi_r^2\right] + \langle \alpha, \sum_{r\in[R]} \lambda_r - 2Wa \rangle \\
\stackrel{1)}{=}& \min_{(y,\phi)\in\otimes_{r\in [R]} C_r}\; \max_{\alpha\in \mathbb{R}^N} \quad \sum_{r\in[R]}\left[ \frac{1}{4}\|A_r\Psi^{-1}W\alpha\|_{\Psi W^{-1}}^2 + \phi_r^2 \right] + \left\langle \alpha, \sum_{r\in[R]} \left(y_r - \frac{1}{2}A_r\Psi^{-1}W\alpha\right) - 2Wa \right\rangle \\
\stackrel{2)}{=}&  \min_{(y,\phi)\in\otimes_{r\in [R]} C_r}\; \max_{\alpha\in \mathbb{R}^N} \quad - \frac{1}{4}\|\alpha\|_{W}^2 +\left\langle \alpha, \sum_{r\in[R]}y_r - 2Wa \right\rangle  + \sum_{r\in[R]} \phi_r^2  \\
= & \min_{(y,\phi)\in\otimes_{r\in [R]} C_r} \quad \|\sum_{r\in[R]}y_r - 2Wa \|_{W^{-1}}^2  + \sum_{r\in[R]} \phi_r^2,
\end{align*}
where 1) is due to $\lambda_r = y_r - \frac{1}{2} A_r\Psi^{-1}W\alpha$ and 2) is based on the fact that $\Psi = \sum_{r\in[R]} A_r$.
This established the claimed result.
\end{proof}
The AP method for solving~\eqref{APcompactform} is described in Algorithm 3. Note that Step 5 is a projection onto cones with respect to the positive diagonal matrix equal to $\Psi W^{-1}$ which differs from the one used in RCD. However, the same methods described in Section 5 can be used to compute these projections. Note that compared to RCD, AP requires one to compute projections onto all $C_r$ in each iteration and thus has large computational overhead in most cases; but AP naturally lends itself to parallelization since the projections can be decoupled.  
\begin{table}[htb]
\centering
\begin{tabular}{l}
\hline
\label{APM}
\textbf{Algorithm 3: } \textbf{The AP Method for Solving ~\eqref{APcompactform}} \\
\hline
\ 0: For all $r$, initialize $y_r^{(0)} \leftarrow 0, \phi_r^{(0)}\leftarrow 0$, and $k\leftarrow 0$\\
\ 1: In iteration $k$:\\
\ 2: \; $\alpha^{(k+1)} \leftarrow 2W^{-1}\sum_r y_r^{(k)} -4a$. \\
\ 3: \; For all $r\in [R]$:\\
\ 4: \quad \; $\lambda_r^{(k+1)} \leftarrow y_r^{(k)} - \frac{1}{2}A_r\Psi^{-1}W\alpha^{(k+1)}$. \\
\ 5: \quad \; $(y_r^{(k+1)}, \phi_{r}^{(k+1)})\leftarrow  \arg\min_{(y_r, \phi_r)\in C_r} \| y_r - \lambda_r\|_{\Psi W^{-1}}^2 + \phi_r^2 $\\
%\ 5: \; $a^{(k+1)}\leftarrow \tilde{a}^{(k)} - A_e^Ty_r^{(k+1)}$ \\
\hline
\end{tabular}
\end{table}

Next, we establish the convergence rate of AP. For this purpose, we use the result of Nishihara et al.~\cite{nishihara2014convergence} on the convergence rate of AP between two convex bodies~\cite{nishihara2014convergence}. In our case, the two convex bodies are the hyperplane,  
\begin{align*}
\mathcal{Z} = \{(y,\phi)\in \left(\otimes_{r\in[R]} \mathbb{R}^N\right) \otimes  \left(\otimes_{r\in[R]} \mathbb{R}\right) | \sum_{r\in[R]} y_r = 2W(a - x^*),  \phi_r = \phi_r^*, y_{r,i} = 0, \forall r\in[R], i\not\in S_r\},
\end{align*}
where $\phi^* = (\phi_r^*)_{r\in [R]}$ is the unique optimal solution of~\eqref{APcompactform}\footnote{As \eqref{APcompactform} is strongly convex in $\phi$ so the optimal $\phi$ is unique.}, and the cone $\mathcal{C} = \otimes_{r\in [R]} C_r$. 

%The following lemma characterizes the convergence property. 
\begin{lemma}[\cite{nishihara2014convergence}] 
Define the distance function 
$$d_{\Psi W^{-1}}((y,\phi), \Xi) = \sqrt{\min\limits_{(y',\phi')\in\Xi} \|y-y'\|_{I(\Psi W^{-1})}^2 + \|\phi -\phi'\|^2}.$$ 
In the $k$-th iteration of Algorithm 5, the pair $(y^{(k)}, \phi^{(k)})$ satisfies
\begin{align*}
d_{\Psi W^{-1}}((y^{(k)}, \phi^{(k)}), \Xi) \leq 2d_{\Psi W^{-1}}((y^{(0)}, \phi^{(0)}), \Xi)(1 - \frac{1}{\kappa_*^2})^k,
\end{align*}
where 
\begin{align*}
\kappa_* = \sup_{(y, \phi) \in \mathcal{C} \cup \mathcal{Z} / \Xi} \frac{d_{\Psi W^{-1}}((y, \phi), \Xi)}{\max\{d_{\Psi W^{-1}}((y, \phi), \mathcal{C}), d_{\Psi W^{-1}}((y, \phi), \mathcal{Z})\}}.
\end{align*}
\end{lemma}
We establish next the following lemma that provides a finite upper bound on $\kappa_*$, and hence guarantees linear convergence rates for AP. 
\begin{lemma}\label{upperboundkappa}
One has $\kappa_*^2 \leq 1+ \mu(\Psi W^{-1}, W^{-1})$.
\end{lemma}
Before proceeding with the proof, we first remark that Lemma~\ref{upperboundkappa} implies that using $k \sim O((\mu(\Psi W^{-1}, W^{-1})+1) \log \frac{1}{\epsilon})$ iterations, one guarantees $d_{\Psi W^{-1}}((y^{(k)}, \phi^{(k)}), \Xi)\leq \epsilon$. Moreover, as $\Psi_{ii} \leq R$, we know that the number of iterations executed by AP will not be larger than that of RCD, although, in each iteration, AP requires one to perform projections on all the $C_r$. 

Next, we turn our attention to the proof of Lemma~\ref{upperboundkappa}.
\begin{proof}
Suppose that $(y, \phi) \in \mathcal{C}/\Xi$. Then,
\begin{align*}
[d_{\Psi W^{-1}}((y, \phi), \mathcal{Z})]^2 = &\min_{\lambda_r, \forall r\in[R]} \sum_{r}\left[\| y_r - \lambda_r\|_{\Psi W^{-1}}^2 + (\phi_r - \phi_r^*)^2\right]  \\
& \text{s.t.  $\sum_{r\in[R]} \lambda_r = 2W(a - x^*)$, $\lambda_{r,i} = 0, \forall r\in[R], i\not\in S_r$}.
%\|\sum_{r}y_r - 2Wa \|_{W^{-1}}^2  + \sum_{r} \phi_r^2.
\end{align*}
By eliminating $\lambda_r$, we arrive at 
$$[d_{\Psi W^{-1}}((y, \phi), \mathcal{Z})]^2 = \|\sum_{r}y_r - 2W(a- x^*) \|_{W^{-1}}^2  + \sum_{r}  (\phi_r - \phi_r^*)^2.$$ 
Based on Lemma~\ref{submodularcone}, we know that there exists a $(y', \phi') \in \Xi$ such that 
\begin{align*}
\mu(\Psi W^{-1},W^{-1}) \left[\|\sum_{r}(y_r - y_r')  \|_{W^{-1}}^2  + \sum_{r} (\phi_r - \phi_r')^2 \right]\geq \|y - y'\|_{\Psi W^{-1}}^2 + \sum_{r} (\phi_r - \phi_r')^2.
\end{align*}
As $\phi_r^*$ is the uniquely optima, it follows that $\phi_r^* = \phi_r'$. Also, $\sum_r y_r' = 2W(a-x^*)$. Moreover, as 
$$\|y - y'\|_{\Psi W^{-1}}^2 + \sum_{r} (\phi_r - \phi_r')^2 \geq [d_{\Psi W^{-1}}((y, \phi), \Xi)]^2$$ 
according to the above definition, we have 
\begin{align*}
\frac{ [d_{\Psi W^{-1}}((y, \phi), \Xi)]^2}{[d_{\Psi W^{-1}}((y, \phi), \mathcal{Z})]^2} \leq \mu(\Psi W^{-1},W^{-1}).
\end{align*}
Next, suppose that $(y, \phi) \in \mathcal{Z}/\Xi$ and that 
$$(y', \phi')  = \arg\min_{(z,\psi) \in  \mathcal{C} } \|y-z\|_{I(\Psi W^{-1})}^2 + \|\phi -\psi\|^2,$$ 
$$(y'', \phi'') = \arg\min_{(z,\psi) \in \Xi} \|y'-z\|_{I(\Psi W^{-1})}^2 + \|\phi'- \psi\|^2.$$ 
Again, due to the definition of the distance $d_{\Psi W^{-1}}((y, \phi), \Xi)$, we have 
\begin{align}\label{ap1}
[d_{\Psi W^{-1}}((y, \phi), \Xi)]^2 \leq \|y-y''\|_{I(\Psi W^{-1})}^2 + \|\phi -\phi''\|^2.
\end{align}
Moreover, because of the way we chose $(y',\phi')$ and due to the fact that $\mathcal{C}$ is convex, we have
\begin{align}\label{ap2}
\|y-y''\|_{I(\Psi W^{-1})}^2 + \|\phi -\phi''\|^2 \leq \|y-y'\|_{I(\Psi W^{-1})}^2 + \|\phi -\phi'\|^2 + \|y'-y''\|_{I(\Psi W^{-1})}^2 + \|\phi' -\phi''\|^2.
\end{align}
Using Lemma~\ref{submodularcone}, we obtain 
\begin{align}\label{ap3}
\|y'-y''\|_{I(\Psi W^{-1})}^2 + \|\phi' -\phi''\|^2 \leq \mu(\Psi W^{-1}, W^{-1})( \|\sum_{r}(y_r'-y_r'')\|_{W^{-1}}^2 + \|\phi' -\phi''\|^2 ),
\end{align}
and we also have
\begin{align}\label{ap4}
 \|\sum_{r}(y_r'-y_r'')\|_{W^{-1}}^2 =  \|\sum_{r}y_r' - 2W(a - x^*)\|_{W^{-1}}^2  =  \|\sum_{r}(y_r'- y_r)\|_{W^{-1}}^2 \stackrel{1)}{\leq}  \|(y'- y)\|_{I(\Psi W^{-1})}^2,
\end{align}
where 1) follows from the Cauchy-Schwarz inequality over the entries $y_{r,i}, \, i\in S_r$. 

Next, consider $\phi = \phi'' = \phi^*$ and combine \eqref{ap1}-\eqref{ap4} to obtain
\begin{align*}
[d_{\Psi W^{-1}}((y, \phi), \Xi)]^2 &\leq (1+ \mu(\Psi W^{-1}, W^{-1}))(\|y-y'\|_{I(\Psi W^{-1})}^2 + \|\phi -\phi'\|^2) \\
&= (1+ \mu(\Psi W^{-1}, W^{-1}))[d_{\Psi W^{-1}}((y, \phi), \mathcal{C})]^2,
\end{align*}
which concludes the proof. 
\end{proof}

\section{Convergence Analysis of the Conic MNP Algorithm}
\subsection{Preliminary Notation and Lemmas}
%First , let us introduce some notations. 
Given an active set $S = \{q_1, q_2,...\},$ and a collection of coefficients $\lambda= \{\lambda_1, \lambda_2, ...\},$ if $y = \sum_{q_i \in S} \lambda_i q_i$, we simply refer to $(y, S, \lambda)$ as a \emph{triple}. Define the following functions that depend on $S$ 
%and mainly work with there triples
\begin{align*}
\tilde{h}(S, \lambda) &\triangleq h(\sum_{q_i \in S} \lambda_i q_i, \sum_{q_i \in S} \lambda_i), \\
\tilde{h}(S) &\triangleq \min_{\lambda:  \lambda_i \in \mathbb{R}, \forall i} \tilde{h}(S, \lambda), \\
\tilde{h}_+(S) &\triangleq \min_{\lambda: \lambda_i \geq 0, \forall i} \tilde{h}(S, \lambda).
\end{align*}
If the coefficients $\lambda$ ($ \lambda_i \in \mathbb{R}, \forall i$) minimize $\tilde{h}(S, \lambda)$, we call the corresponding triple $(y, S, \lambda)$ a \emph{good} triple. Given a triple $(y, S, \lambda)$, we also define 
$$\triangle(y, q) = - \langle y - a, q \rangle - \sum_{q_i\in S} \lambda_i,$$ 
$$\triangle (y) = \max_{q\in B} \triangle (y) = - \min_{q\in B}\langle y - a, q \rangle - \sum_{q_i\in S} \lambda_i,$$ 
and
$$\text{err}(y) = h(y, \sum_{q_i\in S} \lambda_i) - h^*.$$ 
The following lemma establishes the optimality of a good triple. 
%\|y - a\|_{\tilde{W}}^2 + (\sum_{q_i\in S} \lambda_i)^2 - \|y^* - a\|_{\tilde{W}}^2 - (r^*)^2
\begin{lemma}\label{perpendicular}
 Given an active set $S$, consider the good triple $(y', S, \lambda')$ and an arbitrary triple $(y, S, \lambda)$. Then, 
 $$\langle y'-a, y\rangle_{\tilde{W}} + \langle\sum_{q_i\in S} \lambda_i', \sum_{q_i\in S} \lambda_i \rangle = \langle y'-a, y' - y\rangle_{\tilde{W}}+ \langle\sum_{q_i\in S} \lambda_i', \sum_{q_i\in S}(\lambda_i'- \lambda_i) \rangle = 0.$$
\end{lemma}
\begin{proof} 
Without loss of generality, assume that $\langle y'-a, y\rangle_{\tilde{W}} + \langle\sum_{q_i\in S} \lambda_i', \sum_{q_i\in S} \lambda_i \rangle < 0$. Then, for any $\epsilon>0$, $(y' + \epsilon y, S, \lambda'+ \epsilon \lambda)$ is also a triple. For $\epsilon$ sufficiently small, we have $\tilde{h}(S, \lambda' + \epsilon \lambda) < \tilde{h}(S, \lambda'),$ which contradicts the optimality of $(y', S, \lambda')$. Hence, 
$$\langle y'-a, y\rangle_{\tilde{W}} + \langle\sum_{q_i\in S} \lambda_i', \sum_{q_i\in S} \lambda_i \rangle = 0.$$ 
As $(y' - y , S, \lambda' - \lambda)$ is also a triple, repeating the above procedure we obtain the claimed equality. 
%It suffices to check the derivative of the function, $\nabla_{\lambda} \tilde{h}$, at the point $\lambda'$. As $\lambda'$ is the minimizer, the statement is true. 
\end{proof}
%The following lemma gives an upper bound on $\phi$ when it satisfies conditional optimality. 
\begin{lemma}\label{boundonr}
For any $\hat{y}\in B$, 
$$\arg\min_{\phi\geq 0} h(\phi\hat{y}, \phi)\leq \frac{\|a\|_{\tilde{W}}}{2}.$$ 
Moreover, $\phi^* \leq \frac{\|a\|_{\tilde{W}}}{2}$.
\end{lemma}
\begin{proof}
Given $\hat{y}$, the optimal value of $\phi$ satisfies 
$$\phi= \frac{\langle a, \hat{y} \rangle_{\tilde{W}}}{1+ \|\hat{y}\|_{\tilde{W}}^2} \leq \frac{\|a\|_{\tilde{W}}}{\|\hat{y}\|_{\tilde{W}} + \frac{1}{\|\hat{y}\|_{\tilde{W}}}}\leq \frac{\|a\|_{\tilde{W}}}{2}.$$ This establishes the claimed result.
\end{proof}

\begin{lemma}\label{errorconnect}
If $(y, S, \lambda)$ is a good triple, then $\triangle (y) \geq \frac{\text{err}(y)}{\|a\|_{\tilde{W}}}$.
\end{lemma}
\begin{proof}
Recall that $(y^*, \phi^*)$ denotes the optimal solution. As $y^*/\phi^* \in B$, we have
\begin{align*}
\phi^*\triangle (y) &\geq - \langle y - a, y^* \rangle_{\tilde{W}} - \langle \phi^*, \sum_{q_i\in S} \lambda_i \rangle \\
&\stackrel{1)}{=}  - \langle y - a, y^* \rangle_{\tilde{W}} - \langle \phi^*, \sum_{q_i\in S} \lambda_i \rangle + \langle y-a, y \rangle_{\tilde{W}} + (\sum_{q_i\in S} \lambda_i)^2 \\
& = - \langle y - a, y^*-a \rangle_{\tilde{W}} - \langle \phi^*, \sum_{q_i\in S} \lambda_i \rangle + \langle y-a, y-a \rangle_{\tilde{W}} + (\sum_{q_i\in S} \lambda_i)^2 \\
&\stackrel{2)}{\geq} \frac{1}{2}\left[  \|y-a\|_{\tilde{W}}^2 + (\sum_{q_i\in S} \lambda_i)^2 -  \|y^*-a\|_{\tilde{W}}^2 + (\phi^*)^2\right] \\
& =  \frac{1}{2} \text{err}(y),
\end{align*} 
where $1)$ follows from Lemma~\ref{perpendicular}, while $2)$ is a consequence of the Cauchy-Schwarz inequality. By using the bound for $\phi^*$ described in Lemma~\ref{boundonr}, we arrive at the desired conclusion. 
\end{proof}
%As for the convexity, we have 
%\begin{align*}
%h(y^*, \phi^*) \geq h(y^{(k)}, \phi^{(k)}) + 2 \langle y
%\end{align*}

\subsection{Proof of Theorem~\ref{Wolfegaurantee}}

We only need to prove the following three lemmas which immediately give rise to Theorem~\ref{Wolfegaurantee}. Lemma~\ref{strictly} corresponds to the first statement of Theorem~\ref{Wolfegaurantee}. Combining Lemma~\ref{terminatingcond} and Lemma~\ref{decreaserate}, we can establish the second statement of Theorem~\ref{Wolfegaurantee}.
This follows as we may choose $\epsilon = \delta \|a\|_{\tilde{W}}$. If Algorithm 2 terminates with less than $O(N\|a\|_{\tilde{W}}^2\max\{Q^2,1\}/\epsilon)$ iterations, then the condition of Lemma~\ref{terminatingcond} is satisfied and thus $\text{err}(y^{(k)})\leq \epsilon = \delta \|a\|_{\tilde{W}}$. If Algorithm 2 does not terminate after $O(N\|a\|_{\tilde{W}}^2\max\{Q^2,1\}/\epsilon)$ iterations, Lemma~\ref{decreaserate} guarantees $\text{err}(y^{(k)})\leq \epsilon = \delta \|a\|_{\tilde{W}}$.

\begin{lemma}\label{strictly}
At any point before Algorithm 2  terminates, one has $h(y^{(k)}, \phi^{(k)}) \geq h(y^{(k+1)}, \phi^{(k+1)})$; moreover, if $(y^{(k)}, \phi^{(k)})$ triggers a MAJOR loop, the claimed inequality is strict. 
\end{lemma}

The following lemma characterizes the pair $(y, \phi)$ at the point when the MNP method terminates.
\begin{lemma}\label{terminatingcond}
In the MAJOR loop at iteration $k$, if $\langle y^{(k)} - a, q^{(k)} \rangle_{\tilde{W}} + \phi^{(k)}\geq -\delta$, then $h(y^{(k)}, \phi^{(k)}) \leq h^* + \|a\|_{\tilde{W}}\delta$.
\end{lemma}

\begin{lemma}\label{decreaserate}
If Algorithm 2 does not terminate, then for any $\epsilon > 0$, one can guarantee that after $O(N\|a\|_{\tilde{W}}^2\max\{Q^2,1\}/\epsilon)$ iterations, Algorithm 2 generates a pair $(y, \phi)$ that satisfies $\text{err(y)} \leq \epsilon$.
\end{lemma}
%First, let us introduce some notations. Given an active set $S = \{q_1, q_2,...\}$ and a collection of coefficients $\lambda= \{\lambda_1, \lambda_2, ...\}$, if $y = \sum_{q_i \in S} \lambda_i q_i$, we call $(y, S, \lambda)$ as a \emph{triple}. If the coefficients $\lambda$ minimize $\tilde{h}(\lambda, S)$, we call this triple as a \emph{good triple}. Given a triple $(y, S, \lambda)$, we also define $\triangle(y, q) = - \langle y - a, q \rangle - \sum_{q_i\in S} \lambda_i$ and $\triangle (y) = \max_{q\in B} \triangle (y) = - \min_{q\in B}\langle y - a, q \rangle - \sum_{q_i\in S} \lambda_i$. Also, define $\text{err}(y) = \|y - a\|_{\tilde{W}}^2 + (\sum_{q_i\in S} \lambda_i)^2 - \|y^* - a\|_{\tilde{W}}^2 - (r^*)^2$.  
%Now, we begin to prove the three Lemma. 
The proofs of Lemma~\ref{strictly} and Lemma~\ref{terminatingcond} are fairly straightforward, while the proof of Lemma~\ref{decreaserate} is significantly more involved and postponed to the next section.  

\textbf{Proof of Lemma~\ref{strictly}}.
Suppose that $(y^{(k)}, \phi^{(k)})$ starts a MAJOR loop. As 
$$\langle y^{(k)} - a, q^{(k)} \rangle_{\tilde{W}} + \phi^{(k)}< -\delta,$$ 
we know that there exists some small $\varepsilon$ such that 
$$h(y^{(k)} + \varepsilon q^{(k)}, \phi^{(k)}+ \varepsilon) < h(y^{(k)}, \phi^{(k)}).$$ 
Consider next the relationship between $(z^{(k)}, \sum_{q_i\in S^{(k)}\cup\{q^{(k)}\}}\alpha_i)$ and  $(y^{(k)}, \phi^{(k)})$. 
Because of Step 6, we know that 
$$h(z^{(k+1)}, \sum_{q_i\in S^{(k)}\cup\{q^{(k)}\}}\alpha_i) = \tilde{h}(S^{(k)}\cup\{q^{(k)}\}) \leq h(y^{(k)} + \varepsilon q^{(k)}, \phi^{(k)}+ \varepsilon) < h(y^{(k)}, \phi^{(k)}).$$ 
If $(y^{(k+1)}, \phi^{(k+1)})$ is generated in some MAJOR loop, then 
$$(y^{(k+1)}, \phi^{(k+1)}) = (z^{(k+1)}, \sum_{q_i\in S^{(k)}\cup\{q^{(k)}\}}\alpha_i),$$ 
which naturally satisfies the claimed condition. If $(y^{(k+1)}, \phi^{(k+1)})$ is generated in some MINOR loop, then $(y^{(k+1)}, \phi^{(k+1)})$ lies strictly within the segment between $(z^{(k+1)}, \sum_{q_i\in S^{(k)}\cup\{q^{(k)}\}}\alpha_i)$ and $(y^{(k)}, \phi^{(k)})$ (because $\theta > 0$). Therefore, we also have $h(y^{(k+1)}, \phi^{(k+1)}) < h(y^{(k)}, \phi^{(k)})$. If $(y^{(k)}, \phi^{(k)})$ starts a MINOR loop, then we have 
$$h(z^{(k+1)}, \sum_{q_i\in S^{(k)}\cup\{q^{(k)}\}}\alpha_i) =\tilde{h}(S^{(k)}) \leq h(y^{(k)}, \phi^{(k)}),$$ 
once again due to Step 6. As $(y^{(k+1)}, \phi^{(k+1)})$ still lies within the segment between $(z^{(k+1)}, \sum_{q_i\in S^{(k)}\cup\{q^{(k)}\}}\alpha_i)$ and $(y^{(k)}, \phi^{(k)}),$ we have $h(y^{(k+1)}, \phi^{(k+1)}) \leq h(y^{(k)}, \phi^{(k)})$.
%, then $(y^{(k+1)}, S^{(k+1)}, \lambda^{(k+1)})$ is a good triple and $S^{(k+1)} = S^{(k)} \cup \{q^{(k)}\}$, so $h(y^{(k+1)}, \phi^{(k+1)}) \leq h(y^{(k)}, \phi^{(k)})$. Moreover, 

\textbf{Proof of Lemma~\ref{terminatingcond}}. Lemma~\ref{terminatingcond} is a corollary of Lemma~\ref{errorconnect}. To see why this is the case, observe that in a MAJOR loop, $(y^{(k)}, S^{(k)}, \lambda^{(k)})$ is always a good triple. Since 
$$\triangle(y^{(k)}) = -( \langle y^{(k)} - a, q^{(k)} \rangle_{\tilde{W}} + \phi^{(k)}) \leq \delta,$$ 
we have $\text{err}(y)\leq \delta \|a\|_{\tilde{W}}$. 

\subsection{Proof of Lemma~\ref{decreaserate}} The outline of the proof is similar to that of the standard case described in~\cite{chakrabarty2014provable}, and some results therein can be directly reused. The key step is to show that in every MAJOR loop $k$ with no more than one MINOR loop, the objective achieved by $y^{(k)}$ decreases sufficiently, as precisely described in Theorem~\ref{descent}. 
\begin{theorem}\label{descent}
For a MAJOR loop with no more than one MINOR loop, if the starting point is $y$, the starting point $y'$ of the next MAJOR loop satisfies 
\begin{align*}
\text{err}(y') \leq \text{err}(y) \left(1 - \frac{ \text{err}(y) }{\|a\|_{\tilde{W}}(Q^2+1)}\right).
\end{align*}
\end{theorem}
Based on this theorem, it is easy to establish the result in Lemma~\ref{decreaserate} by using the next lemma and the same proof approach as described in~\cite{chakrabarty2014provable}.   
\begin{lemma}[Lemma 1~\cite{chakrabarty2014provable}]
In any consecutive $3N+1$ iteratons, there exists at least one MAJOR loop with not more than one MINOR loop.
\end{lemma}
We now focus on the proof of Theorem~\ref{descent}. The next geometric lemma is the counterpart of Lemma 2~\cite{chakrabarty2014provable} for the conic case. 
\begin{lemma}\label{Wolfeconv2}
Given an active set $S$, consider a good triple $(y', S, \lambda')$ and an arbitrary triple $(y, S, \lambda)$. Then, for any $q\in\text{lin}(S)$ such that $\triangle(y,q) >0$, we have
   \begin{align*}
   \|y-a\|_{\tilde{W}}^2 + (\sum_{q_i\in S}\lambda_i)^2 -  \|y'-a\|_{\tilde{W}}^2 -  (\sum_{q_i\in S}\lambda_i')^2 \geq \frac{\triangle^2(y,q)}{\|q\|_{\tilde{W}}^2 + 1}.
   \end{align*}
\end{lemma}
\begin{proof}
First, we have
\begin{align*}
&\|y-a\|_{\tilde{W}}^2 + (\sum_{q_i\in S}\lambda_i)^2 -  \|y'-a\|_{\tilde{W}}^2 -  (\sum_{q_i\in S}\lambda_i')^2  \\
=& \|y- y'\|_{\tilde{W}}^2 + [\sum_{q_i\in S}(\lambda_i- \lambda_i')]^2 + 2\langle y - y', y' -a \rangle_{\tilde{W}} + 2 \langle \sum_{q_i\in S}(\lambda_i-\lambda_i'), \sum_{q_i\in S}\lambda_i'\rangle \\
\stackrel{1)}{=} & \|y- y'\|_{\tilde{W}}^2 + [\sum_{q_i\in S}(\lambda_i- \lambda_i')]^2,
\end{align*}
where 1) follows from Lemma~\ref{perpendicular}. Next, for any $\phi\geq 0$, 
\begin{align}
 &\|y- y'\|_{\tilde{W}}^2 + [\sum_{q_i\in S}(\lambda_i- \lambda_i')]^2 \nonumber \\
\stackrel{1)}{\geq}  & \frac{\left[\langle y- y',  y- \phi q\rangle_{\tilde{W}}+ \langle \sum_{q_i\in S}(\lambda_i- \lambda_i'),  \sum_{q_i\in S}\lambda_i - \phi\rangle\right]^2 }{\|y- \phi q\|_{\tilde{W}}^2 + (\sum_{q_i\in S}\lambda_i - \phi)^2} \nonumber\\
= & \frac{\left[\langle y -a,  y- \phi q\rangle_{\tilde{W}}+ \langle \sum_{q_i\in S}\lambda_i,  \sum_{q_i\in S}\lambda_i - \phi\rangle- \langle y' -a,  y- \phi q\rangle_{\tilde{W}}- \langle \sum_{q_i\in S}\lambda_i',  \sum_{q_i\in S}\lambda_i - \phi\rangle\right]^2 }{\|y- rq\|_{\tilde{W} }^2 + (\sum_{q_i\in S}\lambda_i - \phi)^2}  \nonumber\\ 
\stackrel{2)}{=} &\frac{\left[\langle y -a,  y- \phi q\rangle_{\tilde{W}}+ \langle \sum_{q_i\in S}\lambda_i,  \sum_{q_i\in S}\lambda_i - \phi\rangle\right]^2 }{\|y- \phi q\|_{\tilde{W} }^2 + (\sum_{q_i\in S}\lambda_i - \phi)^2}, \label{geo1}
 \end{align}
where $1)$ follows from the Cauchy-Schwarz inequality and $2)$ is due to Lemma~\ref{perpendicular}. Since $\triangle(y,q) > 0$, letting $\phi \rightarrow \infty$ reduces equation~\eqref{geo1} to $\frac{\triangle^2(y,q)}{\|q\|_{\tilde{W}}^2 + 1}$. 
\end{proof}

Next, using Lemma~\ref{Wolfeconv2}, we may characterize the decrease of the objective function for one MAJOR loop with no MINOR loop. As $(y^{(k+1)}, S^{(k+1)}, \lambda^{(k+1)})$ is a good triple and $y^{(k)}$ also lies in $\text{lin}(S)$, we have the following result.
\begin{lemma}\label{Wolfeconv3}
Consider some MAJOR loop $k$ without MINOR loops. Then 
\begin{align*}
\text{err}(y^{(k)})  - \text{err}(y^{(k+1)}) \geq\frac{\triangle^2(y^{(k)}, q^{(k)})}{Q^2 + 1} = \frac{\triangle^2(y^{(k)})}{Q^2 + 1}.
\end{align*}
\end{lemma}

Next, we characterize the decrease of the objective function for one MAJOR loop with one single MINOR loop.
\begin{lemma}\label{Wolfeconv4}
Consider some MAJOR loop $k$ with only one MINOR loop. Then 
\begin{align*}
\text{err}(y^{(k)})  - \text{err}(y^{(k+2)}) \geq\frac{\triangle^2(y^{(k)})}{Q^2 + 1}.
\end{align*}
\end{lemma}
\begin{proof}
Suppose that the active sets associated with $y^{(k)}, y^{(k+1)}, y^{(k+2)}$ are $S^{(k)}, S^{(k+1)}, S^{(k+2)},$ respectively. We know that within the MINOR loop, $(z^{(k)}, S^{(k)}\cup\{q^{(k)}\}, \alpha)$ is a good triple and $y^{(k+1)} = \theta y^{(k)} + (1- \theta) z^{(k)},$ for some $\theta\in [0,1]$. Let
\begin{align}\label{gap1}
A = \|y^{(k)}-a\|_{\tilde{W}}^2 + (\sum_{q_i\in S^{(k)}}\lambda_i^{(k)})^2 -  \|z^{(k)}-a\|_{\tilde{W}}^2 -  (\sum_{q_i\in S^{(k)}\cup\{q^{(k)}\}}\alpha_i \;)^2.
\end{align}
From Lemma~\ref{Wolfeconv2}, we have
\begin{align}\label{gap2}
A \geq \frac{\triangle^2(y^{(k)},q^{(k)})}{\|q^{(k)}\|_{\tilde{W}}^2 + 1} \geq \frac{\triangle^2(y^{(k)})}{Q^2 + 1}.
\end{align}
Note that both $S^{(k)}$ and $S^{(k+1)}$ are subsets of $S^{(k)}\cup\{q^{(k)}\}$. As $(z^{(k)}, S^{(k)}\cup\{q^{(k)}\}, \alpha)$ is a good triple, using Lemma~\ref{perpendicular}, we obtain 
$$\langle z^{(k)}- a,  z^{(k)} - y^{(k)} \rangle_{\tilde{W}}+\langle \sum_{q_i\in S^{(k)}\cup\{q^{(k)}\}} \alpha_i, \sum_{q_i\in S^{(k)}\cup\{q^{(k)}\}}  (\alpha_i- \lambda_i^{(k)}) \rangle = 0.$$ 
Furthermore, as $y^{(k+1)} = \theta y^{(k)} + (1- \theta) z^{(k)}= z^{(k)} - \theta(z^{(k)} - y^{(k)}) $ and $\lambda^{(k+1)} = \alpha  - \theta ( \alpha -\lambda^{(k)})$, we have
\begin{align*}
\|y^{(k)}-a\|_{\tilde{W}}^2 + (\sum_{i\in S^{(k)}}\lambda_i^{(k)})^2 -  \|y^{(k+1)}-a\|_{\tilde{W}}^2- (\sum_{i\in S^{(k+1)}}\lambda_i^{(k+1)})^2 = (1 - \theta^2)A.
\end{align*}
Moreover, we have 
\begin{align}\label{onestep}
\triangle(y^{(k+1)}, q^{(k)}) = \theta\triangle(y^{(k)}, q^{(k)}) + (1- \theta)\triangle(z^{(k)}, q^{(k)}) = \theta\triangle(y^{(k)}, q^{(k)}) = \theta\triangle(y^{(k)}),
\end{align}
which holds because $\triangle(y, q)$ is linear in $y$ and Lemma~\ref{perpendicular} implies $\triangle(z^{(k)}, q^{(k)}) = 0$. 

Since according to Lemma~\ref{strictly}, $h(y^{(k)}, \phi^{(k)})> h(y^{(k+1)}, \phi^{(k+1)})$, Lemma 4 in~\cite{chakrabarty2014provable} also holds in our case, and thus $q^{(k)} \in S^{(k+1)}$. To obtain $y^{(k+2)}$ and $S^{(k+2)}$, one needs to remove active points with a zero coefficients from $S^{(k+1)}$, so that $y^{(k+2)}$ once again belongs to a good triple with corresponding $S^{(k+1)}$. 
Based on~\ref{Wolfeconv2} and equation~\eqref{onestep}, we have the following result.
\begin{align}\label{gap3}
\|y^{(k+1)}-a\|_{\tilde{W}}^2 + (\sum_{q_i\in S^{(k+1)}}\lambda_i^{(k+1)})^2 -  &\|y^{(k+2)}-a\|_{\tilde{W}}^2 - (\sum_{q_i\in S^{(k+2)}}\lambda_i^{(k+2)})^2\\
 \geq& \frac{\triangle^2(y^{(k+1)}, q^{(k)})}{Q^2 + 1} = \frac{\theta^2\triangle^2(y^{(k)}) }{Q^2 + 1}.
\end{align}
Consequently, combining equations~\eqref{gap1},~\eqref{gap2} and~\eqref{gap3}, we arrive at
\begin{align*}
\text{err}(y^{(k)})  - \text{err}(y^{(k+2)})  \geq \frac{\triangle^2(y^{(k)})}{Q^2 + 1}.
\end{align*}
\end{proof}
And, combining Lemma~\ref{Wolfeconv3}, Lemma~\ref{Wolfeconv4} and Lemma~\ref{errorconnect} establishes Theorem~\ref{descent}.

\section{The Conic Frank-Wolfe Algorithm for Solving~\eqref{projection}} 
\label{appsec:FW}
The Frank-Wolfe algorithm adapted to cones is summarized in Algorithm 4. Note that Steps 2, 4 are specialized for cones, similarly to the Steps 2 and 5 of Algorithm 2. The main difference between the FW and MNP methods is the size of active set: FW only maintains two active points, while MNP may maintain as many as $N$ points. Theorem~\ref{FWalgCR} establishes that the conic FW algorithm also has a $1/k-$convergence rate. 
 
 \begin{table}[htb]
\centering
\begin{tabular}{l}
\hline
\textbf{Algorithm 4:} \textbf{The Conic FW Algorithm for Solving~\eqref{projection}} \\
\hline
\ \textbf{Input}:  $W$, $a$, $B$ and a small positive $\delta$\\
Initialize $y^{(0)} \leftarrow 0 $, $ \phi^{(0)}\leftarrow 0$ and $k\leftarrow 0$\\ 
1. \textbf{Iteratively execute the following steps:} \\
2. \;\; $q^{(k)} \leftarrow \arg\min_{q\in B} \langle \nabla_{y}h(y^{(k)}, \phi^{(k)}) , q \rangle $\\
3. \;\;  \textbf{If} $\langle y^{(k)} - a, q^{(k)} \rangle_{\tilde{W}} + \phi^{(k)}\geq -\delta$, \textbf{break}.\\
4. \;\; \textbf{Else}: $(\gamma_1^{(k)}, \gamma_2^{(k)}) \leftarrow  \arg\min_{\gamma_1 \geq 0,\gamma_2\geq 0} h(\gamma_1y^{(k)} + \gamma_2 q^{(k)}, \gamma_1^{(k)} \phi^{(k)}  + \gamma_2^{(k)})$ \\% \|\gamma_1y^{(t)} + \gamma_2 q^{(t)} - a\|_W^2 + (\gamma_1^{(t)} r^{(t)}  + \gamma_2^{(t)} )^2$ \\
5. \;\;\;\;\;\;\;\; \;\;  $y^{(k+1)} \leftarrow \gamma_1^{(k)} y^{(k)} + \gamma_2^{(k)} q^{(k)}, \, \phi^{(k+1)}\leftarrow  \gamma_1^{(k)} \phi^{(k)}  + \gamma_2^{(k)}$, \, $k\leftarrow k+1$.\\%\inf \{\phi: y^{(k+1)}\in \phi B\}$. \\
\hline
\end{tabular}
\end{table}

\begin{theorem}\label{FWalgCR}
Let $B$ be an arbitrary polytope in $\mathbb{R}^{N}$ and let $C = \{(y,\phi)| y\in \phi B, \phi\geq 0\}$ denote its corresponding cone. For some positive diagonal matrix $\tilde{W}$, define $Q = \max_{q\in B} \|q\|_{\tilde{W}}$. Then, the point $(y^{(k)}, \phi^{(k)})$ generated by Algorithm 3 satisfies $h(y^{(k)}, \phi^{(k)})  \leq h^* + \frac{2\|a\|_{\tilde{W}}^2 Q^2}{k+2}$.
\end{theorem}

\begin{proof}
Using the same strategy as in the proof of Lemma~\ref{boundonr}, provided the optimality of $\gamma_1^{(k)} \phi^{(k)}  + \gamma_2^{(k)}$ in step 3 of Algorithm 3, we have the following Lemma.  
\begin{lemma}
In Algorithm 3, for all $k$, $\phi^{(k+1)}\leq \frac{\|a\|_{\tilde{W}}}{2}$.
\end{lemma}
%Therefore, as an optimal solution, $\gamma_1^{(t)} r^{(t)}  + \gamma_2^{(t)}\leq \frac{\|a\|_{\tilde{W}}}{2}$. Step 4 implies $y^{(t+1)} \in \gamma_1^{(t)}r^{(t)} B + \gamma_2^{(t)} B$ which surely induces $r^{(t+1)} = \inf \{r: y^{(t+1)}\in rB\} \leq \gamma_1^{(t)}r^{(t)} + \gamma_2^{(t)}\leq  \frac{\|a\|_{\tilde{W}}}{2}$. 

Now, we prove Theorem~\ref{FWalgCR}. We write $y = \phi \hat{y},$ where $\hat{y}\in B$, so that
\begin{align*}
h(y,\phi)= h( \phi \hat{y}, \phi )  \geq h( y^{(k)}, \phi^{(k)})  + 2\langle y^{(k)} -a , \phi\hat{y} - y^{(k)}\rangle_{\tilde{W}} + \phi^2 - (\phi^{(k)})^2. 
\end{align*}
For both sides, minimize $(\hat{y},\phi)$ over $B \times [0, \frac{\|a\|_{\tilde{W}}}{2}]$, which contains $(y^*, \phi^*)$. Since $q^{(k)} = \arg\min_{q\in B} ( y^{(k)} -a , q \rangle_{\tilde{W}}$, we know the optimal $\hat{y} = q^{(k)} $ and the optimal $\phi =\tilde{\phi}^{(k)} = \min\{\max\{0, -\langle  y^{(k)}-a,  q^{(k)} \rangle \}, \frac{\|a\|_{\tilde{W}}}{2}\}$ of the RHS
\begin{align}\label{FWalgproof1}
h^*  =h(y^*, \phi^*)   \geq h( y^{(k)}, \phi^{(k)})   + 2\langle  y^{(k)} -a , \tilde{\phi}^{(k)}q^{(k)} - y^{(k)}\rangle_{\tilde{W}} + (\tilde{\phi}^{(k)})^2 - (\phi^{(k)})^2. 
\end{align}
%Because of the convexity of $B$, $r^{(t+1)}\leq \gamma_1^{(t)}r^{(t)} + \gamma_2^{(t)}$. Therefore, $\tilde{g}( y^{(t+1)},r^{(t+1)}) \leq \tilde{g}( \gamma_1^{(t)} y^{(t)} + \gamma_2^{(t)} q^{(t)},  \gamma_1^{(t)}r^{(t)}+ \gamma_2^{(t)}) $.\\
Moreover, because of the optimality of $(\gamma_1^{(k)}\times\gamma_2^{(k)})\in\mathbb{R}_{\geq0}^2$, for arbitrary $\gamma \in [0,1]$ we have 
\begin{align*}
 &\quad h(\gamma_1^{(k)} y^{(k)} + \gamma_2^{(k)} q^{(k)},  \gamma_1^{(k)}\phi^{(k)}+ \gamma_2^{(k)})  \\
& \leq h((1-\gamma) y^{(k)}+ \gamma\tilde{\phi}^{(k)} q^{(k)}, (1-\gamma)\phi^{(k)} + \gamma\tilde{\phi}^{(k)}) \\
&= h(y^{(k)}, \phi^{(k)}) + 2\gamma \langle y^{(k)} - a,  \tilde{\phi}^{(k)}q^{(k)} - y^{(k)} \rangle_{\tilde{W}} \\
& +\gamma [(\tilde{\phi}^{(k)})^2 - (\phi^{(k)})^2] + \gamma^2 \|\tilde{\phi}^{(k)}q^{(k)} - y^{(k)}\|_{\tilde{W}}^2 + (\gamma^2 - \gamma) (\tilde{\phi}^{(k)} - \phi^{(k)})^2 \\
&\stackrel{1)}{\leq}  h(y^{(k)},  \phi^{(k)}) + \gamma ( h^* - h(y^{(k)},  \phi^{(k)})) + \gamma^2 \|\tilde{\phi}^{(k)}q^{(k)} - y^{(k)}\|_{\tilde{W}}^2 \\
&\stackrel{2)}{\leq} h(y^{(k)},  \phi^{(k)}) + \gamma ( h^* -  h(y^{(k)},  \phi^{(k)})) + \gamma^2 \|a\|_{\tilde{W}}^2 Q^2,
\end{align*}
where $1)$ follows from \eqref{FWalgproof1} and $\gamma^2 -\gamma \leq 0$, and $2)$ follows from 
$$\|\tilde{\phi}^{(k)}q^{(k)}- y^{(k)} \|_{\tilde{W}}^2  \leq 4\frac{\|a\|_{\tilde{W}}^2}{4} \max_{q\in B}\|q\|_{\tilde{W}}^2 =\|a\|_{\tilde{W}}^2 Q^2.$$ 
Now, we prove the claimed result by induction. First, let $\hat{y}^* = y^*/\phi^{*},$ where $\phi^{*} = \frac{\langle \hat{y}^*, a \rangle_{\tilde{W}}}{1+\|\hat{y}^*\|_{\tilde{W}}^2}$. Therefore, 
\begin{align*}
h(y^{(0)},  \phi^{(0)}) - h^* \leq 2\langle y^*, a\rangle - ( y^*)^2 - (\phi^{*})^2 = \frac{\langle \hat{y}^*, a \rangle_{\tilde{W}}^2}{1+\|\hat{y}^*\|_{\tilde{W}}^2} \leq \|a\|_{\tilde{W}}^2 Q^2 .
\end{align*}
Suppose that $h(y^{(k)},  r^{(k)}) - h^* \leq \frac{2\|a\|_{\tilde{W}}^2 Q^2 }{k+2}.$ Then, for all $\gamma\in[0,1]$, we have 
\begin{align*}
 h(y^{(t+1)},  \phi^{(t+1)}) - h^* \leq  (1-\gamma) [h(y^{(k)},  \phi^{(k)}) - h^*]  +  \gamma^2 \|a\|_{\tilde{W}}^2 Q^2 .
 \end{align*}
 By choosing $\gamma = \frac{1}{k+2}$, we obtain $h(y^{(k+1)},  \phi^{(k+1)}) - h^* \leq  \frac{2\|a\|_{\tilde{W}}^2 Q^2}{k+3},$ which concludes the proof. 
\end{proof}

\section{Iteration Cost of the MNP and FW Methods}

We want to estimate the complexity of an iteration in the MNP and FW methods. For this purpose, let $\text{EO}$ stand for the time to query the value of $F$. The linear program in Step 2 of both algorithms requires $O(N\log N + N \times EO)$ operations using a greedy algorithm. Step 5 of the MNP method requires one to solve a quadratic program with no constraints, and the solution may be obtained in closed form using $O(N|S|^2)$ operations. Other operations in the MNP method introduce a $O(N)$ complexity term. Hence, the execution of each MAJOR or MINOR loop requires $O(N\log N + N\times EO + N|S|^2)$ operations. In the FW method, the optimization problem in Step 4 is relatively simple, since it reduces to solving a nonnegative quadratic program with only two variables and consequently induces a complexity term $O(N)$. Therefore, the complexity of the FW algorithm is dominated by Step 2, which requires $O(N\log N + N\times EO)$ operations.
%We next estimate the complexity of an iteration in the MNP methods. For this purpose, let $\text{EO}$ stand for the time to query the value of $F$. The linear program in Step 2 requires $O(N\log N + N \times EO)$ operations using a greedy algorithm. The step 5 requires one to solve a quadratic program with no constraints, and the solution may be obtained in closed form using $O(N|S|^2)$ operations. Other operations in the MNP method introduce a $O(N)$ complexity term. Hence, the execution of each MAJOR or MINOR loop requires $O(N\log N + N\times EO + N|S|^2)$ operations.

\section{An Exact Projection Algorithm for Directed/Undirected Hyperedges}
\label{appsec:exactprojection}
The key idea is to revert the projection $\Pi_{C}(a)$ back to its primal form. In this setting, optimization becomes straightforward and only requires careful checking of the values of gradients. First, following a similar strategy as described in Lemma 3.2, it is easy to show that 
\begin{align*}
\min_{z} \|z - b\|_{W}^2 + [f(z)]^2
\end{align*}
is dual to the problem~\eqref{projection}, where $b = \frac{1}{2}W^{-1}a$, and $f$ is the Lov{\'a}sz extension corresponding to the base polytope $B$. The point $(y, \phi)$ may be recovered via $y = a - 2Wz$, $\phi= 2\langle y, z \rangle_{W^{-1}}$. In the context of directed hypergraph learning (see Table 1), for a directed hyperedge with head and tail sets $H, T$, respectively, $f(z) = \max_{i\in H}z_i - \min_{j\in T}z_j$. Note that undirected hypergraph learning can be viewed as a special case of directed hypergraph learning when $H,\,T$ are both identical to the corresponding incidence set. To solve this optimization problem, define two intermediate variables $\gamma = \max_{i\in H}z_i $ and $\delta = \min_{j\in T}z_j$. Denote the derivates with respect to $\gamma$ and $\delta$ as $\triangle_\gamma$ and $\triangle_\delta$, respectively. The optimal $\gamma$ and $\delta$ satisfy 
\begin{align*}
\triangle_\gamma = \gamma - \delta + \sum_{i\in S_H(\gamma)} W_i(\gamma - b_i), \quad \triangle_\delta = \delta - \gamma + \sum_{j\in S_T(\delta) } W_j(\delta - b_j),
\end{align*}
where $S_H(\gamma) = \{i| i\in H, b_i \geq \gamma\}$ and $S_T(\delta) = \{j| j\in T, b_j \leq \delta\}$. 
The optimal values of $\gamma$ and $\delta$ should simultaneously satisfy $\triangle_\gamma=0$ and $\triangle_\delta=0$. 
Algorithm 4 can be used to find such $\gamma$ and $\delta$. The search for $(\gamma, \delta)$ starts from $(\max_{i\in H} b_i, \min_{j\in T} b_j)$, and one gradually decreases $\gamma$ and increases $\delta$ while keeping $\triangle_\gamma = -\triangle_\delta$ (Steps 5-10). 
%Finally, we can achieve $\triangle_\gamma = -\triangle_\delta=0$. 
The complexity of Algorithm 4 is dominated by the sorting step, which requires $O(N\log N)$ operations, as $w_H$, $w_T$, $\triangle_\gamma$, $\triangle_\delta,$ can all be efficiently tracked within the inner loops.  

\begin{table}[htb]
\centering
\begin{tabular}{l}
\hline
\label{generalalg}
\textbf{Algorithm 4: } \textbf{Projection Algorithm for a Directed Hyperedge} \\
\hline
\textbf{Input}:  $W$, $b$, $H$, $T$\\
%Let $(y, r)$ satisfies $y \in B$, $r\leftarrow \arg\min_r \|ry - a\|_W^2 + r^2$, $S \leftarrow \{y\}$, $\lambda_1 \leftarrow 1$.  \\ 
1.\; Sort $\{b_i\}_{i\in H}$ and $\{b_j\}_{j\in T}$. \\
2.\; Initialize $\gamma \leftarrow \max_{i\in H} b_i$ and $\delta \leftarrow \min_{j\in T} b_j$. \\
3.\; \textbf{If} $\gamma \leq \delta$, \textbf{return} $z = b$. \\
4.\; Iteratively execute: \\
5.  \;\;\;\;  $w_H\leftarrow \sum_{i\in S_H(\gamma)}W_v$, $w_T \leftarrow \sum_{j\in S_T(\delta)}W_v$ \\
6.  \;\;\;\;  $\gamma_1 \leftarrow \max_{i\in H/S_H(\gamma)} b_v$, $\delta_1 \leftarrow \delta + (\gamma - \gamma_1)w_H/w_T $\\
7.  \;\;\;\;  $\delta_2 \leftarrow \min_{j\in T/S_T(\delta)} b_v$, $\gamma_2 \leftarrow \gamma - (\delta_2 - \delta)w_T/w_H  $ \\
8.  \;\;\;\;  $k* \leftarrow \arg\min_{k\in\{1,2\}} \delta_k $ \\
9.  \;\;\;\;  \textbf{If} $\gamma_{k^*} \leq \delta_{k^*}$ or $\triangle_{\gamma_{k^*}} \leq 0$, \textbf{break} \\
10.\;\;\;\;  $(\gamma, \delta) \leftarrow (\gamma_{k^*}, \delta_{k^*})$ \\
11.\;$(\gamma, \delta) \leftarrow  ( w_T,  w_H)\frac{\triangle_{\gamma}}{w_Tw_H + w_T + w_H}$ \\
12.\;Set $z_i$  to $\gamma$, if $i\in S_H(\gamma)$, $\delta$, if $i\in S_T(\delta)$, and $b_i$ otherwise. \\
\hline
\end{tabular}
\end{table}

\section{Additional Experiments and Descriptions}
\subsection{Experimental Setting and Dataset Descriptions}
As it may be time consuming to find the precise projections via MNP and FW, we always fix the number of MAJOR loops of the MNP and the number of iterations of the FW method to $100|S_r|$ and $100|S_r|^2,$ respectively. Empirically, these choices provide an acceptable trade-off between accuracy and complexity. 

PDHG and SGD depend on some parameters. We choose them in standard ways: for PDHG, we set $\sigma = \tau = \frac{1}{\sqrt{1+ \max_{i} D_{ii}}}$ and for the SGD, we set $\eta_k = \frac{1}{k\beta\max_{i} W_{ii}}$. 

The nearly optimal tuned $\beta$ for the table in Figure.~\eqref{fig:synthetic1} is: For QDSFM as the objective, including methods QRCD-SPE, QAP-SPE, PDHG, and SGD, we set $\beta = 0.02$; For DSFM as the objective, including the DRCD method, we set $\beta = 1$; And for InvLap, we set $\beta = 0.001$. 
%To test the effect of $\beta$, we fix $W_{ii} = 1$ for all $i$, and vary $\beta \in [10^{-3}, 10^{3}]$. To test the effect of $W$, we fix $\beta = 1$ and randomly choose half of the vertices and set their $W_{ii}$ values to lie in $\{1, 10, 10^2, 10^3\}$, and set the remaining ones to $1$. 
%The logarithm of gap ratios is proportional to $\log\beta$, when $\beta$ is small. It implies that the number of iterations needed to achieve \textcolor{red}{same accuracy? not sure what you mean} is proportional to $\log\beta^{-1}$, and not as sensitive as suggested by Theorem~\ref{linearconv}. Moreover, as predicted, when $\beta$ is relatively large ($>1$), the dependence on $\beta$ levels out. The dependence on $\max_{i, j\in[N]}\frac{W_{ii}}{W_{jj}}$ is also not as strong as suggested by Theorem~\ref{linearconv} due to the use of the logarithmic scale.  

The properties of the hypergraphs generated based on real datasets are summarized in Table~\ref{tab:dataset}. 
\begin{table*}[h] 
\begin{tabular}{|p{1.6cm}<{\centering}|p{1.6cm}<{\centering}|p{2.6cm}<{\centering}|p{2.6cm}<{\centering}|}
\hline
Dataset & Mushroom & Covertype$45$ & Covertype$67$ \\
\hline
$N$ &  8124 & 12240 & 37877 \\
$R$ &  112 &  127 & 136 \\
$\sum_{r\in [R]} |S_r|$ & 170604 & 145999 & 451529 \\
\hline
\end{tabular}
\centering 
\caption{The UCI datasets used for experimental testing.}
\label{tab:dataset}
\end{table*}

\subsection{Derivation of the Theoretical Parameter Dependence for the Experiments on Synthetic Data}\label{sec:para-dep}
Let $x' = W^{-1/2}x$ and $a' = W^{-1/2}a$. Then, we can transform the objective~\eqref{expobj} into a standard QDSFM problem:
\begin{align*}
\beta \|x'- a'\|_W^2 + \sum_{r\in[R]}\max_{i,j\in S_r}(x_i' - x_j')^2
\end{align*}
According to Theorem~\ref{linearconv}, we know that in order to achieve an $\epsilon$-optimal solution, one requires $O(R\mu(\beta^{-1}W^{-1}, \beta^{-1}W^{-1})\log \frac{1}{\epsilon})$ iterations for the RCD algorithm (Algorithm 1). According to the particular setting of the experiment (undirected unweighted hypergraphs), we have 
\begin{align}\label{exp-rho}
\rho^2 = \max_{y_r\in B_r, \forall r} \sum_{r\in [R]}\|y_r\|_1^2=  \max_{y_r\in B_r, \forall r}\sum_{r\in [R]}2= 2R
\end{align}
From the definition of $\mu$~\eqref{defmu}, we may tightly characterize $\mu(\beta^{-1}W^{-1}, \beta^{-1}W^{-1})$ as:
\begin{align}
\mu(\beta^{-1}W^{-1}, \beta^{-1}W^{-1}) &\stackrel{1)}{=} \max \left\{ \frac{N^2}{2}\left(\max_{i,j\in[N]} \frac{W_{ii}}{W_{jj}} + 1\right),  \frac{9}{4}\rho^2 N\beta^{-1}\max_{j\in[N]}\frac{1}{W_{jj}}\right\} \nonumber \\
& \stackrel{2)}{=} \max \left\{ \frac{N^2}{2} \left(\max_{i,j\in[N]} \frac{W_{ii}}{W_{jj}} + 1\right),  \frac{9}{2}\beta^{-1} NR\max_{j\in[N]}\frac{1}{W_{jj}}\right\}\nonumber \\
& \stackrel{3)}{=} \max \left\{ \frac{N^2}{2} \left(\max_{i,j\in[N]} \frac{W_{ii}}{W_{jj}} + 1\right),  \frac{9}{2}\beta^{-1} N^2\max_{i,j\in[N]}\frac{W_{ii}}{W_{jj}}\right\} \label{exp-para}
\end{align}
where $1)$ holds because half of the values of $W_{ii}$ are set to 1 and the other half to a value in $\{1, 0.1, 0.01, 0.001\}$, $2)$ follows from \eqref{exp-rho} and $3)$ is due to the particular setting $N=R$ and $\max_{i\in[N]}W_{ii} = 1$. The equation~\eqref{exp-para} can be rewritten as $O(N^2 \max(1, 9/(2\beta))\max_{i, j\in[N]}W_{ii}/W_{jj}),$ which establishes the corresponding statement. 

\subsection{Experiments Illustrating the Convergence Rate of QRCD-MNP and QRCD-FW}\label{sec:mnp-fw}
We also provide some experimental results to demonstrate the convergence rate of QRCD-MNP and QRCD-FW. For the QDSFM problem~\eqref{QDSFM}, we fix $N = 100$, $R = 100$, and $W_{ii} = 1,$ for all $i\in[N]$. We generate each incidence set $S_r$ by choosing uniformly at random a subset of $[N]$ of cardinality $10$, and set the entries of $a$ to be iid standard Gaussian. We consider the following submodular functions: for $r\in[R]$, $S\subseteq S_r$
\begin{align*}
F_r(S) = \frac{\min\{|S|, |S_r/S|\}^\theta}{ (|S_r|/2)^{\theta}}, \quad \quad \theta\in \{0.25,\,0.5,\,1\}.
\end{align*}
The number of RCD iterations is set to $300R = 3\times 10^4$. The convergence results are shown in Figure.\ref{fig:MNP-FW}. As may be seen, QRCD-FW takes longer CPU-time to converge than QRCD-MNP. Moreover, QRCD-FW cannot achieve high accuracy  because of the inaccurate inner-loop projections.

\begin{figure*}[h]
\includegraphics[trim={0cm 0cm 0cm 0},clip,width=.32\textwidth]{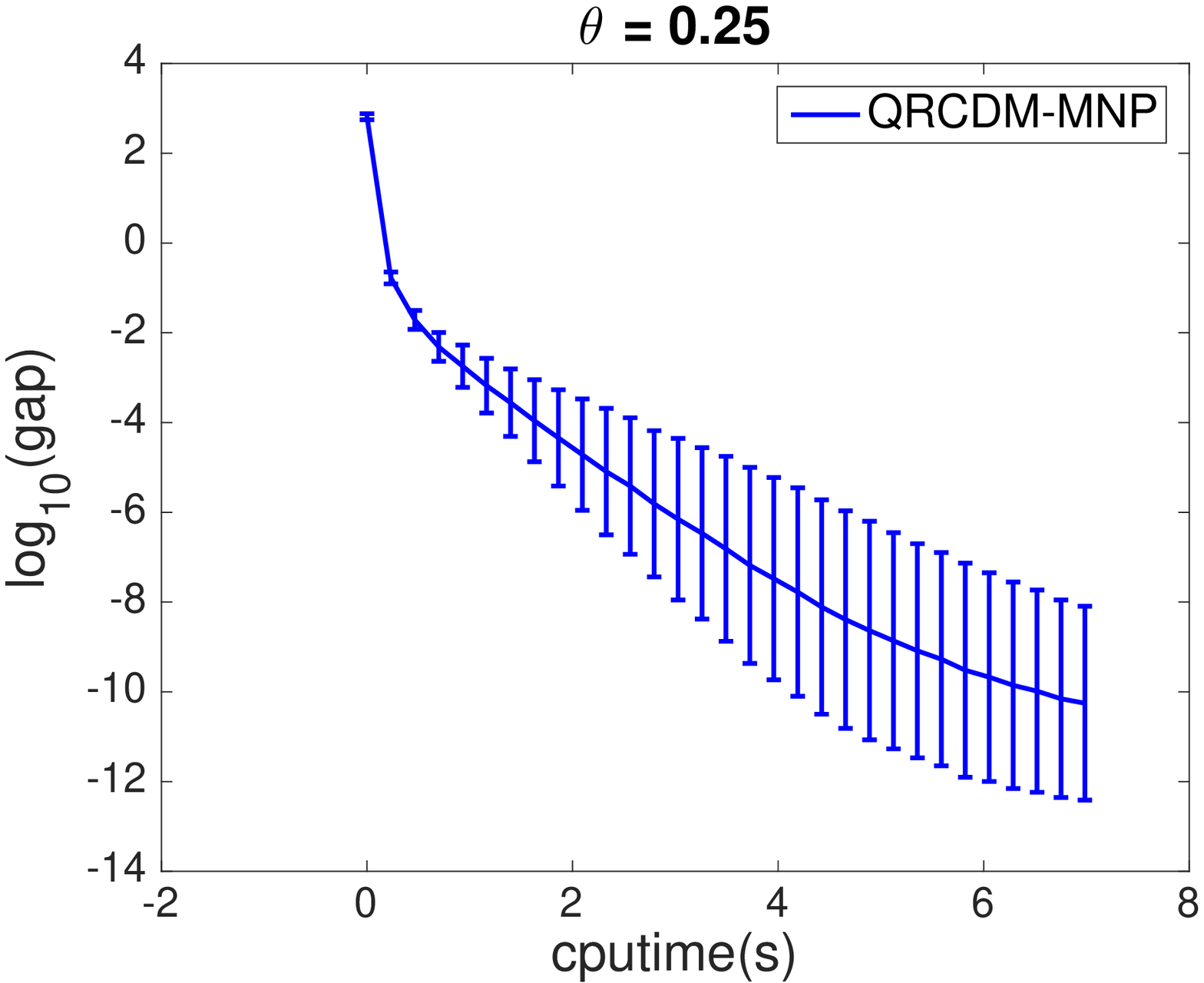}
\includegraphics[trim={0cm 0cm 0cm 0},clip,width=.32\textwidth]{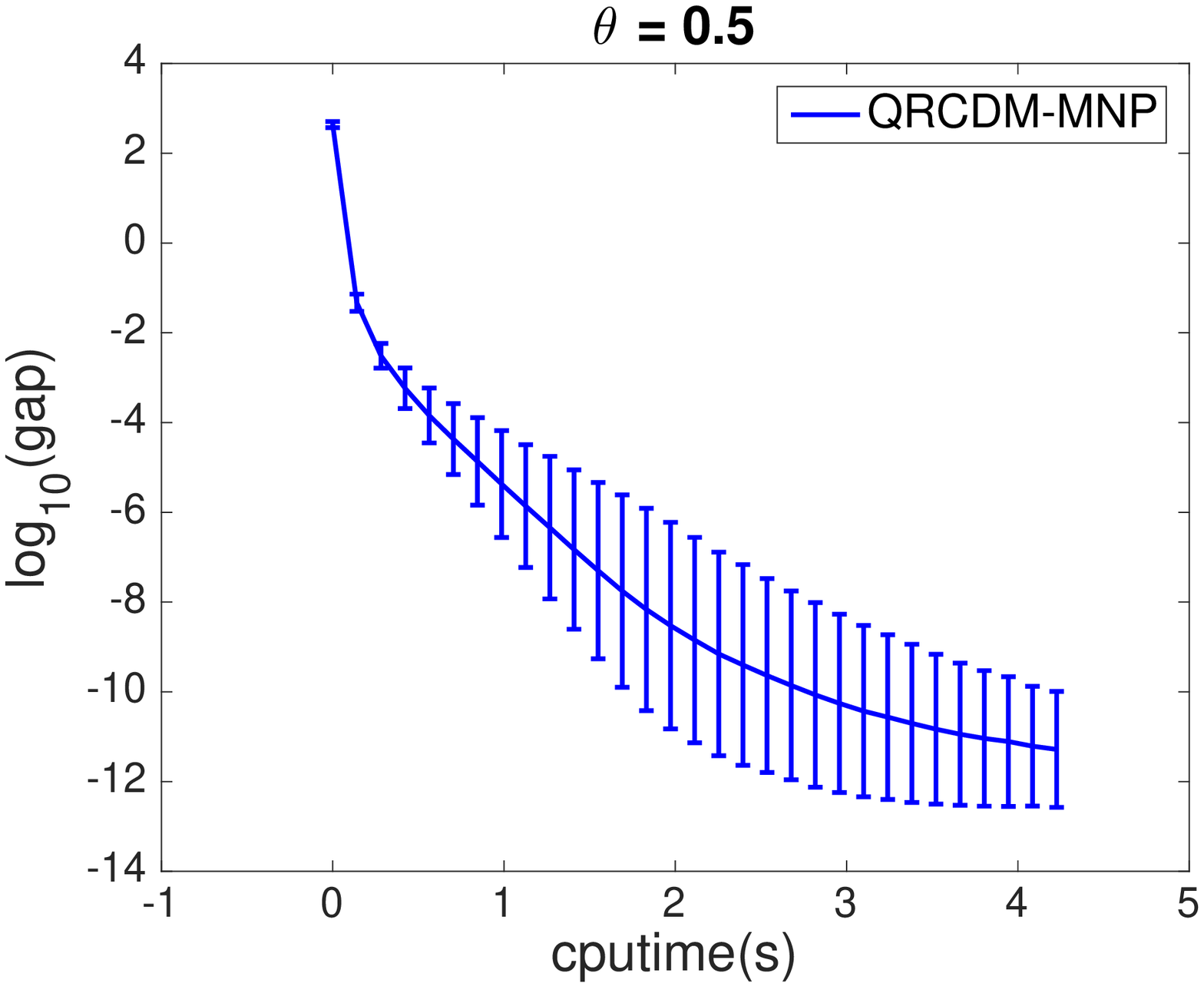}
\includegraphics[trim={0cm 0cm 0cm 0},clip,width=.32\textwidth]{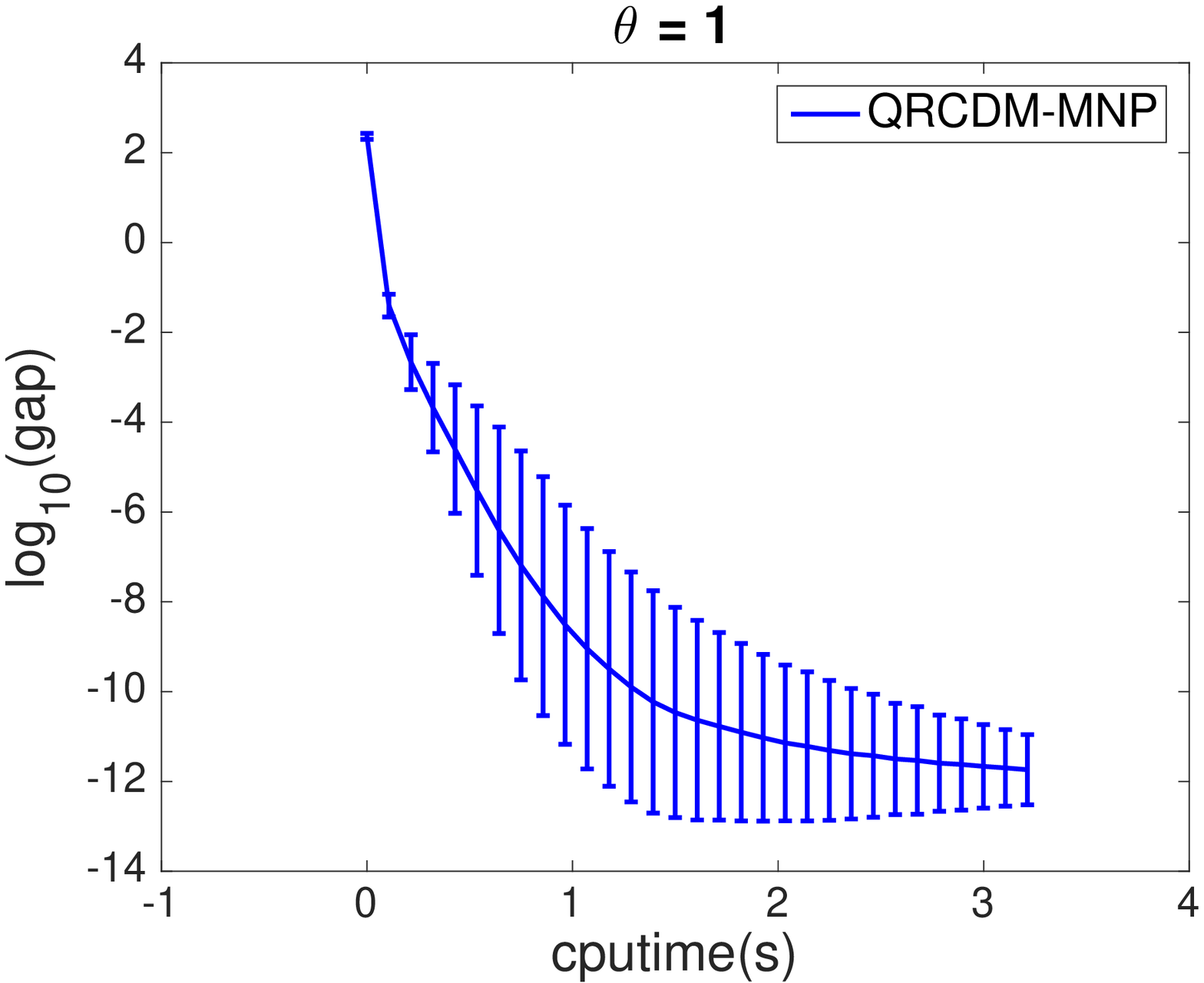}
\includegraphics[trim={0cm 0cm 0cm 0},clip,width=.33\textwidth]{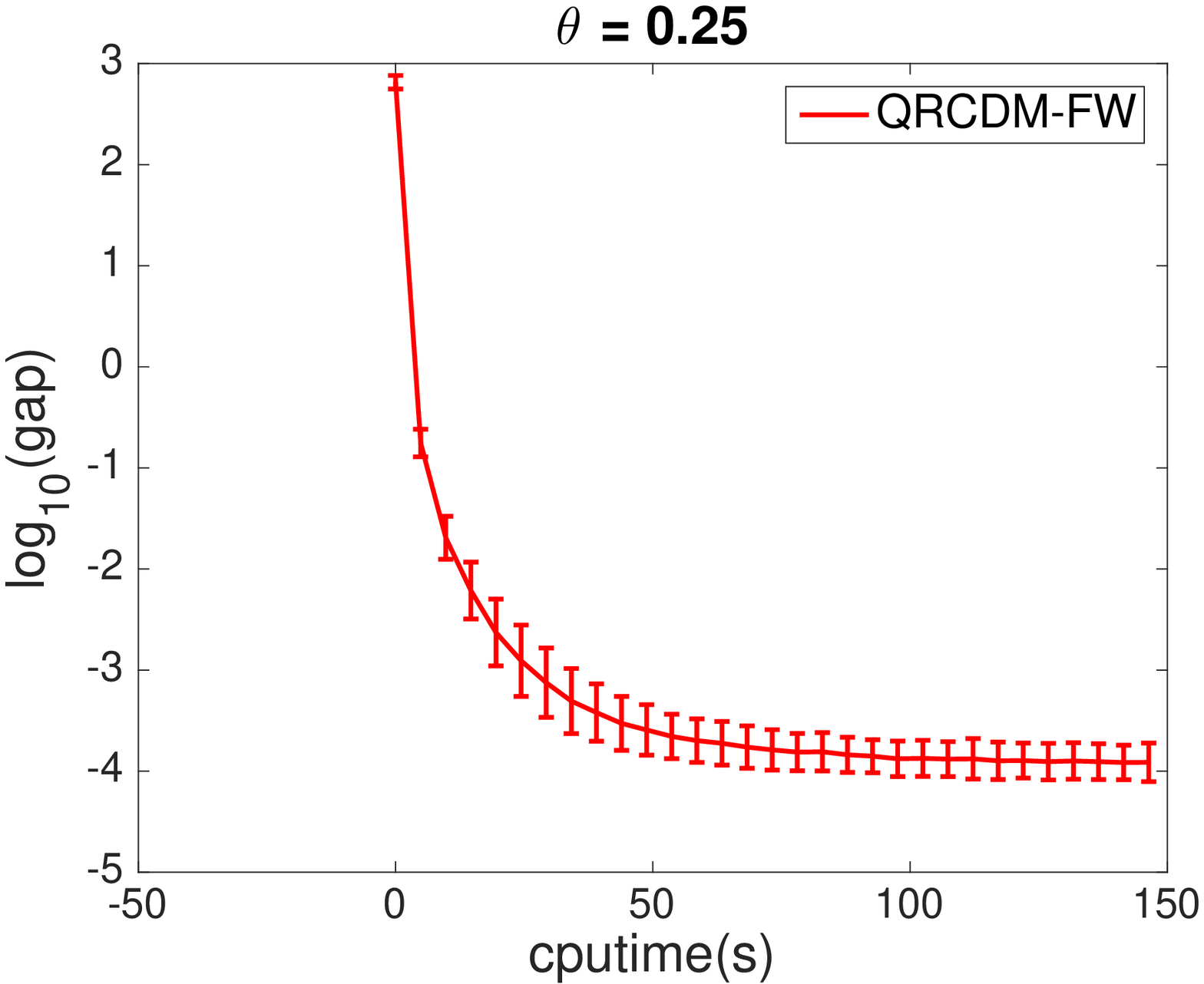}
\includegraphics[trim={0cm 0cm 0cm 0},clip,width=.33\textwidth]{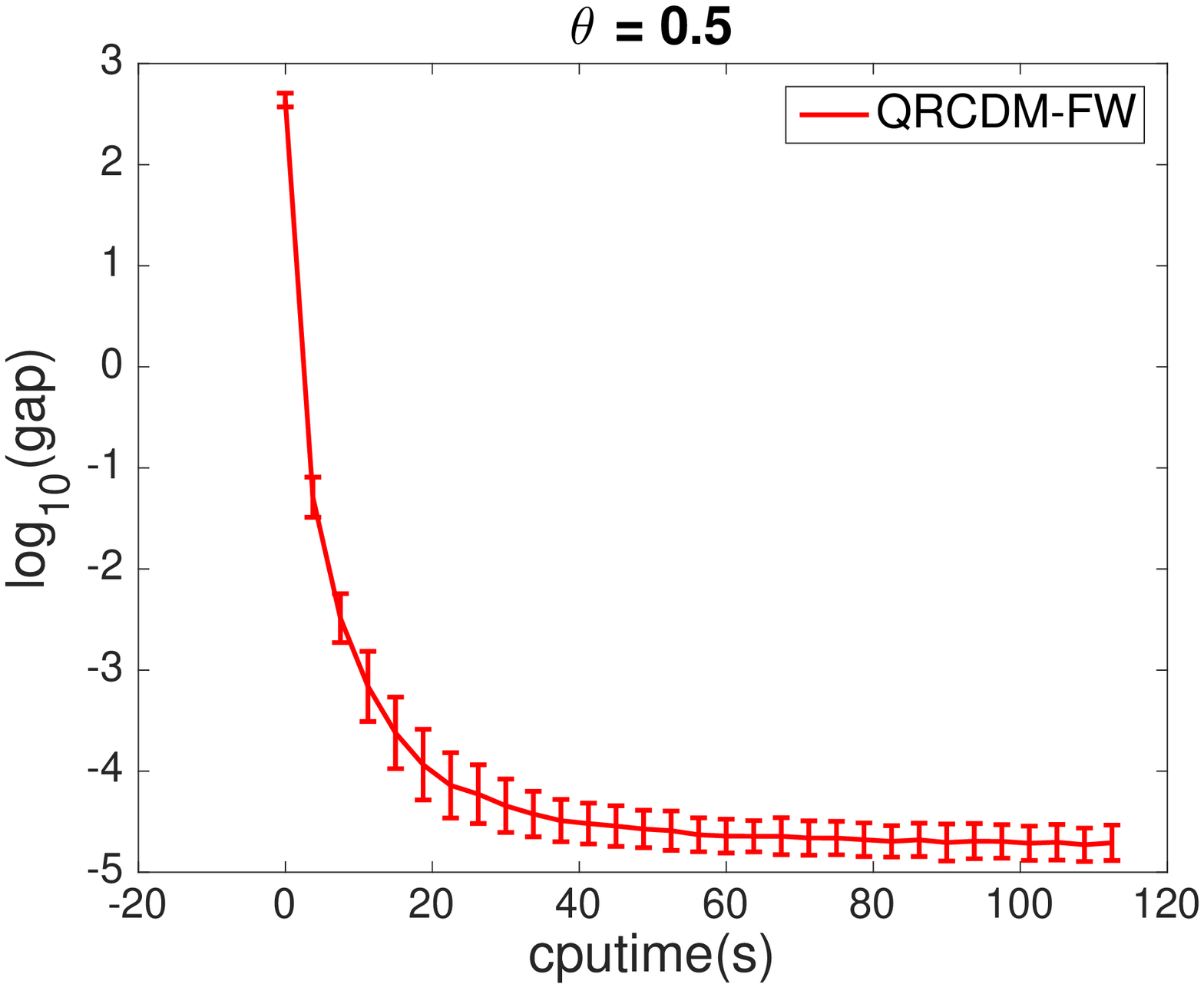}
\includegraphics[trim={0cm 0cm 0cm 0},clip,width=.33\textwidth]{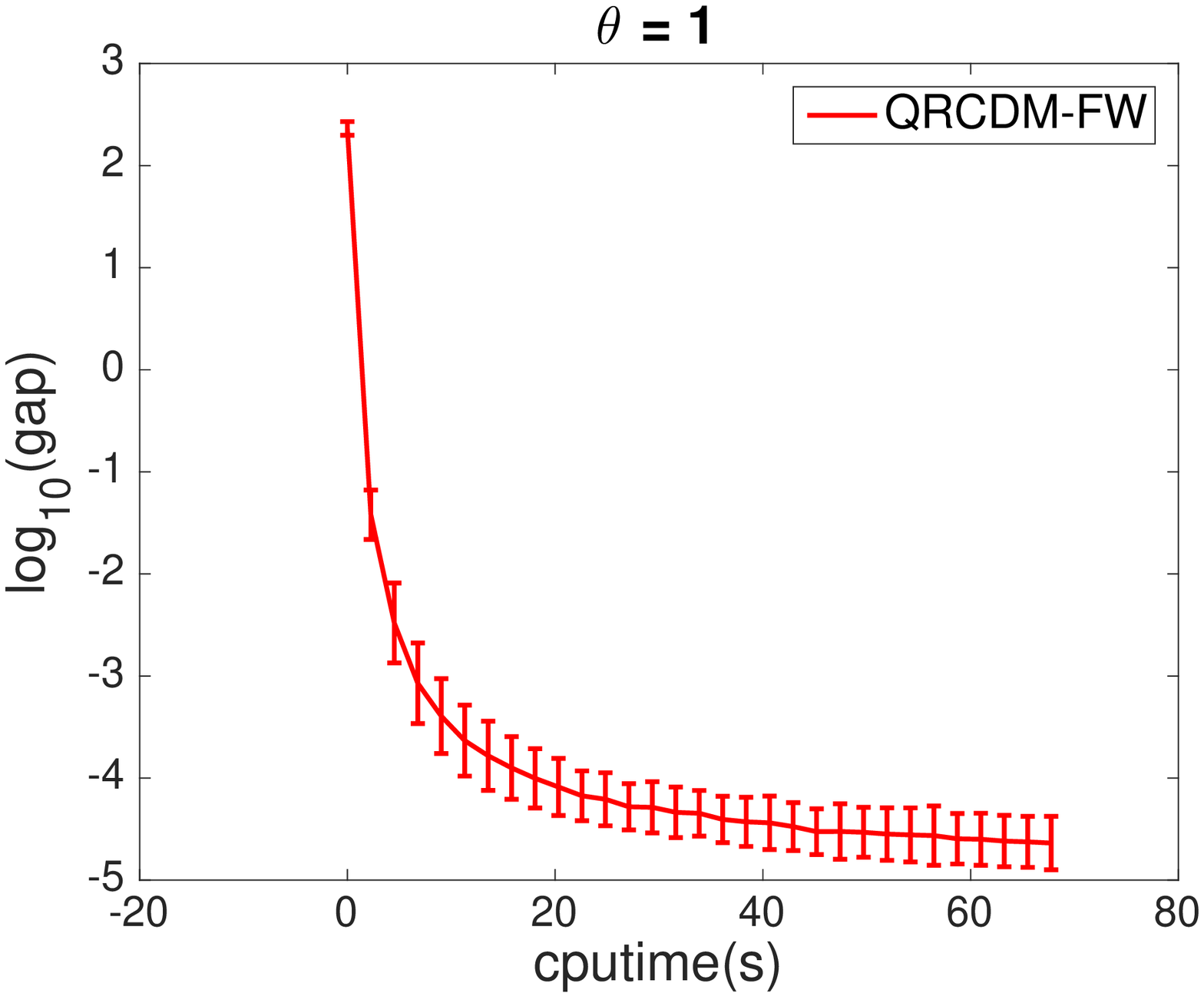}
\caption{Convergence of QRCD-MNP and QRCD-FW for general submodular functions.}
\label{fig:MNP-FW}
\end{figure*}

In Figure~\ref{fig:MNP-FW2}, we also provide the experimental results of QRCD-MNP and QRCD-FW that are evaluated over the synthetic dataset mentioned in Section 6. 
\begin{figure*}[h]
\centering
\includegraphics[trim={0cm 0cm 0cm 0},clip,width=.32\textwidth]{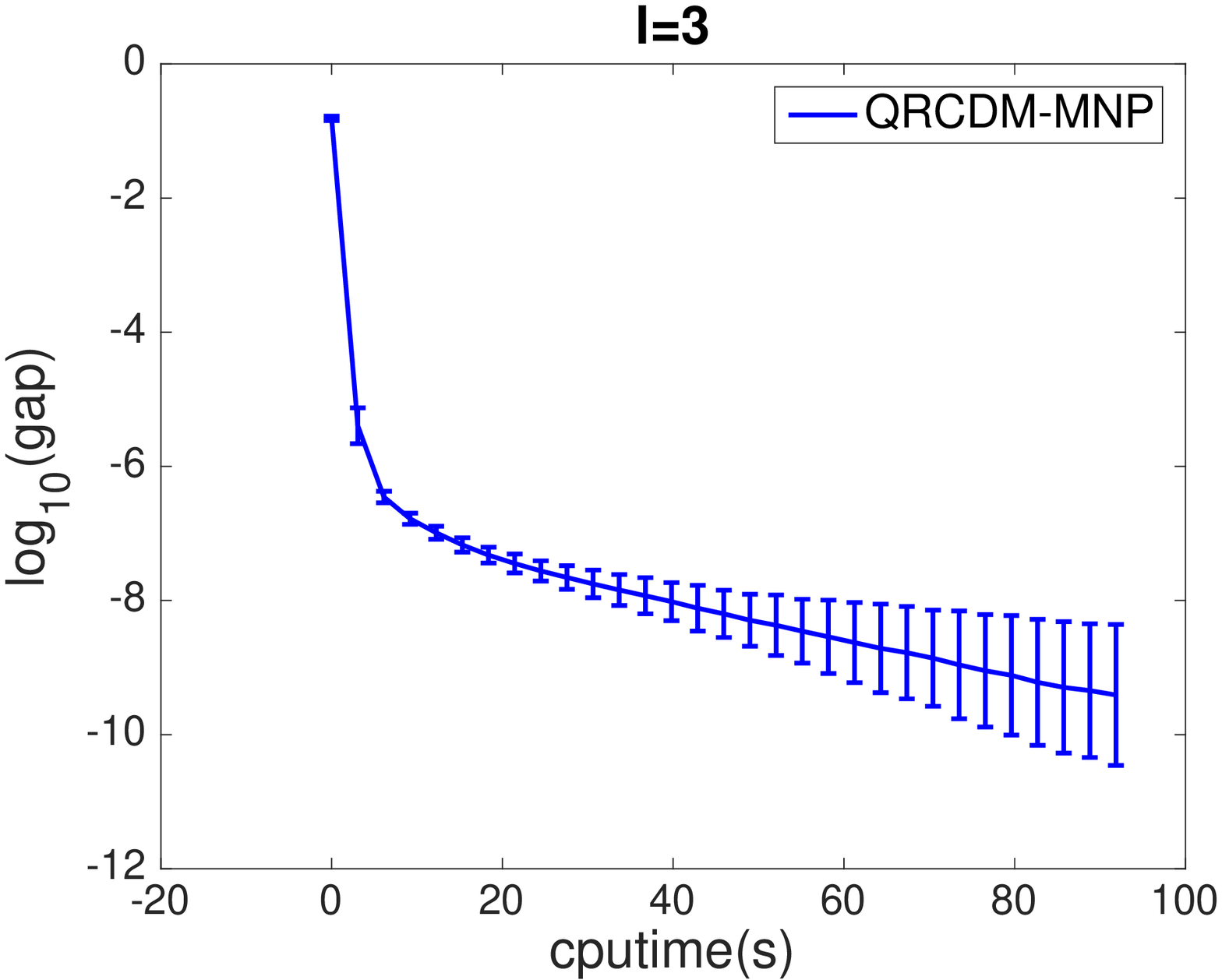}
\vspace{0.15\textwidth}
\includegraphics[trim={0cm 0cm 0cm 0},clip,width=.33\textwidth]{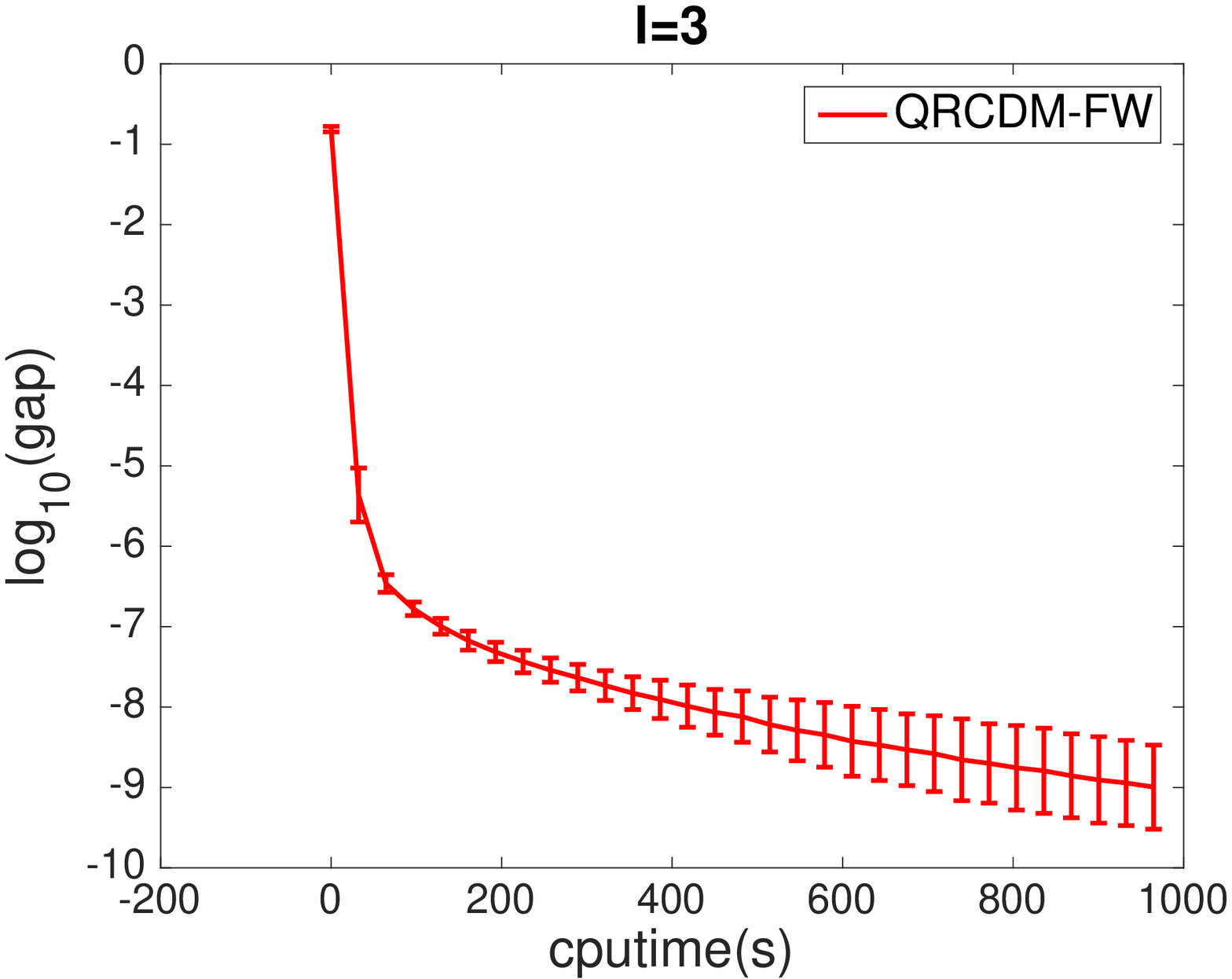}
\caption{Convergence of QRCD-MNP and QRCD-FW for SSL over hypergraphs. The setting is same as that used for the top-left figure in Figure 1: $N = 1000$, $R = 2000$.}
\label{fig:MNP-FW2}
\end{figure*}

%\section*{References}
%
%References follow the acknowledgments. Use unnumbered first-level
%heading for the references. Any choice of citation style is acceptable
%as long as you are consistent. It is permissible to reduce the font
%size to \verb+small+ (9 point) when listing the references. {\bf
%  Remember that you can use more than eight pages as long as the
%  additional pages contain \emph{only} cited references.}
%\medskip
%
%\small
%
%[1] Alexander, J.A.\ \& Mozer, M.C.\ (1995) Template-based algorithms
%for connectionist rule extraction. In G.\ Tesauro, D.S.\ Touretzky and
%T.K.\ Leen (eds.), {\it Advances in Neural Information Processing
%  Systems 7}, pp.\ 609--616. Cambridge, MA: MIT Press.
%
%[2] Bower, J.M.\ \& Beeman, D.\ (1995) {\it The Book of GENESIS:
%  Exploring Realistic Neural Models with the GEneral NEural SImulation
%  System.}  New York: TELOS/Springer--Verlag.
%
%[3] Hasselmo, M.E., Schnell, E.\ \& Barkai, E.\ (1995) Dynamics of
%learning and recall at excitatory recurrent synapses and cholinergic
%modulation in rat hippocampal region CA3. {\it Journal of
%  Neuroscience} {\bf 15}(7):5249-5262.

\end{document}